\documentclass[]{bytedance_seed}

\usepackage{amsmath}
\usepackage{amssymb}
\usepackage{amsthm}
\usepackage{mathtools}
\usepackage{bbm}
\usepackage{booktabs}
\usepackage{bm}
\usepackage{algorithm}
\usepackage{algpseudocode}
\usepackage{xspace}
\usepackage{marvosym}

\numberwithin{equation}{section}
\allowdisplaybreaks[2]
\theoremstyle{plain}
\newtheorem{theorem}{Theorem}[section]
\newtheorem{lemma}[theorem]{Lemma}

\theoremstyle{definition}

\theoremstyle{remark}

\newcommand{\Lcal}{\mathcal{L}}
\newcommand{\E}{\mathbb{E}}
\newcommand{\R}{\mathbb{R}}
\newcommand{\KL}{\operatorname{KL}}
\newcommand{\Normal}{\mathcal{N}}
\newcommand{\data}{p_{\mathrm{data}}}
\newcommand{\dd}{\,\mathrm{d}}
\newcommand{\ind}{\mathbbm{1}}
\newcommand{\sg}{\operatorname{sg}}
\newcommand{\method}{$\mathcal{C}ola2$\xspace}
\fancypagestyle{firststyle}{%
  \fancyhf{}%
  \fancyhead[L]{\vskip 4mm\includegraphics[width=53.3mm,height=6mm]{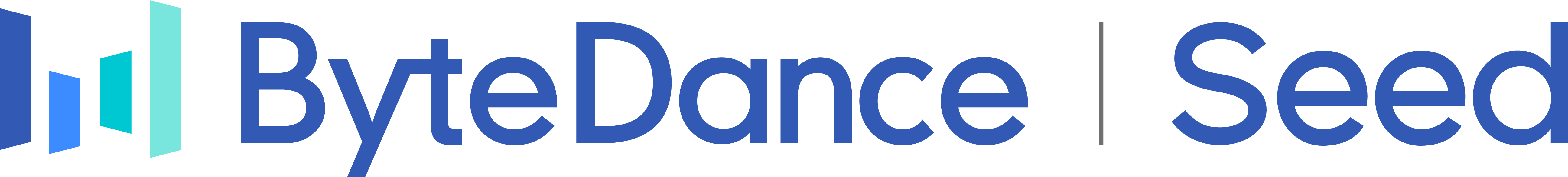}}%
  \fancyhead[R]{\vskip 1mm\raisebox{-5mm}[0pt][0pt]{\includegraphics[height=15mm]{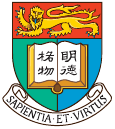}}}%
}

\title{\raisebox{-0.3ex}{\includegraphics[height=2.3ex]{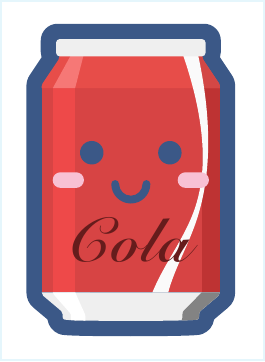}} Cola2: Co-Adaptive Hierarchical Language Modeling with Continuous Semantic Compression and\\ Balanced Discrete Generation}

\author[1,2,\ddagger]{Hongcan Guo}
\author[2,3]{Yian Zhao}
\author[2]{Feng Wang}
\author[4]{Shen Nie}
\author[2,\ddagger]{Guoqiang Wei}
\author[2,\dagger,\ddagger]{\\Yan Zeng}
\author[1,\text{\Letter}]{Hengshuang Zhao}

\affiliation[1]{The University of Hong Kong}
\affiliation[2]{ByteDance Seed}
\affiliation[3]{Peking University}
\affiliation{\\$^{4}$Renmin University of China}

\contribution[\dagger]{Project Lead}
\contribution[\ddagger]{Major Contributor}
\contribution[\text{\Letter}]{Corresponding author}

\abstract{
Hierarchical language generation separates continuous semantic modeling from discrete text realization, but its effectiveness depends on what information the latent space should compress and how the discrete generator should use it. We propose \method, which jointly studies continuous semantic compression, the discrete reconstruction--generation balance, and their co-adaptation. \method dynamically segments text into semantic spans, controls the compression rate with an auxiliary-loss-free negative-feedback mechanism, and combines reconstruction and generation supervision to learn compact and predictive latent representations. A block-causal masked diffusion model (MDM) decoder is jointly trained under clean, noisy, and dropped latent conditions, and classifier-free guidance (CFG) adjusts the balance between semantic reconstruction and autonomous generation. During training, token-prediction gradients update the encoder through the latents, driving the continuous representation and discrete generation to co-adapt, while Flow Matching prior learning is decoupled from representation learning via stop-gradient. At inference, the generated text is re-encoded under the fixed span layout, keeping the latent history consistent with the actual text context. Systematic component analyses and controlled comparisons on five downstream tasks with about 210B training tokens validate the core designs and show competitive, unsaturated scaling. With only 346.03B source tokens, the unconverged checkpoint reaches a 60.10 eight-task average and surpasses Llama-3-8B on HellaSwag and StoryCloze despite its over $43{\times}$ larger reported pretraining budget. These results indicate effective transfer of the Qwen-initialized causal decoder's generation capability and substantial scaling headroom for hierarchical language generation. More broadly, continuous-latent hierarchical modeling can combine effectively with established token-space generation paradigms, including autoregressive and masked-diffusion decoders.
}

\date{\today}
\projectlead{Yan Zeng at \email{zengyan.yanne@bytedance.com}}
\correspondence{Hengshuang Zhao at \email{hszhao@cs.hku.hk}}

\checkdata[Project Page]{\url{https://hongcanguo.github.io/Cola2}}

\begin{document}
\maketitle

\section{Introduction}
\label{sec:intro}

Large language models (LLMs) have made remarkable progress under the autoregressive (AR) paradigm~\citep{brown2020language,touvron2023llama}, but the fixed left-to-right generation order imposes an inherent serial dependency on inference. Diffusion language models (DLMs) \citep{austin2021structured,sahoo2024simple,nie2026large,arriola2025block} offer a new way to break this constraint. Cola DLM~\citep{guo2026continuous} further decomposes text generation into semantic organization in a continuous latent space and text realization in a discrete space; through latent prior modeling and conditional decoding, it provides a viable path toward hierarchical language generation beyond direct token modeling.

However, a structural division between continuous and discrete spaces does not automatically yield an effective division of labor. The key questions are what information the latent space should compress and how the discrete generator should use it. Shortening the latent sequence is not equivalent to obtaining a compact semantic representation, and a reconstruction-only objective does not guarantee that the latents can predict missing or future text~\citep{meshchaninov2026cosmos,jiang2026textldm,liang2026aurora}. Meanwhile, the decoder must not only faithfully realize the latent semantics, but also generate the remaining content from the text context when the latent condition is incomplete or imprecise~\citep{uziel2026crocodil}. Semantic compression and discrete generation therefore cannot be designed independently: the information the latents should retain must be determined jointly with the information the decoder can complete.

This interdependence further raises the question of how continuous and discrete representations should be learned jointly. Fixing the encoder in advance determines the information carried by the latents before the discrete generator is fully learned, limiting the continual adaptation of the representation to decoding needs. At the same time, letting the latent-prior fitting objective directly shape the encoder may push the latents toward directions that are easier to fit but not necessarily more beneficial for discrete generation. Hierarchical language modeling thus requires not only learning continuous representations and a discrete generator, but also specifying which objectives should jointly shape the representation and which learning processes should remain decoupled.

To address these issues, we propose \method, which jointly studies continuous semantic compression, the discrete reconstruction--generation balance, and their co-adaptation within a unified hierarchical language modeling framework. Its core idea is that, under limited latent capacity, the actual demands of text reconstruction and generation jointly determine what the continuous representation should retain, while the discrete generator handles the residual uncertainty that the latents do not explain. Specifically, \method dynamically segments text into semantic spans, controls the compression rate with an auxiliary-loss-free negative-feedback mechanism, and combines reconstruction and generation supervision to learn compact and predictive latent representations. On the discrete side, a block-causal masked diffusion model (MDM) decoder \citep{sahoo2024simple,arriola2025block} is jointly trained under clean, noisy, and dropped latent conditions, balancing semantic reconstruction with context-based autonomous generation, and classifier-free guidance (CFG) \citep{ho2022classifier} adjusts this balance at inference.

During training, the token-prediction loss is backpropagated through the latents to the encoder, so that the continuous representation and discrete generation co-adapt; the Diffusion Transformer (DiT) \citep{peebles2023scalable} learns the current latent distribution via Flow Matching \citep{lipman2023flow} without backpropagating gradients to the encoder, thereby decoupling prior fitting from representation learning. At inference, the model re-encodes the generated text block by block under the fixed span layout, keeping the latent history consistent with the actual text context and aligning the two spaces throughout generation.

Building on these ideas, we systematically study co-adaptive hierarchical language generation across continuous and discrete spaces from three perspectives: probabilistic modeling, representation learning, and generation performance. Our main contributions are as follows:
\begin{itemize}
    \item \textbf{We propose \method, a co-adaptive hierarchical language modeling framework.} From the perspective of conditional information decomposition, we analyze semantic compression, the reconstruction--generation balance, and continuous--discrete co-adaptation in a unified manner, clarifying the relation between the information retained by the latents and the residual left to discrete generation.
    \item \textbf{We develop a generation-oriented joint learning mechanism for continuous and discrete representations.} It combines dynamic semantic compression, generation-oriented supervision, and training under latent conditions of varying reliability, and coordinates representation learning with prior fitting through gradient decoupling, so that compact semantic representations and discrete generation capability co-evolve.
    \item \textbf{We validate the core designs and their scaling potential through systematic experiments.} Controlled five-task scaling to about 210B tokens and a separate 346.03B-source-token pretraining run show competitive, unsaturated scaling. The latter reaches a 60.10 average over eight tasks and surpasses Llama-3-8B on HellaSwag and StoryCloze despite over $43{\times}$ fewer reported tokens, demonstrating decoder capability transfer and compatibility with autoregressive and masked-diffusion token-space generators.
\end{itemize}

\section{Related Work}
\label{sec:related}

\subsection{Semantic Compression and Hierarchical Language Models}

Fixed-granularity methods shorten the modeled sequence through chunking or downsampling~\citep{nawrot2022hierarchical,yu2023megabyte}. Adaptive compression methods dynamically aggregate bytes or tokens according to content~\citep{pagnoni2025byte,hwang2026dynamic,qu2025dynamic}. Concept-level methods organize generation through sentence representations or concept prediction~\citep{barrault2024large,cao2026ncp}. In contrast, \method maintains the target compression rate through negative-feedback control without an auxiliary compression-rate loss, and combines reconstruction and generation supervision to learn compact and predictive continuous semantic representations.

\subsection{Token-Space Language Models}

Autoregressive methods generate text serially, token by token~\citep{brown2020language,grattafiori2024llama,yang2025qwen3}. Discrete diffusion methods predict multiple positions in parallel through iterative denoising~\citep{austin2021structured,lou2023discrete,sahoo2024simple,nie2026large,ye2508dream,sahoo2025diffusion}, and have been further extended to block-level generation and token editing~\citep{arriola2025block,bie2025llada2,bie2026llada2}. Continuous-augmented methods use paired continuous states to assist discrete denoising~\citep{zheng2026continuously,zhou2025coevolutionary,uziel2026crocodil}. In contrast, \method learns a semantic prior in a compressed latent space and then performs conditional generation with a discrete decoder, handling semantic modeling and text realization hierarchically.

\subsection{Continuous Representations and Latent-Space Language Models}

Token-aligned methods perform diffusion or flow modeling on continuous representations~\citep{li2022diffusion,gong2022diffuseq,gulrajani2023likelihood,chen2026langflow,yang2026continuous,lee2026flow,hu2026elf}. Autoencoder-based latent-space methods separate latent generation from text decoding~\citep{lovelace2023latent,zhang2023planner,meshchaninov2026cosmos}. Joint-learning methods further co-optimize the latent space and the generative model~\citep{guo2026continuous,meshchaninov2026train}. In contrast, \method drives representation learning with discrete reconstruction and generation supervision while isolating the prior-fitting gradient, so that semantic compression and the decoder's reconstruction--generation balance are optimized jointly.

\section{Cola2: Continuous Latent Diffusion for Hierarchical Text Modeling}
\label{sec:method}

We formulate hierarchical text generation as a continuous latent-variable model and, from an information-decomposition perspective, discuss semantic compression, the reconstruction--generation balance, and continuous--discrete co-adaptation (Section~\ref{sec:theory}). We then present the unified model, training procedure, latent-feedback inference procedure, and computational cost (Sections~\ref{sec:workflow}--\ref{sec:cost}).

\subsection{Probabilistic and Information-Theoretic Perspective}
\label{sec:theory}

\subsubsection{Hierarchical Probabilistic Modeling}

Let $C$ denote the observed context, $Y$ the text to be generated, and $Z=(z_1,\ldots,z_K)$ a continuous latent sequence with $z_k\in\R^{d_z}$. If $Y$ contains $T$ tokens, the compression rate is $\rho=K/T$. \method factorizes the text probability into a continuous latent prior and discrete conditional generation:
\begin{equation}
    p(Y\mid C)
    =
    \int p_\theta(Z\mid C)\,
    p_\psi(Y\mid C,Z)\dd Z.
    \label{eq:main-prob}
\end{equation}
During training, introducing the encoder $q_\phi(Z\mid C,Y)$ as an approximate posterior yields the conditional evidence lower bound (ELBO)
\begin{equation}
\begin{aligned}
    \Lcal_{\mathrm{cELBO}}
    ={}&
    \E_{q_\phi} [\log p_\psi(Y\mid C,Z)]\\
    &-\KL\!\left(q_\phi(Z\mid C,Y)\,\|\,p_\theta(Z\mid C)\right)
    \le \log p(Y\mid C).
\end{aligned}
\label{eq:main-celbo}
\end{equation}
Defining the conditional aggregated posterior
\begin{equation}
    \bar q_\phi(Z\mid C)
    =
    \int q_\phi(Z\mid C,Y)\,\data(Y\mid C)\dd Y,
    \label{eq:main-agg}
\end{equation}
the average negative ELBO decomposes as
\begin{equation}
\begin{aligned}
    -\E\Lcal_{\mathrm{cELBO}}
    ={}&
    \E[-\log p_\psi(Y\mid C,Z)]
    +I(Y;Z\mid C)\\
    &+\E_C\KL\!\left(\bar q_\phi(Z\mid C)\,\|\,p_\theta(Z\mid C)\right).
\end{aligned}
\label{eq:main-elbo-decomp}
\end{equation}
Thus, the encoder determines which information is retained within the limited latent capacity, the discrete generator decoder implements $p_\psi(Y\mid C,Z)$, and the latent DiT only needs to fit the $p_\theta(Z\mid C)$ induced by the online encoder.

\subsubsection{Semantic Compression and Information Density}

Given $C$, the uncertainty of the target text satisfies
\begin{equation}
    H(Y\mid C)
    =I(Y;Z\mid C)+H(Y\mid C,Z).
    \label{eq:main-info}
\end{equation}
If $S$ denotes a high-level semantic variable determined by $Y$, the chain rule of mutual information gives
\begin{equation}
\begin{aligned}
    H(Y\mid C)
    ={}&\underbrace{I(S;Z\mid C)}_{\text{semantic information}}
    +\underbrace{I(Y;Z\mid S,C)}_{\text{local residual information}}\\
    &+\underbrace{H(Y\mid C,Z)}_{\text{discrete generation residual}}.
\end{aligned}
\label{eq:main-decomp}
\end{equation}
Under the compression constraint $K=\rho T$, whether the latents are worth further modeling in continuous space depends not only on whether $K$ is sufficiently short, but more importantly on whether the information retained in the limited latent positions is semantic and predictive. To characterize this, we define the analytical quantity
\begin{equation}
    \mathcal D_{\mathrm{sem}}
    :=\frac{I(S;Z\mid C)}{K}
    =\frac{I(S;Z\mid C)}{\rho T},
    \label{eq:main-sem-density}
\end{equation}
which measures the density of high-level semantic information carried per latent position. This quantity characterizes the design objective rather than serving as a loss estimated directly during training.

Furthermore, partitioning $Y$ into a visible part $Y^{\mathrm{vis}}$ and a target part $Y^{\mathrm{tar}}$ gives
\begin{equation}
\begin{aligned}
    I(Y^{\mathrm{tar}};Z\mid C,Y^{\mathrm{vis}})
    ={}&H(Y^{\mathrm{tar}}\mid C,Y^{\mathrm{vis}})\\
    &-H(Y^{\mathrm{tar}}\mid C,Y^{\mathrm{vis}},Z).
\end{aligned}
\label{eq:main-predictive-info}
\end{equation}
Therefore, beyond local reconstructability, a high-quality latent space should also reduce the conditional uncertainty of future or missing tokens, i.e., concentrate more generative information within its limited capacity. Section~\ref{sec:workflow} instantiates this principle through dynamic semantic compression and generation-oriented latent supervision.

\subsubsection{Reconstruction--Generation Balance and Continuous--Discrete Co-adaptation}

Let $\widetilde Z$ be the latent condition obtained by passing $Z$ through a stochastic degradation channel, satisfying the conditional Markov chain $Y\rightarrow Z\rightarrow\widetilde Z$ given $C$. By the data processing inequality,
\begin{equation}
    I(Y;Z\mid C)
    \ge I(Y;\widetilde Z\mid C),
    \qquad
    H(Y\mid C,Z)
    \le H(Y\mid C,\widetilde Z).
    \label{eq:main-recon-gen-dpi}
\end{equation}
When the latent condition is intact, the decoder's main task is to realize the continuous semantics as the corresponding discrete tokens; as the condition is perturbed or removed, $H(Y\mid C,\widetilde Z)$ increases, and more uncertainty must be handled by the decoder's own discrete prediction capability. This gives an information-theoretic form of the reconstruction--generation trade-off.

On the other hand, since $H(Y\mid C)$ is determined by the data distribution, changing the encoder effectively reallocates
\begin{equation}
    H(Y\mid C)
    =I(Y;Z_\phi\mid C)+H(Y\mid C,Z_\phi).
    \label{eq:main-coadapt-info}
\end{equation}
If $E_\phi$ is fixed before the decoder is trained, this allocation is determined before the discrete generator is learned. \method instead jointly optimizes
\begin{equation}
    (\phi^\star,\psi^\star)
    =\arg\min_{\phi,\psi}
    \E_{\data(C,Y)}
    \E_{q_\phi(Z\mid C,Y)}
    [-\log p_\psi(Y\mid C,Z)]
    \quad
    \text{s.t.}\quad \frac{K}{T}\approx\rho^\star,
    \label{eq:main-coadapt-opt}
\end{equation}
so that token prediction retains a direct gradient to the encoder:
\begin{equation}
    \nabla_\phi\Lcal_{\mathrm{tok}}
    =
    \frac{\partial\Lcal_{\mathrm{tok}}}{\partial Z}
    \frac{\partial Z}{\partial\phi}
    \neq0.
    \label{eq:main-coadapt-grad}
\end{equation}
This lets the continuous latent space and the discrete realization co-adapt: the encoder learns semantic representations better suited to reconstruction and generation by the actual decoder, while latent-prior fitting is decoupled from this representation learning. The concrete mechanisms for compression, reliability training, gradient routing, and inference alignment are given in the next section.

\subsection{Workflow}
\label{sec:workflow}

Figure~\ref{fig:method_overview} instantiates the above information decomposition with three coordinated modules: the compression variational autoencoder (VAE) encoder $E_\phi$ learns continuous latents, the discrete generator decoder $D_\psi$ realizes the latent condition as discrete tokens, and the block-causal DiT $G_\theta$ learns the generative distribution of the detached online latents. Their relationship can be summarized as
\begin{equation}
    (C,Y)\xrightarrow{\ E_\phi\ } Z,
    \qquad
    (C,Z)\xrightarrow{\ D_\psi\ }Y,
    \qquad
    \sg(Z)\xrightarrow{\ \Lcal_{\mathrm{FM}}\ }G_\theta.
    \label{eq:workflow-overview}
\end{equation}
Here $D_\psi$ is a \emph{block-causal masked-diffusion decoder} (MDM), which is causal across blocks and performs bidirectional token denoising within the current block. Stage~1 jointly learns $E_\phi$ and $D_\psi$; Stage~2 retains this pathway and additionally trains a detached latent DiT. Full implementation details are given in Appendix~\ref{app:workflow-details}.

\subsubsection{Model Architecture}
\label{subsec:workflow-arch}

\paragraph{Compression VAE encoder.}
The target sequence $Y=(y_1,\ldots,y_T)$ is dynamically partitioned into $K$ contiguous spans. Let $\mathcal S_k$ denote the token index set of the $k$-th span and $L_k=|\mathcal S_k|$ its length. Then
\begin{equation}
    \bigcup_{k=1}^{K}\mathcal S_k=\{1,\ldots,T\},
    \qquad
    1\le L_k\le L_{\max},
    \qquad
    \rho=\frac{K}{T}.
    \label{eq:workflow-spans}
\end{equation}
The encoder $E_\phi$ first performs dynamic segmentation over token representations, then aggregates bidirectionally within each span and models causally along the latent axis, yielding a diagonal Gaussian posterior
\begin{equation}
    q_\phi(z_k\mid C,Y)
    =\Normal\!\left(\mu_k,\operatorname{diag}(\sigma_k^2)\right),
    \qquad
    z_k=\mu_k+\sigma_k\odot\epsilon_k,
    \quad \epsilon_k\sim\Normal(0,I).
    \label{eq:workflow-posterior}
\end{equation}
The compression rate is maintained near the target $\rho^\star$ by an independent, gradient-free rate-bias controller rather than learned through an additional differentiable rate loss. The dynamic boundaries, span-local pooling, and convergence properties of the controller are given in Appendix~\ref{app:segmentation}.

\begin{figure}[!t]
    \centering
    \includegraphics[width=\linewidth]{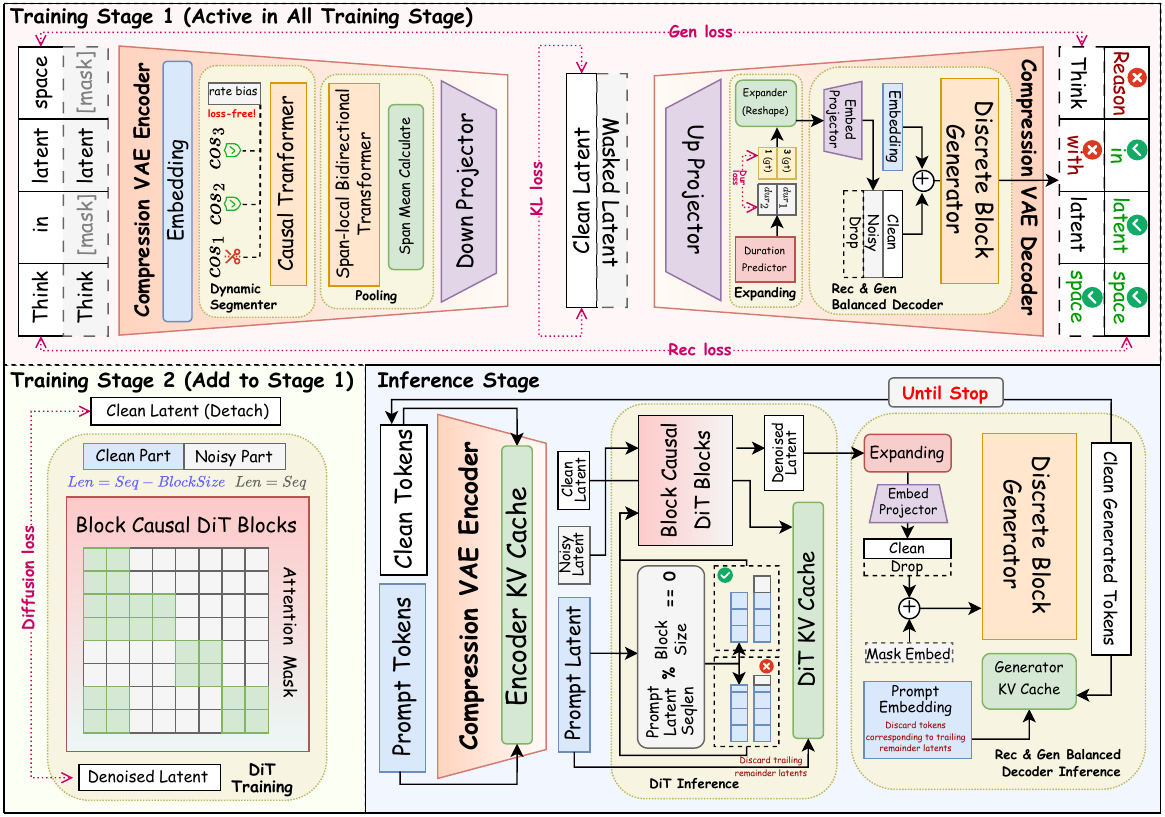}
    \caption{\textbf{Method overview.} Stage~1 jointly trains the compression encoder and the reconstruction--generation balanced decoder; Stage~2 additionally trains a DiT on detached latents; at inference, the DiT generates latents block by block and the decoder decodes them into tokens.}
    \label{fig:method_overview}
\end{figure}

\paragraph{Latent-to-token interface and block-causal MDM decoder.}
Each latent is first projected into the decoder hidden space, and a duration head predicts the number of tokens it covers:
\begin{equation}
    p_\psi(L_k\mid z_k)
    =\operatorname{softmax}\!\big(\operatorname{MLP}_{\mathrm{dur}}(W_{\mathrm{in}}z_k)\big).
    \label{eq:workflow-duration}
\end{equation}
Given the span lengths, the latents are expanded into a token-aligned condition $H^Z=(h^Z_1,\ldots,h^Z_T)$. If $g(t)=k$ indicates that token $t$ belongs to span $k$, then
\begin{equation}
    h^Z_t
    =\operatorname{Expand}\!\left(z_{g(t)},j_t,L_{g(t)}\right),
    \label{eq:workflow-expand}
\end{equation}
where $j_t$ is the relative position of token $t$ within its span. The block-causal MDM decoder $D_\psi$ then predicts the target tokens in token space:
\begin{equation}
    p_\psi(y_t\mid \mathcal K_t,H^Z_t)
    =\operatorname{softmax}\!\big(\operatorname{Head}(D_\psi(\mathcal K_t),H^Z_t)\big),
    \label{eq:workflow-discrete-generator}
\end{equation}
where $\mathcal K_t$ contains the historical clean blocks and the current masked block; attention is causal across blocks and bidirectional denoising is allowed within the current block, which cannot read its own clean copy. The attention mask and training objective are detailed in Appendix~\ref{app:discrete-generator-details}.

\paragraph{Latent generator.}
The latent sequence is partitioned into contiguous blocks of fixed size $B$,
\begin{equation}
    \mathcal B_j
    =\{(j-1)B+1,\ldots,\min(jB,K)\},
    \qquad
    J=\left\lceil\frac{K}{B}\right\rceil .
    \label{eq:workflow-blocks}
\end{equation}
The generator $G_\theta$ is a block-causal DiT: causal across blocks and bidirectional within the current block. It learns $p_\theta(Z\mid C)$ in continuous latent space via Flow Matching. High-level generation therefore iterates along a latent axis of length $K\approx\rho T$, while discrete details are recovered by $D_\psi$ in token space.

\subsubsection{Training Procedure and Objectives}
\label{subsec:workflow-train}

\paragraph{Stage 1: semantic compression and reconstruction--generation co-adaptation.}
Stage~1 jointly optimizes $E_\phi$, the latent-to-token interface, and $D_\psi$. To cover the spectrum in Eq.~\eqref{eq:main-recon-gen-dpi} from strong-condition reconstruction to weak-condition generation, we sample one of three mutually exclusive condition-reliability states for each latent block. Let $A\in\{\textsf{cond},\textsf{noise},\textsf{drop}\}$ and define
\begin{equation}
    \widetilde Z_A=
    \begin{cases}
        Z, & A=\textsf{cond},\\
        (1-u^{\mathrm c})Z+u^{\mathrm c}\varepsilon,
        & A=\textsf{noise},\\
        \varnothing, & A=\textsf{drop},
    \end{cases}
    \qquad
    \varepsilon\sim\Normal(0,I),
    \label{eq:workflow-reliability}
\end{equation}
where $\varnothing$ removes only the latent condition of the current block while retaining the same discrete token context. Let $\Lcal_{\mathrm{tok}}^{A}$ denote the MDM prediction objective under state $A$. The three states form
\begin{equation}
    \Lcal_{\mathrm{bal}}
    =\lambda_{\mathrm c}\Lcal_{\mathrm{tok}}^{\textsf{cond}}
    +\lambda_{\mathrm n}\Lcal_{\mathrm{tok}}^{\textsf{noise}}
    +\lambda_{\mathrm d}\Lcal_{\mathrm{tok}}^{\textsf{drop}},
    \label{eq:workflow-balance-loss}
\end{equation}
where $\Lcal_{\mathrm{tok}}^{A}$ is the masked denoising loss under the corresponding state, defined in Appendix~\ref{app:discrete-generator-details}.

To increase the generative information in Eq.~\eqref{eq:main-predictive-info}, we additionally introduce
\begin{equation}
    \Lcal_{\mathrm{gen}}=\Lcal_{\mathrm{encmask}},
    \label{eq:workflow-gen-loss}
\end{equation}
which masks part of the encoder input under a fixed span layout and directly supervises the encoder through the corresponding token prediction. Thus,
\begin{equation}
    \Lcal_{\mathrm{DGD}}
    =\Lcal_{\mathrm{bal}}
    +\lambda_{\mathrm{gen}}\Lcal_{\mathrm{gen}}.
    \label{eq:workflow-dg-loss}
\end{equation}

The KL regularization of the diagonal Gaussian posterior is
\begin{equation}
    \Lcal_{\mathrm{KL}}
    =\frac{K_{\mathrm{ref}}d_{\mathrm{ref}}}{T_{\mathrm{ref}}}
    \frac{1}{Kd_z}\sum_{k=1}^{K}
    \KL\!\left(q_\phi(z_k\mid C,Y)\,\|\,\Normal(0,I)\right),
    \label{eq:workflow-kl}
\end{equation}
and the duration predictor is supervised with ground-truth span lengths, giving $\Lcal_{\mathrm{dur}}$. The unified Stage~1 objective is
\begin{equation}
    \Lcal^{(1)}
    =\Lcal_{\mathrm{DGD}}
    +\beta_{\mathrm{KL}}\Lcal_{\mathrm{KL}}
    +\lambda_{\mathrm{dur}}\Lcal_{\mathrm{dur}}.
    \label{eq:workflow-stage1}
\end{equation}
Its key gradient relation is
\begin{equation}
    \nabla_\phi\Lcal_{\mathrm{DGD}}
    =\frac{\partial\Lcal_{\mathrm{DGD}}}{\partial Z}
     \frac{\partial Z}{\partial\phi}
    \neq0,
    \label{eq:workflow-dec-to-enc-grad}
\end{equation}
i.e., discrete token prediction directly shapes the encoder latents, so that $E_\phi$ and $D_\psi$ jointly determine the information allocation in Eq.~\eqref{eq:main-coadapt-info}. This pathway remains active in Stage~2.

\paragraph{Stage 2: detached latent diffusion co-training.}
Stage~2 retains Eq.~\eqref{eq:workflow-stage1} and adds only detached latent-prior fitting. Let $\bar Z=\sg(Z)$; for the $j$-th latent block, we sample
\begin{equation}
    u_j=\sigma(\xi_j),
    \qquad
    \xi_j\sim\Normal(\mu_u,\sigma_u^2),
    \label{eq:workflow-logitnormal}
\end{equation}
and all latents within the block share the same $u_j$. Under the linear Flow Matching path,
\begin{equation}
    z_u=(1-u)\bar z+u\epsilon,
    \qquad
    v^\star=\epsilon-\bar z,
    \qquad
    \epsilon\sim\Normal(0,I),
    \label{eq:workflow-flow}
\end{equation}
which gives
\begin{equation}
    \Lcal_{\mathrm{FM}}
    =\E\!\left[\left\|G_\theta(z_u,u;\bar Z_{<j})-v^\star\right\|_2^2\right].
    \label{eq:workflow-fm}
\end{equation}
The total Stage~2 objective is therefore
\begin{equation}
    \Lcal^{(2)}
    =\Lcal^{(1)}+\lambda_{\mathrm{FM}}\Lcal_{\mathrm{FM}},
    \label{eq:workflow-stage2}
\end{equation}
and hence
\begin{equation}
    \nabla_\phi\Lcal_{\mathrm{FM}}=0,
    \qquad
    \nabla_\phi\Lcal_{\mathrm{DGD}}\neq0,
    \label{eq:workflow-gradient-routing}
\end{equation}
i.e., prior fitting is decoupled from representation learning, while the continuous representation and the discrete realization remain co-adapted end to end. The complete gradient routing is given in Appendix~\ref{app:stage2-details}.

\subsubsection{Inference Procedure}
\label{subsec:workflow-infer}

\paragraph{Propose--decode--re-encode--commit.}
Given a prompt $C$, the encoder takes the posterior mode to obtain $K_{\mathrm p}$ prompt latents. Writing
\begin{equation}
    K_{\mathrm p}
    =B\left\lfloor\frac{K_{\mathrm p}}{B}\right\rfloor+r,
    \qquad 0\le r<B,
    \label{eq:workflow-prompt}
\end{equation}
only the complete block-aligned prefix is persistently committed; the trailing $r$ latents and their tokens serve as the clean prefix of the first generated block, and the remaining $B-r$ latent slots are initialized from Gaussian noise.

The $j$-th generated block follows a unified transaction
\begin{equation}
    Z_j^{\mathrm{prop}}
    \xrightarrow{\ D_\psi\ }
    Y_j
    \xrightarrow{\ E_\phi(\,\cdot\,;\widehat{\mathcal S}_j)\ }
    Z_j^{\mathrm{corr}}
    \xrightarrow{\ \mathrm{commit}\ }
    \mathcal K_{j+1},
    \label{eq:workflow-feedback}
\end{equation}
where $\widehat{\mathcal S}_j$ is the span layout already determined by the duration predictor; feedback does not rerun the dynamic segmenter. The proposal serves only the current token block, and the persistent latent history always stores $Z_j^{\mathrm{corr}}$.

This design first preserves train--inference consistency. During training, the clean history is obtained by encoding ground-truth tokens; at inference,
\begin{equation}
    Z_{\le j}^{\mathrm{hist}}
    =E_\phi\!\left(Y_{\le j};\widehat{\mathcal S}_{\le j}\right),
    \label{eq:workflow-train-infer-consistency}
\end{equation}
so the latent history read by the next block is still produced by the same encoder path as in training.

It also keeps the continuous latent context aligned with the discrete token context:
\begin{equation}
    \mathcal K^{Z}_{\le j}
    \equiv
    E_\phi\!\left(\mathcal K^{Y}_{\le j};\widehat{\mathcal S}_{\le j}\right).
    \label{eq:workflow-context-alignment}
\end{equation}
If $Z_j^{\mathrm{prop}}$ were committed directly, the subsequent DiT history and the token history actually generated by the decoder could stem from different semantic states, causing cumulative context bias; fixed-span feedback projects both back onto the same committed history. The cache protocol is detailed in Appendix~\ref{app:inference-details}.

\paragraph{Two-level CFG.}
The latent DiT modulates the historical latent condition in velocity space:
\begin{equation}
    \hat v_\gamma=v_{\mathrm u}+\gamma_{\mathrm{gen}}(v_{\mathrm c}-v_{\mathrm u}),
    \label{eq:workflow-cfg}
\end{equation}
where $v_{\mathrm c}$ and $v_{\mathrm u}$ are the velocity outputs with and without the historical latent condition, respectively, and $\gamma_{\mathrm{gen}}$ is the guidance scale. Its local probabilistic interpretation is given in Appendix~\ref{app:cfg}.

The MDM decoder uses \texttt{masked\_only} conditioning, so the historical token KV contains no latent condition, and applies prefix--latent two-axis CFG in logit space. Fixing the current masked block and the span layout, let $\ell_{ab}$ denote the logits obtained by keeping ($1$) or masking ($0$) the historical token prefix and the current latent condition, where $a,b\in\{0,1\}$ correspond to the two conditions, respectively. With two-axis scales $\gamma_{\mathrm p},\gamma_{\mathrm z}$, we define
\begin{equation}
\begin{aligned}
    \ell_{\gamma_{\mathrm p},\gamma_{\mathrm z}}
    ={}&\gamma_{\mathrm p}\gamma_{\mathrm z}\ell_{11}
      +\gamma_{\mathrm p}(1-\gamma_{\mathrm z})\ell_{10}\\
      &+(1-\gamma_{\mathrm p})\gamma_{\mathrm z}\ell_{01}
      +(1-\gamma_{\mathrm p})(1-\gamma_{\mathrm z})\ell_{00}.
\end{aligned}
\label{eq:workflow-prefix-cfg}
\end{equation}
Setting $\gamma_{\mathrm p}=1$ modulates only the latent, setting $\gamma_{\mathrm z}=1$ modulates only the prefix, and setting both to $1$ recovers the fully conditioned branch. This rule changes neither the training objective nor the feedback protocol. The four-branch definition and cache implementation are given in Appendices~\ref{app:prefix-cfg} and~\ref{app:prefix-cfg-implementation}.

Let $S_{\mathrm{gen}}$ be the number of DiT denoising steps. On the uniform time grid $u_i=1-i/S_{\mathrm{gen}}$, the explicit Euler update is
\begin{equation}
    Z_j^{(i+1)}
    =Z_j^{(i)}
    -\frac{1}{S_{\mathrm{gen}}}
    \hat v_\gamma\!\left(Z_j^{(i)},u_i\right),
    \qquad i=0,\ldots,S_{\mathrm{gen}}-1.
    \label{eq:workflow-euler}
\end{equation}
Given the proposal, $D_\psi$ completes the current token block with a few iterative demasking steps, followed by the feedback correction in Eq.~\eqref{eq:workflow-feedback}.

\subsection{Computational Cost}
\label{sec:cost}

Let $D_s$ be the training token budget of Stage~$s$, $c_{\mathrm{rep}}^{(s)}$ the average per-token cost of the encoder--decoder path in that stage, and $c_G^{(2)}$ the average per-token cost of the added DiT, both including the corresponding attention and output projections. Then
\begin{equation}
    \mathcal C_{\mathrm{train}}
    \approx
    D_1c_{\mathrm{rep}}^{(1)}
    +D_2\big(c_{\mathrm{rep}}^{(2)}+c_G^{(2)}\big).
    \label{eq:cost-train-total}
\end{equation}
Under the same prior architecture and a comparable layout, $K=\rho T$ reduces the dominant linear-layer and attention-pair terms to $\mathcal O(\rho)$ and $\mathcal O(\rho^2)$ of their token-axis counterparts, respectively, but the representation path persists in both stages.

Inference cost consists of DiT proposal/commit, discrete decoding (including condition preparation), and fixed-span feedback, denoted by $\mathcal C_G^{\mathrm{infer}}$, $\mathcal C_D^{\mathrm{infer}}$, and $\mathcal C_{\mathrm{fb}}$, respectively. Taking as reference the AR cost $\mathcal C_{\mathrm{AR}}^{\mathrm{infer}}$ under the same prompt/output budget and accounting convention,
\begin{equation}
    R_{\mathrm{infer}}
    =\frac{\mathcal C_G^{\mathrm{infer}}
           +\mathcal C_D^{\mathrm{infer}}
           +\mathcal C_{\mathrm{fb}}}
          {\mathcal C_{\mathrm{AR}}^{\mathrm{infer}}}.
    \label{eq:cost-infer-ratio}
\end{equation}
Both sides exclude the one-time prefill but retain the cost of subsequently reading the prompt KV.

Let $\rho_{\mathrm{gen}}$ be the realized latent/token ratio of the generated text, and let $H_D$ and $H_F$ be the numbers of per-block call rounds for condition preparation/decoding/clean-token commit and for feedback/corrected-cache commit, respectively. Ignoring boundary effects and assuming ideally parallel CFG branches, the number of serial call rounds $\mathcal H$ satisfies
\begin{equation}
    \frac{\mathcal H_{\text{\method}}}{\mathcal H_{\mathrm{AR}}}
    \approx
    \frac{\rho_{\mathrm{gen}}}{B}
    \big(S_{\mathrm{gen}}+H_D+H_F\big).
    \label{eq:cost-depth}
\end{equation}
This is a forward-round proxy rather than a latency guarantee. The end-to-end cost also depends on the actual number of CFG branches, the token decoder, and feedback; per-term counts, the AR comparison, and the KV analysis are given in Appendix~\ref{app:compute-details}.

\section{Experiments}
\label{sec:experiments}

\begingroup
\setcounter{topnumber}{2}
\setcounter{totalnumber}{4}
\renewcommand{\topfraction}{0.92}
\renewcommand{\textfraction}{0.06}
\renewcommand{\floatpagefraction}{0.85}
\setlength{\floatsep}{7pt plus 2pt minus 2pt}
\newcounter{expobs}
\newcommand{\obs}{\refstepcounter{expobs}Obs.~\theexpobs}

Following the training pipeline in Section~\ref{sec:workflow}, our experiments address three questions: (1) how to build a compressed latent space that is semantically rich and easy to model; (2) how to train the latent generator; and (3) how the decoder should allocate its capacity between conditional reconstruction and autonomous generation.

\subsection{Experimental Setup}
\label{subsec:exp-setup}

\textbf{Datasets.} For the component analyses and controlled scaling experiments, we train on the open-source OLMo~2 pretraining corpus~\citep{olmo20242}. The component analyses are conducted mainly on LAMBADA and examine the construction of the semantic latent space, the learning of the latent generator, and the reconstruction--generation balance of the decoder. The controlled scaling experiments are conducted on five downstream tasks, RACE~\citep{lai2017race}, OBQA~\citep{mihaylov2018can}, Story Cloze~\citep{mostafazadeh2016corpus}, LAMBADA~\citep{paperno2016lambada}, and SQuAD~\citep{rajpurkar2016squad}, and compare how different generation paradigms improve as training scales. The post-pretraining comparison additionally evaluates HellaSwag~\citep{zellers2019hellaswag}, NarrativeQA~\citep{kovcisky2018narrativeqa}, ARC-Easy~\citep{clark2018think}, TruthfulQA~\citep{lin2022truthfulqa}, and MMLU~\citep{hendrycks2020measuring}.

\textbf{Baselines.} Component studies compare \method variants and DLCM~\citep{qu2025dynamic} for rate control. Controlled scaling compares AR~\citep{brown2020language}, LLaDA~\citep{nie2026large}, Cola DLM~\citep{guo2026continuous}, and the uncompressed Cola-MDM ($\rho=1$) under matched model size, data, and evaluation; Cola-MDM shares the \method-MDM decoder to isolate the effect of compression. Post-pretraining comparisons use Llama-3-8B~\citep{grattafiori2024llama}, Qwen2.5-7B~\citep{hui2024qwen2}, Qwen3-8B-Base~\citep{yang2025qwen3}, and LLaDA-8B-Base~\citep{nie2026large}.

\textbf{Metrics.} Answers are extracted from model outputs according to each task format, and accuracy (\%) is computed by exact match. All tasks except LAMBADA use few-shot prompting. The internal ablations report LAMBADA accuracy by default, and the controlled scaling experiments report the accuracy on all five tasks and their average.

\textbf{Settings.} In the component analyses and controlled scaling experiments, we use the OLMo~2 (dolma2) tokenizer and the AdamW optimizer with $\beta=(0.9,0.95)$, weight decay $0.01$, and a maximum sequence length of $512$. Stage~1 jointly trains the encoder and decoder; Stage~2 trains the DiT while continuing to update the VAE, and blocks the Flow Matching gradient from reaching the encoder. The learning rates of the Stage~1 VAE and the Stage~2 DiT both start at $10^{-6}$, increase linearly over $5{,}000$ warmup steps to $10^{-4}$ and $1.5\times10^{-4}$, respectively, and then follow cosine decay to $10^{-6}$ at step $50{,}000$ and $10^{-5}$ at step $1{,}000{,}000$, respectively. In Stage~2, the VAE learning rate continues from Stage~1 with a $5{,}000$-step transition, and the encoder learning rate is $0.1\times$ that of the decoder. Unless otherwise stated, each latent covers $2$ tokens on average, the latent dimension is $64$, and the maximum span length is $8$. DiT noise times follow a logit-normal distribution ($\mathrm{loc}=2.0$, $\mathrm{scale}=1.0$); the decoder samples clean, noisy, and dropped latent conditions per block with probabilities $0.2$, $0.4$, and $0.4$, respectively. Inference uses $8$ DiT denoising steps and $4$ MDM demasking steps by default, and the controlled scaling experiments train on about $210$B tokens.

\subsection{Building the Semantic Latent Space}
\label{subsec:exp-stage1}

\subsubsection{Dynamic Segmentation and Rate Control}

\begin{figure}[t]
    \centering
    \includegraphics[width=\textwidth]{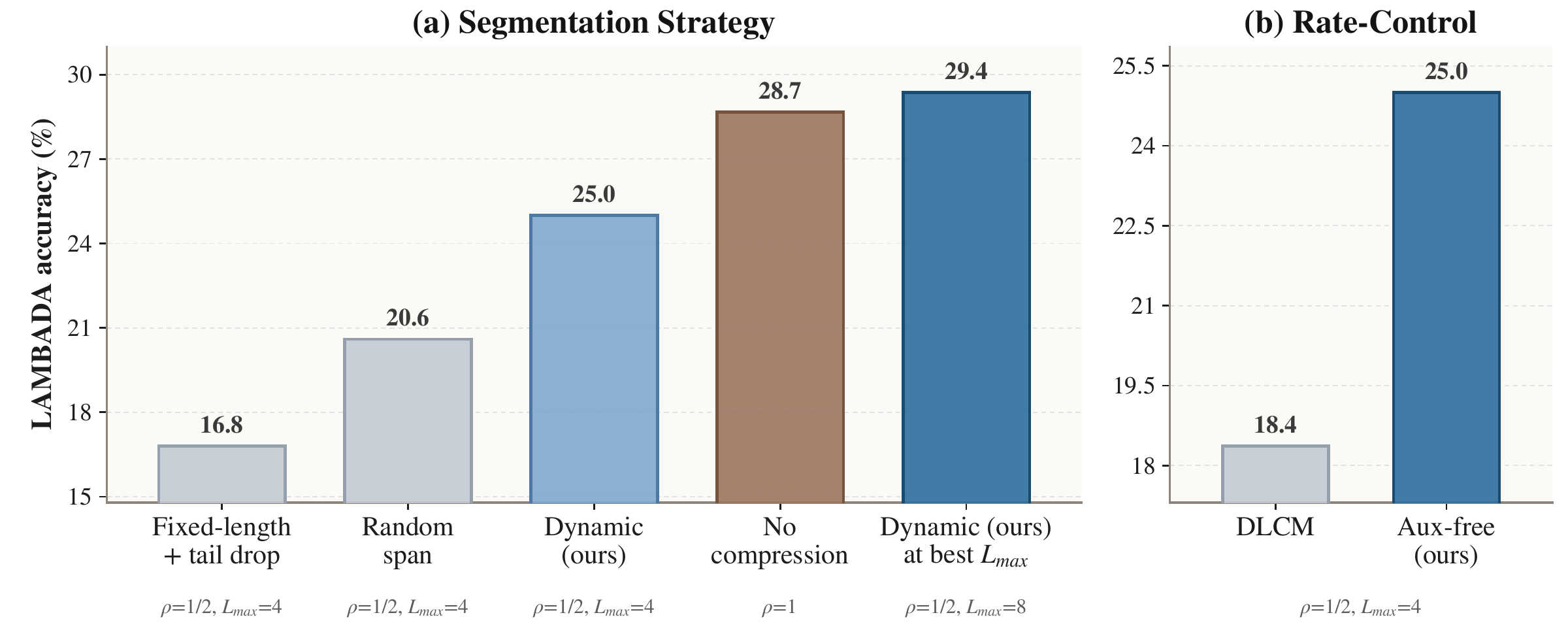}
    \caption{\textbf{Segmentation and rate control.} Dynamic segmentation outperforms fixed-length and random segmentation; auxiliary-free negative-feedback control outperforms DLCM~\citep{qu2025dynamic}.}
    \label{fig:exp-stage1-segmentation}
\end{figure}

\begin{figure}[t]
    \centering
    \includegraphics[width=\textwidth]{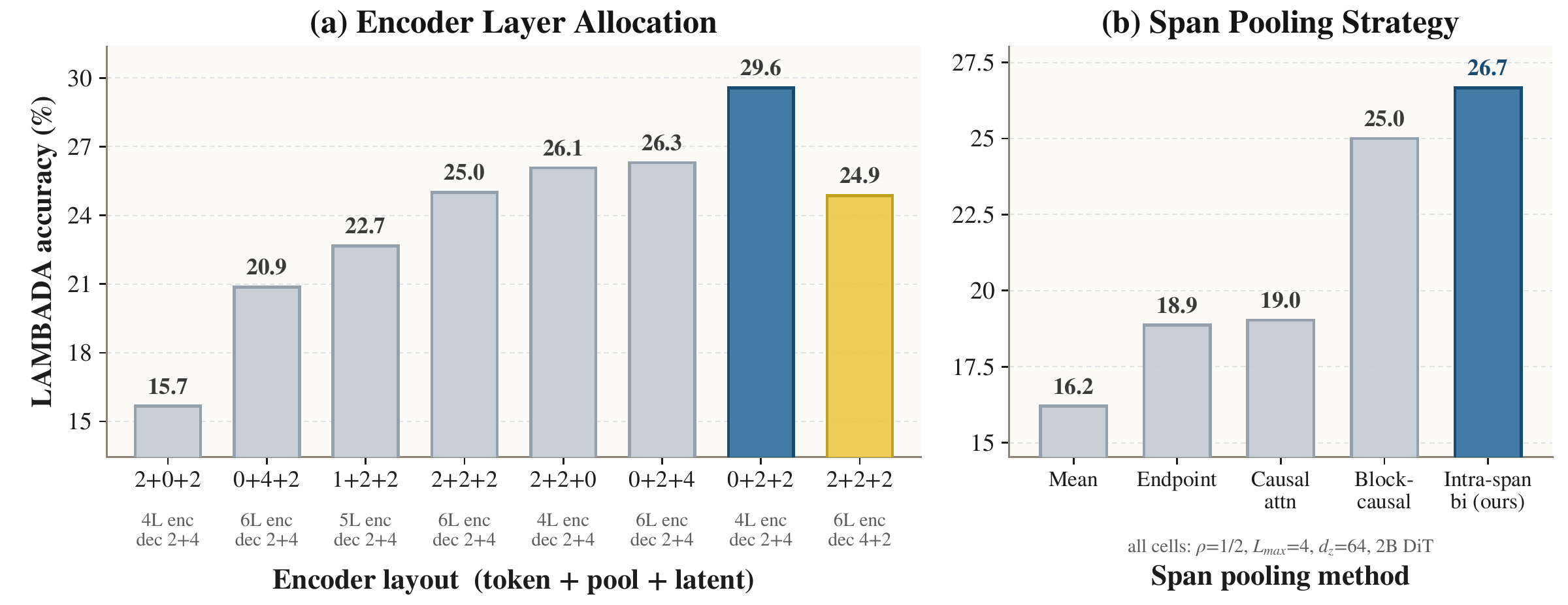}
    \caption{\textbf{VAE architecture.} Left: the allocation of encoder layers matters more than depth, with $0{+}2{+}2$ performing best. Right: intra-span bidirectional attention is the best pooling operator.}
    \label{fig:exp-stage1-architecture}
\end{figure}

\textbf{\obs{} Dynamic segmentation is essential.} With $\rho=1/2$ and $L_{\max}=4$, dynamic segmentation ($25.0\%$) outperforms random spans and fixed-length segmentation by $4.4$ and $8.2$ points (Figure~\ref{fig:exp-stage1-segmentation}); relaxing $L_{\max}$ to $8$ reaches $29.4\%$, exceeding the uncompressed model ($28.7\%$), which shows that content-based compression is not only lossless but also beneficial.

\textbf{\obs{} Feedback rate control beats an auxiliary rate loss.} Under the same setting, the former reaches $25.0\%$ whereas DLCM~\citep{qu2025dynamic} reaches only $18.4\%$: directly regulating the hard compression rate avoids the competition between the rate loss and representation learning.

\subsubsection{VAE Architecture}

\textbf{\obs{} Encoder layout matters more than size.} The four-layer $0{+}2{+}2$ encoder achieves the best result of $29.6\%$, outperforming all six-layer variants ($20.9$--$26.3\%$); removing the pooling layers ($2{+}0{+}2$) yields only $15.7\%$, whereas reallocating decoder layers has almost no effect ($25.0\%$ vs.\ $24.9\%$; Figure~\ref{fig:exp-stage1-architecture}, left). An oversized encoder makes the latent distribution more complex and harder for the DiT to model. Intra-span bidirectional attention performs best ($26.7\%$), followed by block-causal attention ($25.0\%$), while fully causal attention, endpoint pooling, and mean pooling all stay at or below $19.0\%$ (Figure~\ref{fig:exp-stage1-architecture}, right).

\subsubsection{Compression and Latent Capacity}

\begin{figure}[t]
    \centering
    \includegraphics[width=\textwidth]{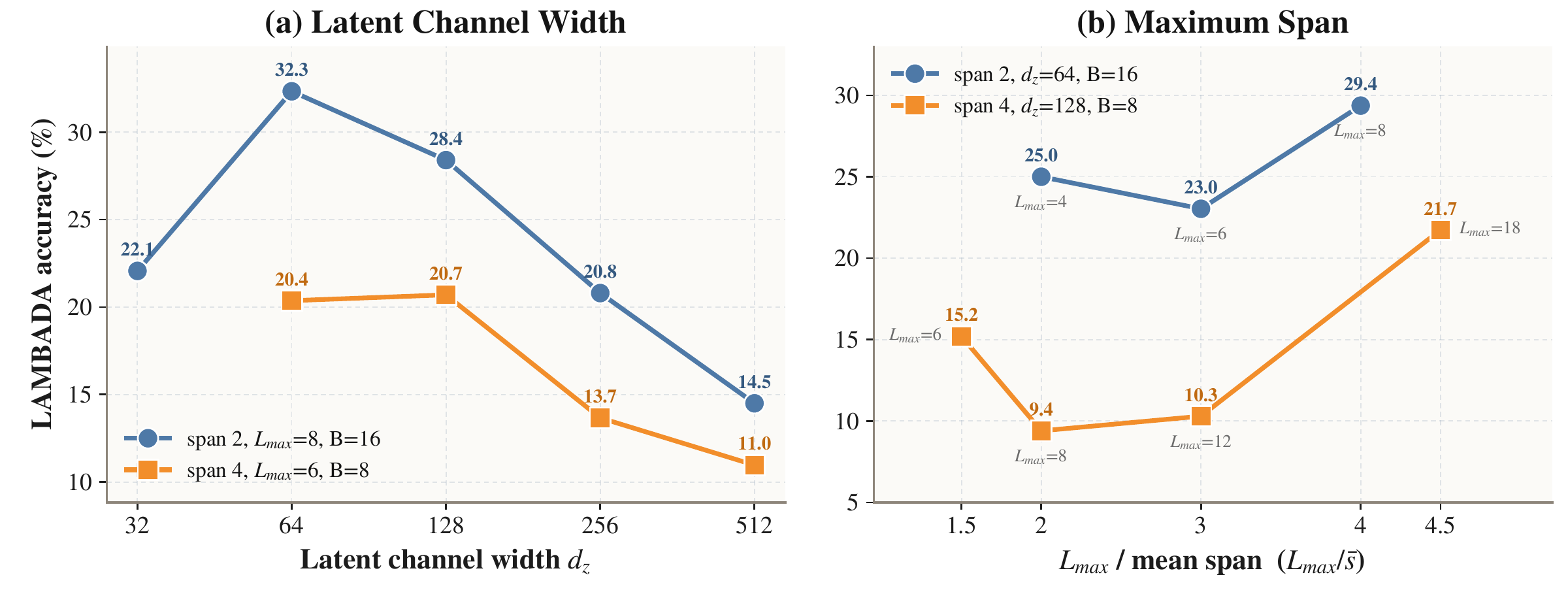}
    \caption{\textbf{Latent width and maximum span length.} Left: $d_z$ has an interior optimum. Right: the effect of $L_{\max}$ is non-monotonic.}
    \label{fig:exp-stage1-capacity}
    \label{fig:exp-stage1-maxspan}
\end{figure}

\begin{figure}[t]
    \centering
    \includegraphics[width=\textwidth]{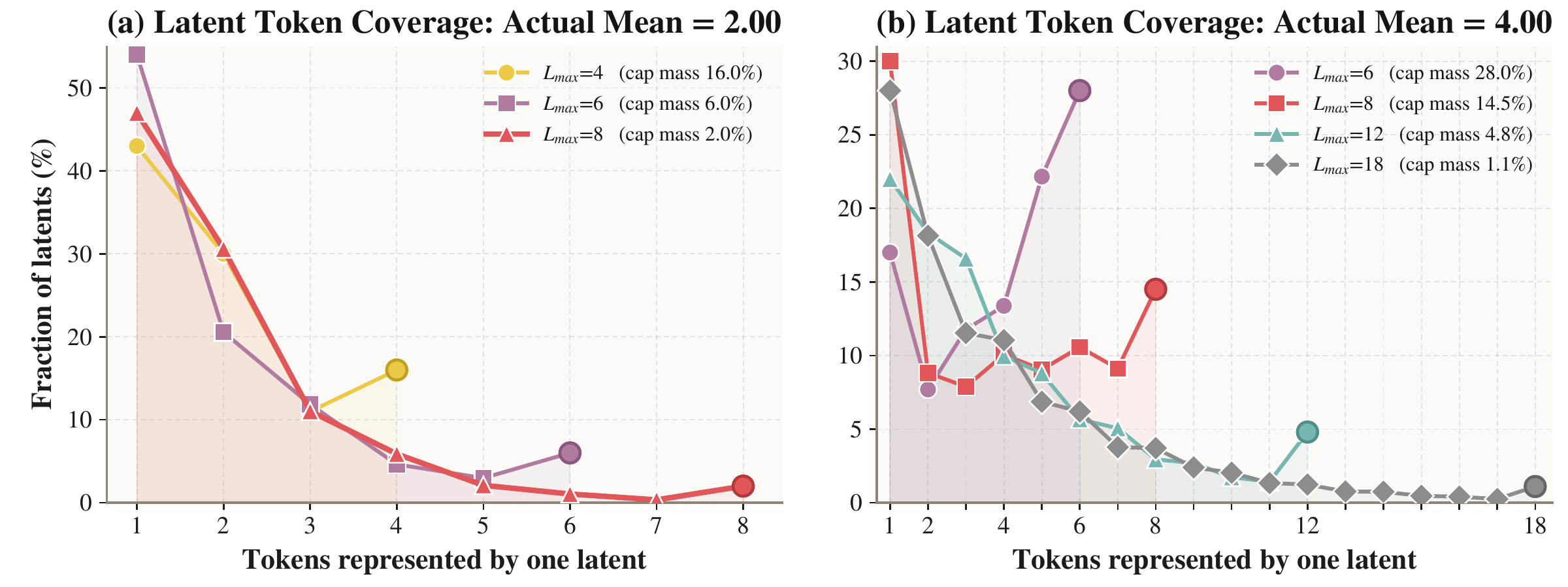}
    \caption{\textbf{Distribution of the number of tokens covered by each latent.} Increasing $L_{\max}$ reduces the number of forcibly truncated spans.}
    \label{fig:exp-stage1-span-distribution}
\end{figure}

\begin{table}[t]
    \centering
    \caption{\textbf{Best configuration at each compression level.} Accuracy decreases as compression becomes stronger.}
    \label{tab:exp-stage1-compression}
    \small
    \setlength{\tabcolsep}{10pt}
    \renewcommand{\arraystretch}{1.10}
    \begin{tabular}{cccc}
        \toprule
        \textbf{Mean span length} & \textbf{Latent width $d_z$} & \textbf{$L_{\max}$} & \textbf{LAMBADA (\%)} \\
        \midrule
        $2$ & $64$  & $8$  & \textbf{32.3} \\
        $3$ & $128$ & $8$  & 26.9 \\
        $4$ & $256$ & $18$ & 25.2 \\
        \bottomrule
    \end{tabular}
\end{table}

\begin{figure}[t]
    \centering
    \includegraphics[width=\textwidth]{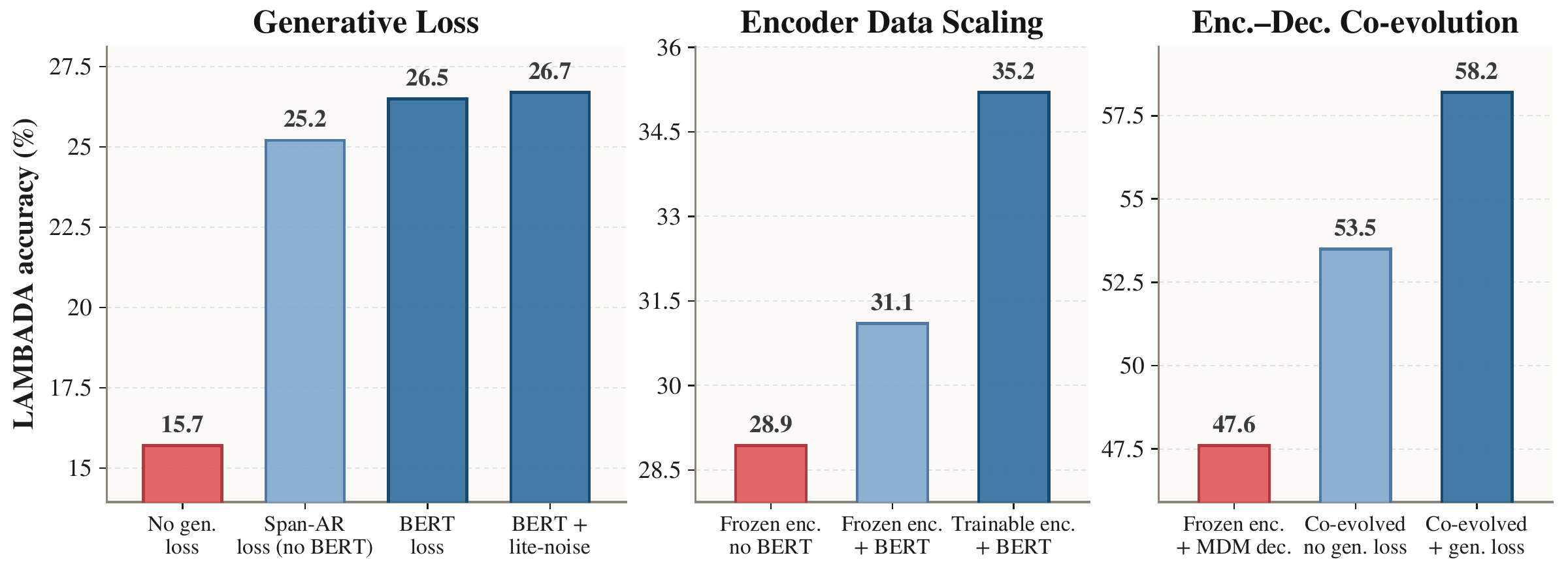}
    \caption{\textbf{Generative supervision.} Left: generative objectives substantially outperform pure reconstruction. Middle: the gain is larger when the encoder is trainable. Right: co-evolving the encoder and decoder with the generative loss performs best.}
    \label{fig:exp-stage1-generative}
\end{figure}

\textbf{\obs{} Wider latents are not always better.} With mean span length $2$, accuracy peaks at $32.3\%$ at $d_z=64$ and drops to $14.5\%$ at $d_z=512$ (Figure~\ref{fig:exp-stage1-capacity}, left). A larger $d_z$ adds capacity, but also increases the complexity of the distribution that the DiT must fit.

\textbf{\obs{} Compression rate, $d_z$, and $L_{\max}$ must be tuned jointly.} Even with the best configuration at each compression level, accuracy still falls from $32.3\%$ (span length $2$) to $25.2\%$ (span length $4$), as shown in Table~\ref{tab:exp-stage1-compression}. The effect of $L_{\max}$ (Figure~\ref{fig:exp-stage1-maxspan}, right) is governed by the truncation ratio (Figure~\ref{fig:exp-stage1-span-distribution}): when truncation is frequent, a smaller $L_{\max}$ keeps each latent simpler; once the cap mass drops to about $1\%$ (span length $4$, $L_{\max}=18$), preserving complete semantic spans dominates and accuracy rises to $21.7\%$.

\subsubsection{Generative Supervision of the Encoder}

\begin{lemma}[Generative Supervision Promotes Predictive Representations]
\label{lem:semantic}
When the supervised tokens are hidden from the encoder, the generative objective cannot be satisfied by copying the current targets. When the encoder is trainable, reducing this objective encourages the latent to retain information predictive of missing or future tokens. If the encoder is frozen, the objective only improves the decoder's use of information already present in the latent. See Appendix~\ref{app:semantic} for the proof.
\end{lemma}

\textbf{\obs{} Generative supervision is key; its form is secondary.} Pure reconstruction reaches only $15.7\%$, and adding the Span-AR or BERT-style objective raises accuracy to $25.2\%$ and $26.5\%$, respectively (Figure~\ref{fig:exp-stage1-generative}). With a frozen encoder, the BERT-style objective yields a $2.2$-point gain ($28.9\%\to31.1\%$), and unfreezing the encoder reaches $35.2\%$. Consistent with Lemma~\ref{lem:semantic}, the generative objective improves both the completion ability of the decoder and the latent itself.

\textbf{\obs{} Encoder and decoder should co-evolve.} With the MDM decoder, a frozen encoder yields $47.6\%$, end-to-end co-evolution raises accuracy to $53.5\%$, and further adding the generative loss reaches $58.2\%$ (Figure~\ref{fig:exp-stage1-generative}, right). What the latent should retain must be determined jointly with what the decoder can complete.

\subsection{Learning the Latent Generator}
\label{subsec:exp-stage2}

\subsubsection{VAE Update Policy and Gradient Detachment}

\begin{lemma}[Prior Fitting and Representation Learning Should Be Decoupled]
\label{lem:hacking}
The DiT Flow Matching objective only requires the current latent to be easier for the prior to fit; its gradient is not guaranteed to align with the token-side representation objective. Applying stop-gradient to the DiT path ensures that prior fitting does not directly rewrite the encoder, while preserving end-to-end encoder updates from the decoder token loss. See Appendix~\ref{app:hacking} for the proof.
\end{lemma}

\begin{figure}[t]
    \centering
    \includegraphics[width=\textwidth]{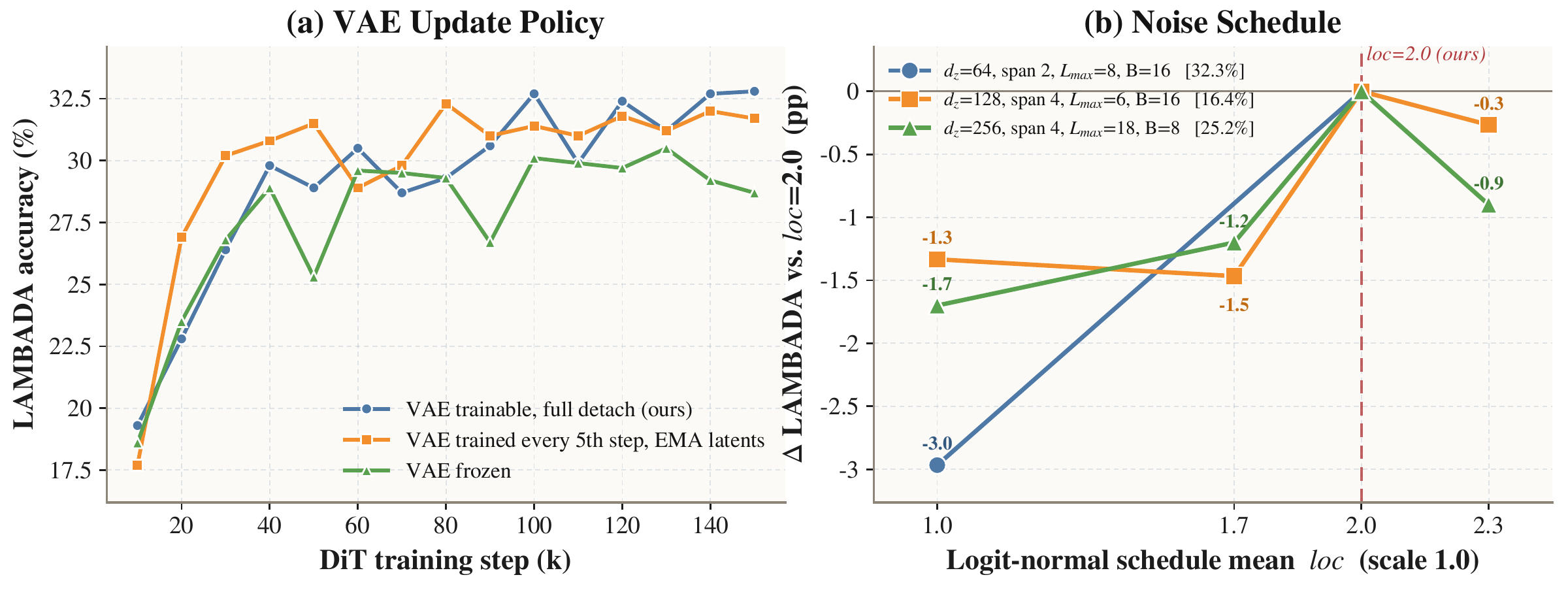}
    \caption{\textbf{VAE update policy and noise schedule.} Left: continuing to train the VAE while detaching the DiT gradient performs best. Right: differences relative to $\mathrm{loc}=2.0$; all three configurations are optimal at $2.0$.}
    \label{fig:exp-stage2-update}
    \label{fig:exp-stage2-noise}
\end{figure}

\textbf{\obs{} Keep training the VAE, but detach the DiT gradient.} A trainable VAE reaches $32.8\%$, whereas a frozen VAE reaches only $28.7\%$ (Figure~\ref{fig:exp-stage2-update}, left). Allowing the DiT gradient to flow into the encoder yields a $2.2$-point lead early on, but this is overtaken after $80$k steps and becomes a $1.0$-point deficit at $140$k steps (Table~\ref{tab:exp-stage2-gradient}). This supports Lemma~\ref{lem:hacking}: the DiT gradient rewrites the latent to be easier to fit rather than more useful for text modeling, and detaching the gradient blocks this latent hacking.

\subsubsection{Noise-Schedule Calibration}

\begin{lemma}[Effective Noise Is Jointly Determined by the Latent and Diffusion]
\label{lem:noise}
At the same DiT timestep, the available latent signal is affected by both VAE posterior uncertainty and diffusion corruption. Consequently, changing the noise scale of the Stage~1 latent changes the effective signal-to-noise ratio induced by the same DiT schedule, so the noise schedule must be calibrated jointly with the latent distribution. See Appendix~\ref{app:noise} for the proof.
\end{lemma}

\textbf{\obs{} The noise schedule must match the representation.} All three configurations achieve their best results at $\mathrm{loc}=2.0$, and deviating from it costs $0.3$--$3.0$ points, with the magnitude varying with $d_z$ and the compression level (Figure~\ref{fig:exp-stage2-noise}, right). By Lemma~\ref{lem:noise}, the schedule should be recalibrated when the representation changes.

\subsubsection{Block Size and Generator Capacity}

\begin{figure}[t]
    \centering
    \includegraphics[width=\textwidth]{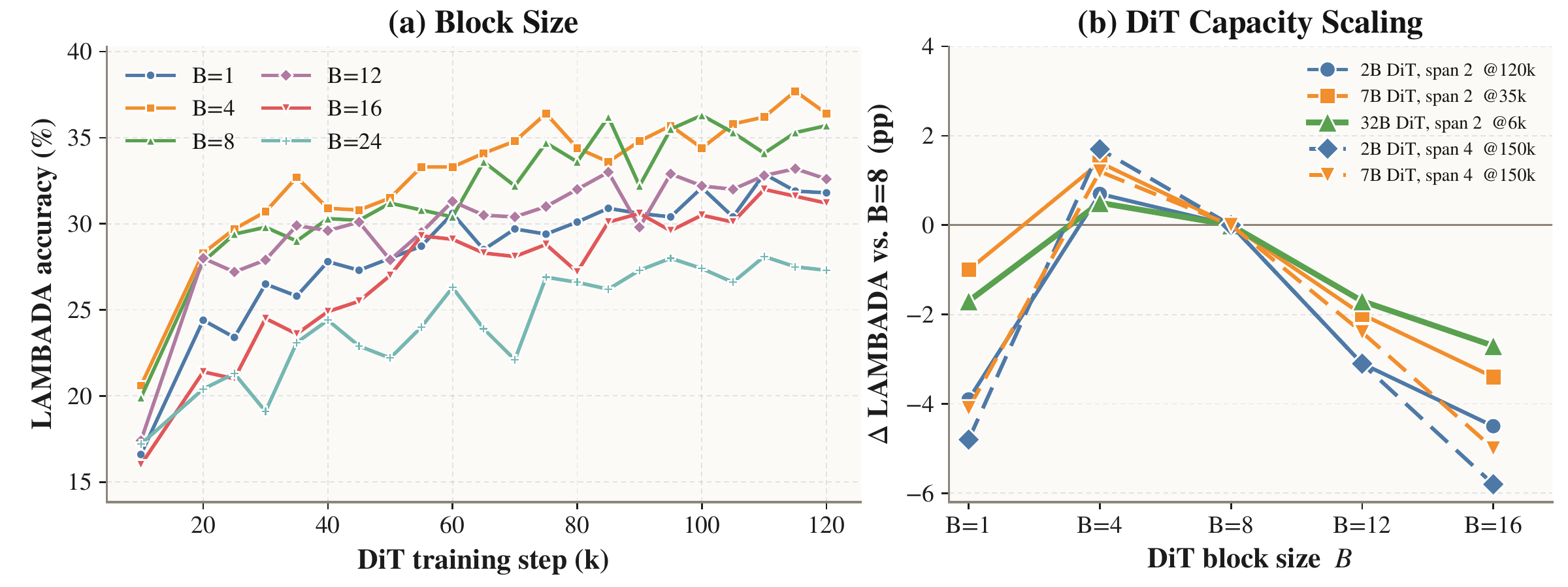}
    \caption{\textbf{DiT block size and capacity.} Left: moderate block sizes are optimal. Right: larger DiTs yield flatter curves.}
    \label{fig:exp-stage2-block}

    \vspace{0.5em}
    \includegraphics[width=\textwidth]{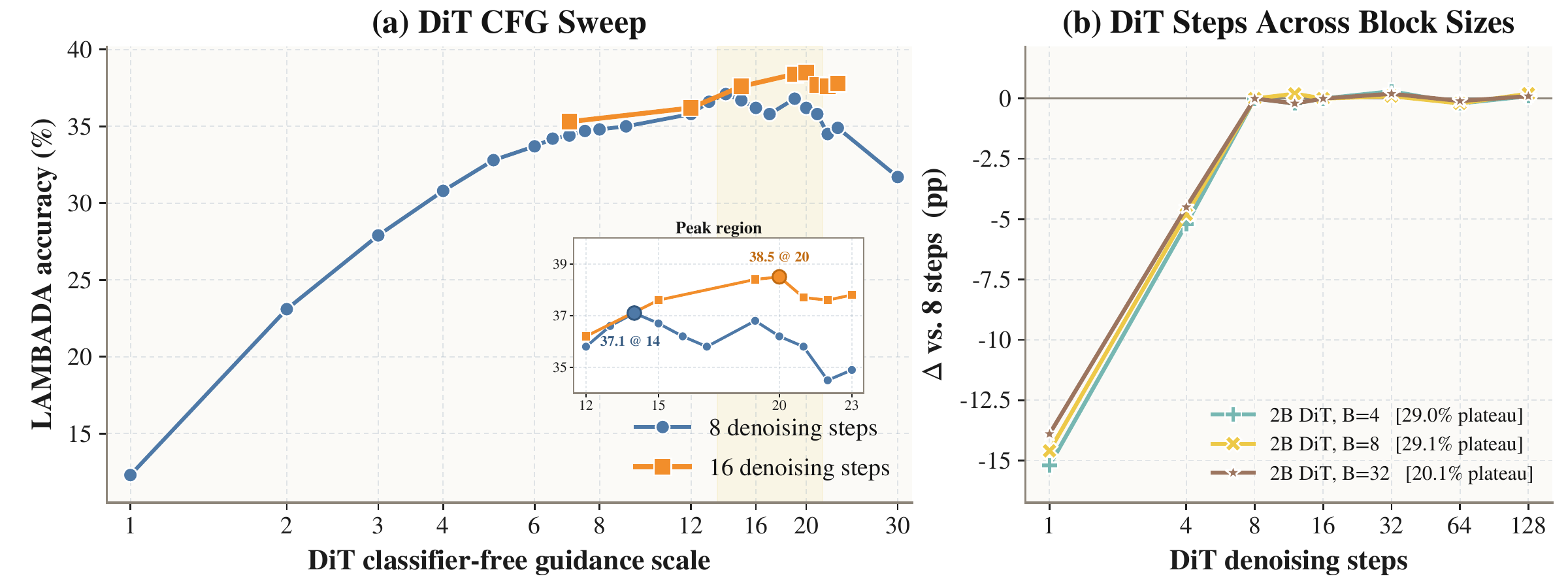}
    \caption{\textbf{DiT inference.} Left: generator CFG sweep. Right: the DiT saturates at $8$ steps.}
    \label{fig:exp-stage2-steps}
\end{figure}

\textbf{\obs{} Moderate blocks are best, and larger DiTs tolerate larger blocks.} For the 2B DiT, $B=4$ is optimal ($36.4\%$), whereas $B=1$ and $B=24$ reach only $31.8\%$ and $27.3\%$, respectively (Figure~\ref{fig:exp-stage2-block}, left). In the right panel of Figure~\ref{fig:exp-stage2-block}, the slope of the decline for $B>8$ becomes monotonically flatter with DiT size: the loss from $B=8\!\to\!16$ decreases from $4.5$ points for 2B to $3.4$ points for 7B and $2.7$ points for 32B, and the same trend holds at mean span length $4$. A larger DiT can jointly model more latents within a single block, which enlarges the range of block sizes it can use effectively.

\subsubsection{Generator Guidance and Inference Steps}

\textbf{\obs{} The generator needs strong CFG and about $8$ steps.} At $8$ steps, accuracy rises from $12.2\%$ at $\gamma_{\mathrm{gen}}=1$ to $37.1\%$ at $\gamma_{\mathrm{gen}}=14$; at $16$ steps, the peak moves to $\gamma_{\mathrm{gen}}=20$ ($38.5\%$) (Figure~\ref{fig:exp-stage2-steps}, left). Using $1$ and $4$ steps loses about $14$ and $5$ points, respectively, whereas results between $8$ and $128$ steps differ by at most $0.3$ points (Figure~\ref{fig:exp-stage2-steps}, right).

\subsection{Balancing the Reconstruction--Generation Decoder}
\label{subsec:exp-stage3}

\begin{lemma}[Latent-Condition Strength Determines Residual Generation]
\label{lem:decoder}
Given the same discrete context, a more reliable latent condition provides more information and leaves less residual uncertainty for the decoder to generate on its own; when the latent is corrupted or removed, the residual uncertainty increases accordingly. Thus, mixed \textsf{cond}/\textsf{noise}/\textsf{drop} training covers the continuum from conditional reconstruction to autonomous generation. See Appendix~\ref{app:decoder} for the proof.
\end{lemma}

\subsubsection{Decoder Type}

\textbf{\obs{} Decoder factorization matters little.} At every checkpoint, the gap between them is at most $1.4$ points, and the leader alternates during training (Table~\ref{tab:exp-stage3-decoder-type}): AR-MTP leads by $0.9$ points at $50$k steps, whereas MDM leads by $0.7$ and $1.4$ points at $100$k and $150$k steps. This indicates that, given a shared latent--token interface, the specific decoder factorization is not the decisive factor for performance, and the hierarchical framework applies to both; subsequent experiments use the MDM decoder.

\subsubsection{Decoder Data Partition}

\begin{figure}[t]
    \centering
    \includegraphics[width=\textwidth]{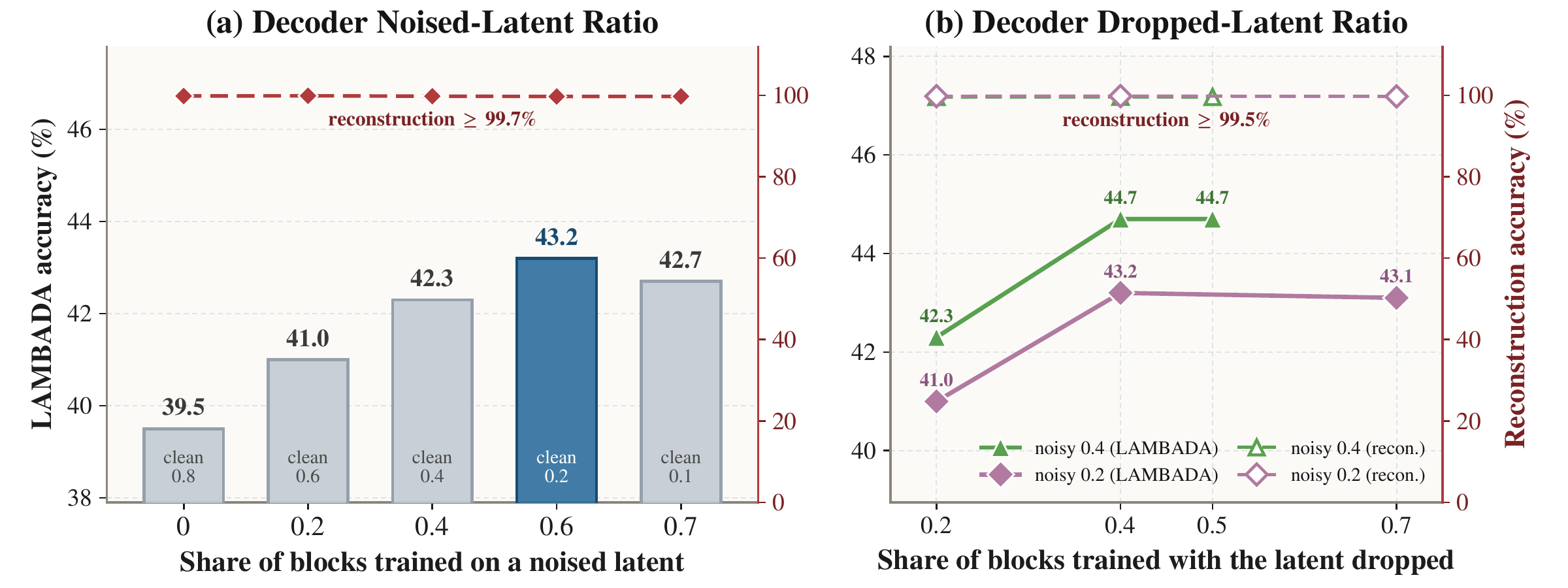}
    \caption{\textbf{Decoder data partition.} Left: varying the noisy-to-clean ratio. Right: varying the drop ratio. Reconstruction accuracy (right axis) stays above $99.5\%$ throughout; $0.4/0.4/0.2$ gives the best balance.}
    \label{fig:exp-stage3-partition}
\end{figure}

\begin{table}[t]
    \centering
    \caption{\textbf{Whether the DiT gradient flows into the encoder.} Gradient injection leads early, but is overtaken by full detachment in the middle and late stages of training.}
    \label{tab:exp-stage2-gradient}
    \small
    \setlength{\tabcolsep}{10pt}
    \renewcommand{\arraystretch}{1.10}
    \begin{tabular}{lccc}
        \toprule
        \textbf{Gradient policy} & \textbf{20k steps} & \textbf{80k steps} & \textbf{140k steps} \\
        \midrule
        Full detachment (ours)                   & 21.4 & \textbf{30.4} & \textbf{33.1} \\
        DiT velocity loss injected every 10 steps & \textbf{23.6} & 29.9 & 32.1 \\
        \midrule
        Difference (ours $-$ injection)          & $-2.2$ & $+0.5$ & $+1.0$ \\
        \bottomrule
    \end{tabular}
\end{table}

\begin{figure}[t]
    \centering
    \includegraphics[width=\textwidth]{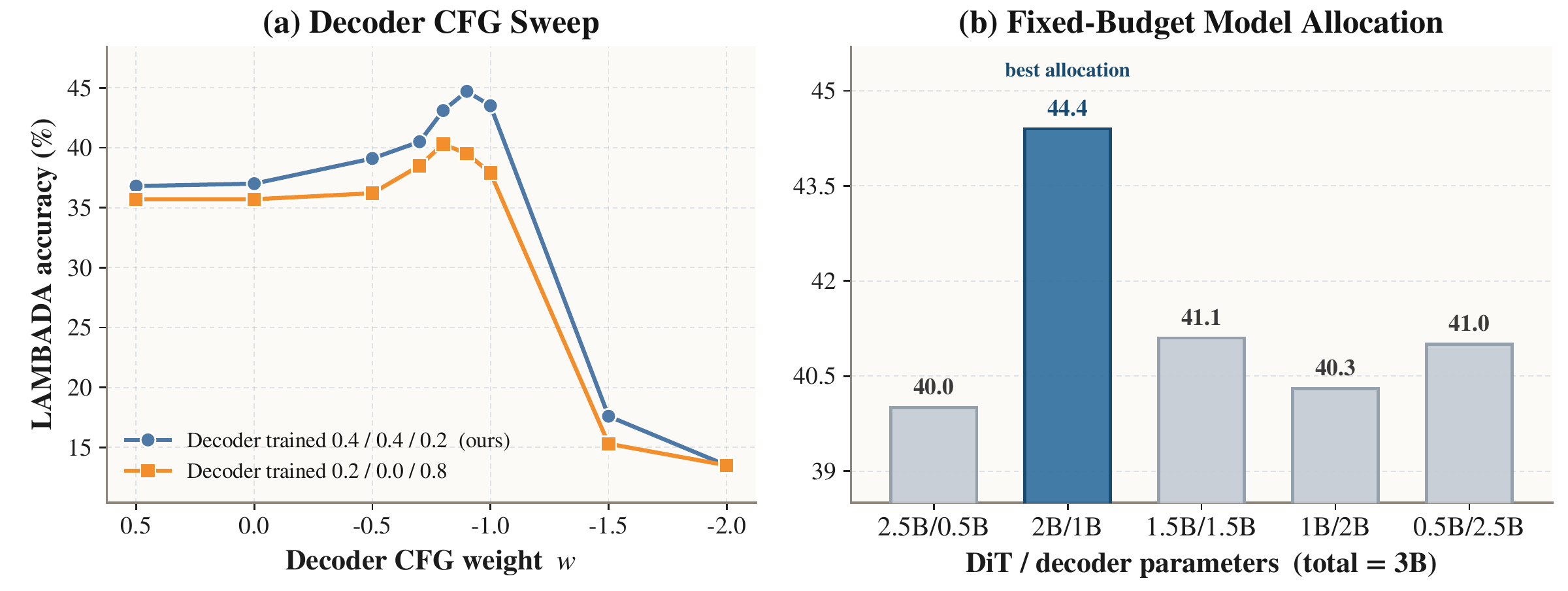}
    \caption{\textbf{Decoder guidance and parameter allocation.} Left: the balanced decoder is optimal at $w_{\mathrm{dec}}=-0.9$. Right: with 3B total parameters, a $2$B DiT with a $1$B decoder is optimal.}
    \label{fig:exp-stage3-cfg}
    \label{fig:exp-stage3-allocation}
\end{figure}

\begin{table}[t]
    \centering
    \caption{\textbf{Decoder factorization.} The two decoders perform similarly.}
    \label{tab:exp-stage3-decoder-type}
    \small
    \setlength{\tabcolsep}{10pt}
    \renewcommand{\arraystretch}{1.10}
    \begin{tabular}{lccc}
        \toprule
        \textbf{Decoder} & \textbf{50k steps} & \textbf{100k steps} & \textbf{150k steps} \\
        \midrule
        AR-MTP decoder & \textbf{39.7} & 43.7 & 44.1 \\
        MDM decoder    & 38.8 & \textbf{44.4} & \textbf{45.5} \\
        \midrule
        Difference (MDM $-$ AR-MTP) & $-0.9$ & $+0.7$ & $+1.4$ \\
        \bottomrule
    \end{tabular}
\end{table}

\textbf{\obs{} Train mostly on noisy and dropped conditions, but keep clean ones.} Reconstruction accuracy stays above $99.5\%$ in all configurations, even when only $10\%$ of blocks are clean (Figure~\ref{fig:exp-stage3-partition}), so most of the data can be devoted to generation under degraded conditions. With the drop ratio fixed at $0.2$, shifting noisy/clean from $0.0/0.8$ to $0.6/0.2$ raises LAMBADA accuracy from $39.5\%$ to $43.2\%$; raising the drop ratio from $0.2$ to $0.4$ adds another $2.2$--$2.4$ points, while pushing further brings no gain (noisy $0.7$: $42.7\%$; drop $0.7$: $43.1\%$). The selected $0.4/0.4/0.2$ (drop/noisy/clean) ties for the best result of $44.7\%$ while retaining more clean data, confirming the reconstruction--generation continuum in Lemma~\ref{lem:decoder}.

\subsubsection{Decoder Guidance}

\textbf{\obs{} Decoder CFG must match the training mixture.} The balanced decoder reaches $44.7\%$ at $w_{\mathrm{dec}}=-0.9$ and collapses rapidly at more negative weights, whereas the clean-heavy decoder reaches at best $40.3\%$ (Figure~\ref{fig:exp-stage3-cfg}, left), indicating that inference-time guidance cannot compensate for an imbalanced training distribution.

\subsubsection{Decoder Capacity}

\begin{figure}[t]
    \centering
    \includegraphics[width=\textwidth]{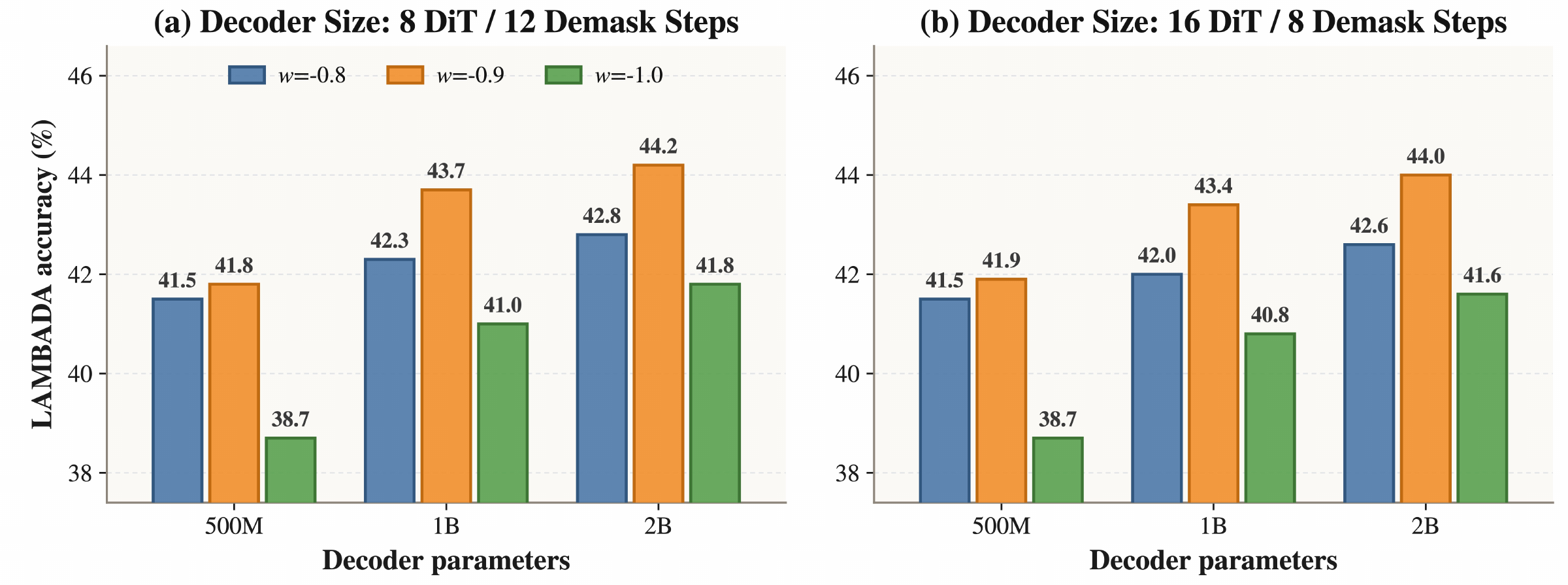}
    \caption{\textbf{Scaling the decoder with a fixed DiT.} Accuracy improves monotonically with decoder size.}
    \label{fig:exp-stage3-size}
\end{figure}

\begin{figure}[t]
    \centering
    \includegraphics[width=\textwidth]{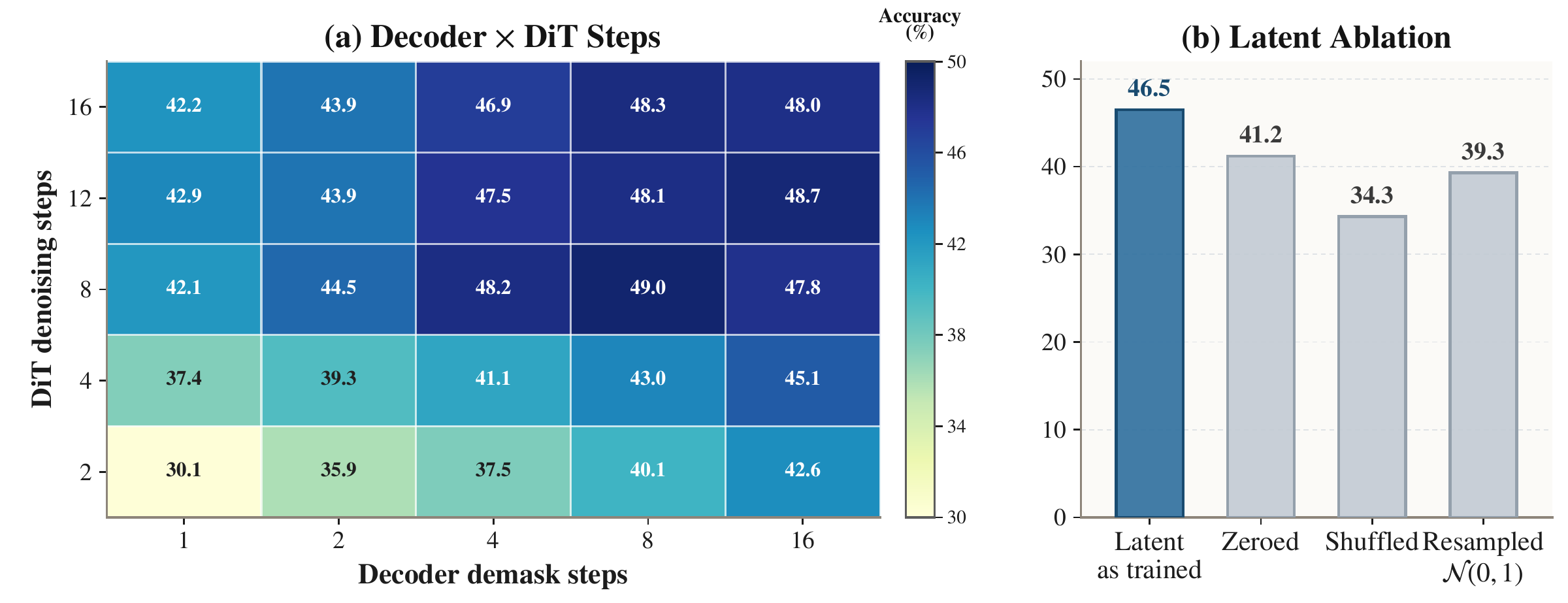}
    \caption{\textbf{Decoder inference and latent ablation.} Left: $8$/$8$ steps are optimal. Right: corrupting the latent substantially degrades performance, with shuffling being the most harmful.}
    \label{fig:exp-stage3-demask}
\end{figure}

\textbf{\obs{} Under a fixed budget, the DiT matters more.} With the DiT fixed, increasing the decoder from 500M to 2B yields monotonic gains of $2.1$--$2.4$ points (Figure~\ref{fig:exp-stage3-size}); under a total budget of 3B, a $2$B DiT with a $1$B decoder ($44.4\%$) outperforms the other allocations by $3.3$--$4.4$ points, and imbalance in either direction degrades performance (Figure~\ref{fig:exp-stage3-allocation}, right).

\subsubsection{Demasking Steps and Latent Use}

\textbf{\obs{} Demasking saturates at about $8$ steps.} $8$ DiT/$8$ demasking steps achieve the best result of $49.0\%$ (Figure~\ref{fig:exp-stage3-demask}, left); fewer than $8$ DiT steps cause a clear drop, while increasing demasking can partially compensate when the DiT uses too few steps. The $8$/$4$ setting ($48.2\%$) is only $0.8$ points below the optimum and serves as the default low-cost setting.

\textbf{\obs{} The decoder relies on aligned latents.} Zeroing, resampling, and shuffling the latent reduce accuracy from $46.5\%$ to $41.2\%$, $39.3\%$, and $34.3\%$, respectively (Figure~\ref{fig:exp-stage3-demask}, right). Zeroing corresponds to the \textsf{drop} state seen in training and therefore degrades more gracefully; shuffled latents are statistically realistic but semantically misaligned and cause the largest damage, showing that the decoder follows the latent content.

\subsection{Scaling Experiments}
\label{subsec:exp-scaling}

\subsubsection{Controlled Scaling Trends}
\label{subsec:exp-scaling-controlled}

\begin{figure}[t]
    \centering
    \includegraphics[width=\textwidth]{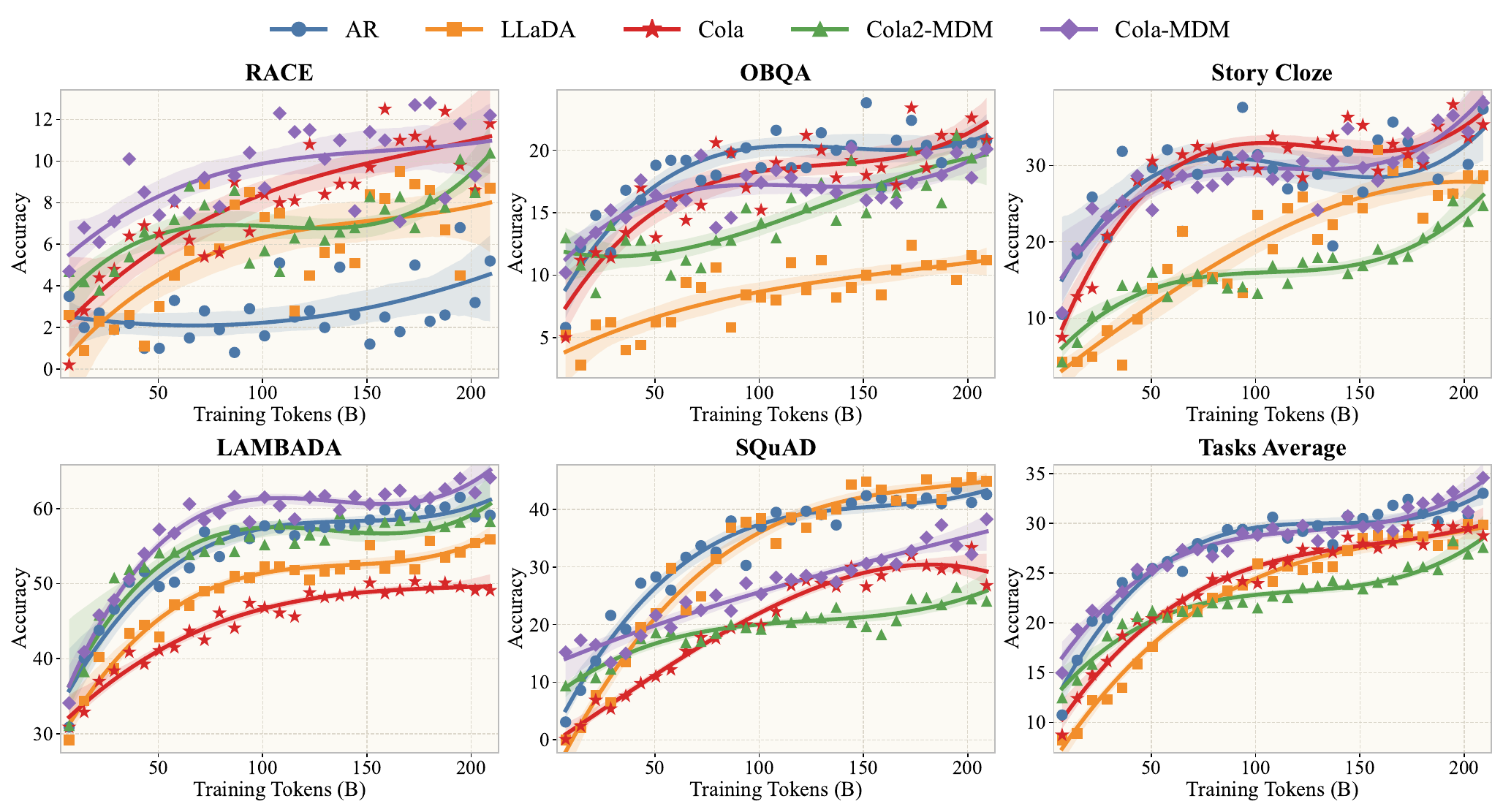}
    \caption{\textbf{Scaling trends with training tokens.} Cola-MDM achieves the strongest overall scaling; \method-MDM also exhibits strong and unsaturated scaling potential under compression.}
    \label{fig:exp-scaling}
\end{figure}

\begin{figure}[t]
    \centering
    \includegraphics[width=\textwidth]{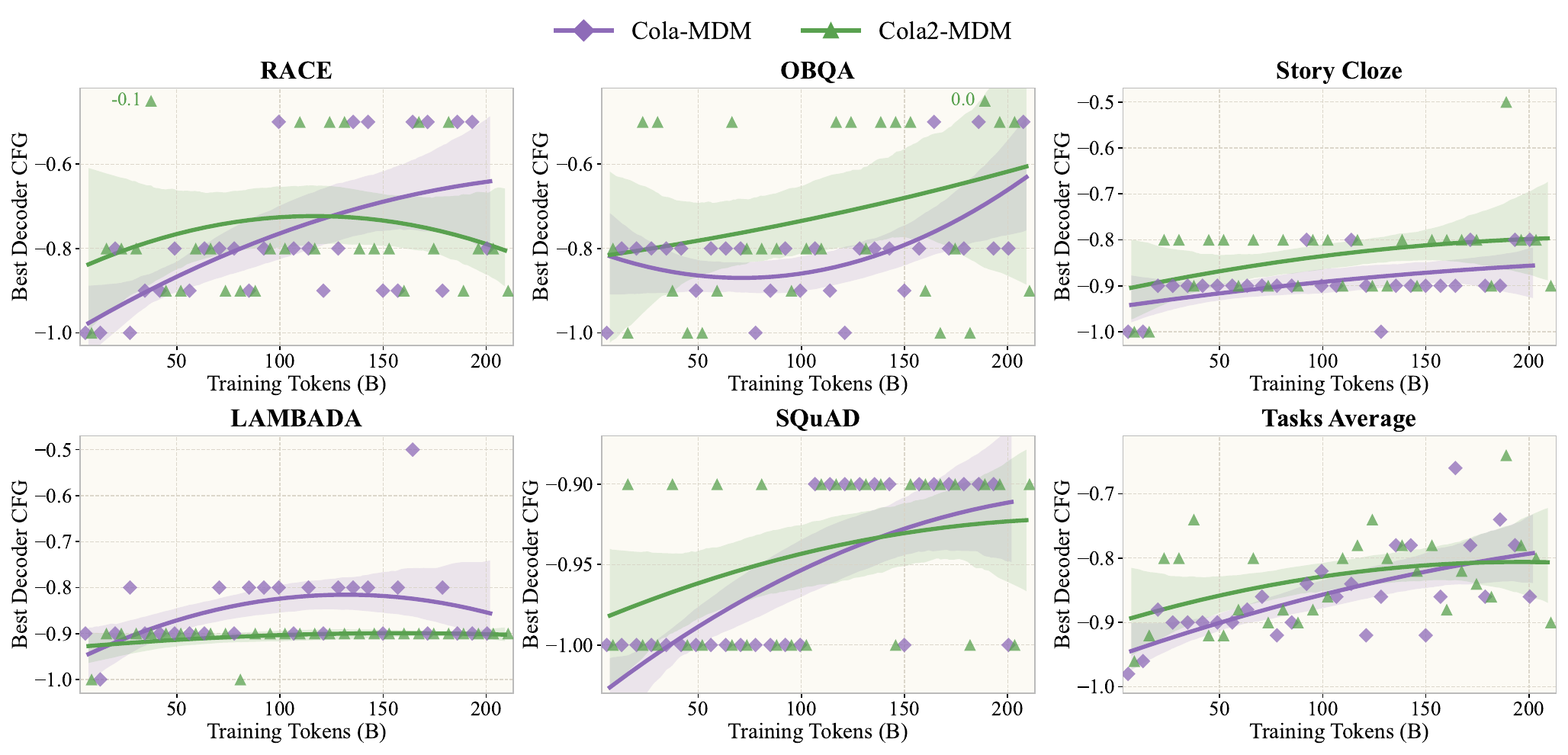}
    \caption{\textbf{Optimal decoder CFG weight during training.} For both models, the optimal weight remains within $[-1.0,-0.8]$, consistent with the $-0.9$ selected in Section~\ref{subsec:exp-stage3}, and drifts slowly toward $0$ during training; the drift of \method-MDM is more gradual, consistent with its slower convergence.}
    \label{fig:exp-scaling-cfg}
\end{figure}

Figure~\ref{fig:exp-scaling} extends training to about $210$B tokens and compares the models on RACE, OBQA, Story Cloze, LAMBADA, and SQuAD. Cola-MDM ($\rho=1$) and \method-MDM (with the compression VAE from Section~\ref{subsec:exp-stage1}) share the MDM decoder selected in Section~\ref{subsec:exp-stage3}, which separates the effect of the decoder design from that of dynamic compression.

\textbf{\obs{} The MDM decoder yields the strongest scaling.} Cola-MDM finally reaches a task average of about $34.5\%$, higher than AR (about $33\%$), and about $64\%$ on LAMBADA, about $15$ points above Cola DLM. The weakness of Cola DLM on generative tasks stems from its decoder rather than from the hierarchical framework.

\textbf{\obs{} Compression preserves semantic tasks, with costs on exact-token tasks.} \method-MDM matches AR on LAMBADA (about $61\%$) and leads substantially on RACE (about $10.5\%$ vs.\ $4.5\%$), with the gap concentrated on SQuAD and Story Cloze, which rely on verbatim reproduction.

\textbf{\obs{} The compressed variant converges more slowly but keeps closing the gap.} Since each latent summarizes more tokens, the latent DiT converges more slowly than the discrete generator. As training continues, the latent condition becomes increasingly useful: after about $150$B tokens, the average gap to LLaDA shrinks from about $4$ to $1.5$ points without saturating. The mixed decoder setting $w_{\mathrm{dec}}=-0.9$ ($44.7\%$) also outperforms both the latent-free ($43.5\%$) and fully conditioned ($37.0\%$) settings, confirming that latent semantics and discrete generation are complementary.

\subsubsection{Post-Pretraining Benchmark Comparison}
\label{subsec:exp-pretrained}

The final checkpoint evaluated here uses a separate large-scale recipe from the controlled scaling study. It is trained on OLMo~2 pretraining data retokenized with the Qwen3-4B tokenizer, using length-$4096$ sequences. The decoder and the VAE encoder embedding are initialized from Qwen3-4B-Base. During decoder training, masked positions are represented by \texttt{[MASK]} inputs, while decoder self-attention remains causal. Stage~I trains the dynamic-compression VAE on $52.43$B source tokens. Stage~II jointly updates the VAE and an approximately 4B-parameter latent DiT, with detached VAE latents, on an additional $125.83$B source tokens. Stage~III freezes the Stage-II EMA VAE and continues training only the DiT on $167.77$B source tokens. The total is $346.03$B source tokens. Architecture, optimization, objective, EMA, and state-transition details are provided in Appendix~\ref{app:final-pretraining}.

Inference uses a low-noise-biased eight-step Euler DiT schedule, $t_i=T(1-i/8)^2$ for $i=0,\ldots,8$, with DiT CFG set to $10$. Decoder CFG is $-0.9$, each decoding call generates one latent span, dynamic duration and latent feedback are enabled, and decoding is greedy.

We compare this checkpoint with representative 7B--8B pretrained models across eight downstream tasks. Each \method task score reported below is the average of three sampled evaluations. Because the models' training recipes, tokenizers, and token-counting conventions differ, this comparison contextualizes post-pretraining performance rather than isolating the effect of the generation paradigm.

\begin{table}[t]
    \centering
    \caption{\textbf{Post-pretraining evaluation.} Accuracy (\%) on eight downstream tasks, with tasks grouped by evaluation focus and models by generation paradigm. \method uses $346.03$B source tokens; the models are not compute-matched.}
    \label{tab:exp-pretrained-comparison}
    \small
    \setlength{\tabcolsep}{4pt}
    \renewcommand{\arraystretch}{1.10}
    \begin{tabular*}{\textwidth}{@{\extracolsep{\fill}}lccccc}
        \toprule
        & \multicolumn{3}{c}{\textbf{Autoregressive}}
        & \multicolumn{1}{c}{\textbf{Discrete diffusion}}
        & \multicolumn{1}{c}{\textbf{Continuous diffusion}} \\
        \cmidrule(lr){2-4}\cmidrule(lr){5-5}\cmidrule(lr){6-6}
        \textbf{Task} & \textbf{Llama-3-8B} & \textbf{Qwen2.5-7B} & \textbf{Qwen3-8B-Base} & \textbf{LLaDA-8B-Base} & \textbf{\method} \\
        \midrule
        \textit{Pretraining tokens} & 15T+ & 18T & ${\sim}$36T & 2.3T & 346B \\
        \midrule
        \multicolumn{6}{l}{\textit{Reading comprehension}} \\
        \quad RACE        & 75.77 & 89.95 & 90.65 & 87.40 & 69.91 \\
        \quad NarrativeQA & 42.70 & 32.90 & 37.30 & 36.00 & 32.70 \\
        \midrule
        \multicolumn{6}{l}{\textit{Commonsense reasoning}} \\
        \quad HellaSwag   & 39.21 & 75.67 & 68.13 & 69.65 & 41.75 \\
        \quad StoryCloze  & 79.49 & 95.09 & 97.44 & 95.94 & 80.42 \\
        \midrule
        \multicolumn{6}{l}{\textit{Language modeling}} \\
        \quad LAMBADA     & 75.70 & 71.82 & 71.20 & 67.90 & 63.00 \\
        \midrule
        \multicolumn{6}{l}{\textit{Knowledge-oriented QA}} \\
        \quad ARC-Easy    & 90.86 & 96.00 & 97.35 & 93.85 & 87.62 \\
        \quad TruthfulQA  & 63.04 & 69.03 & 72.58 & 64.87 & 54.10 \\
        \quad MMLU        & 63.19 & 70.16 & 72.31 & 61.27 & 51.32 \\
        \midrule
        \textbf{8-task average} & \textbf{66.25} & \textbf{75.08} & \textbf{75.87} & \textbf{72.11} & \textbf{60.10} \\
        \bottomrule
    \end{tabular*}
\end{table}

\textbf{\obs{} \method shows strong scaling potential under a limited training budget.} With only $346.03$B source tokens and before convergence, \method reaches an eight-task average of $60.10$, only $6.15$ points behind Llama-3-8B ($66.25$), which uses at least $15$T tokens---over $43{\times}$ more data. \method also surpasses Llama-3-8B on HellaSwag ($41.75$ vs.\ $39.21$) and StoryCloze ($80.42$ vs.\ $79.49$), while reaching $87.62$ on ARC-Easy. These results indicate substantial scaling headroom and successful transfer of the Qwen-initialized causal decoder's text generation capability. More broadly, continuous-latent hierarchical modeling can integrate effectively with established token-space generation paradigms, including autoregressive and masked-diffusion decoders.

\subsection{Summary}

The results support the design of \method in three respects. \textbf{Semantic compression:} dynamic segmentation and generative supervision keep the latents strongly semantic at high compression rates; the former slightly surpasses the uncompressed model with an appropriate $L_{\max}$, and the latter raises accuracy from $15.7\%$ to $26.7\%$. \textbf{Reconstruction--generation balance:} mixed training dominated by noisy and dropped conditions enables the decoder both to faithfully use the latent (shuffling the latent costs $12.2$ points) and to autonomously generate the residual information. \textbf{Continuous--discrete co-adaptation:} the information placed in the latent should be determined by the decoder; co-training the encoder and decoder yields a $5.9$-point gain, and both continue to benefit from more training data, giving the hierarchical model scaling trends comparable to or better than AR and LLaDA.

\FloatBarrier
\endgroup

\section{Limitations \& Future Prospects}
\label{sec:future-work}

Although our results provide initial evidence for the feasibility, competitiveness, and scaling potential of \method, we view the current system as a starting point rather than a finished endpoint. First, the model remains far from convergence, especially the latent DiT, while the present training scale is limited by available resources. Larger models, longer training, and broader data and compute budgets are therefore needed to establish its performance ceiling. Second, compression reduces the dominant linear and attention FLOPs of the latent prior in theory, but end-to-end generation speed also depends on the decoder, CFG branches, latent feedback, and hardware utilization. Infrastructure co-design---including specialized kernels, cache management, and parallel scheduling---is needed to translate theoretical savings into practical speedups. Third, \method already shows compatibility with autoregressive and masked-diffusion decoders, but their co-training objectives, inference trade-offs, and optimal division of labor remain underexplored. Finally, the continuous latent space may provide a shared interface between text and other continuous modalities such as vision, opening a path toward unified latent representations, cross-modal priors, and unified understanding and generation.

\section{Conclusion}
\label{sec:conclusion}

We presented \method, a hierarchical language modeling framework that unifies dynamic semantic compression with balanced discrete reconstruction and generation through continuous--discrete co-adaptation. Dynamic compression and generation-oriented supervision encourage compact, predictive latent representations, while the decoder realizes latent semantics and handles residual uncertainty from context. Token-prediction gradients drive encoder--decoder co-adaptation, whereas latent-prior fitting remains decoupled from representation learning. Systematic component analyses and controlled evaluations on five downstream tasks, scaled to about 210B training tokens, support these design principles and show competitive, unsaturated scaling. In a separate resource-constrained run, the final checkpoint trained on 346.03B source tokens achieves a 60.10 average over eight tasks and surpasses Llama-3-8B on HellaSwag and StoryCloze, although that baseline reports over $43{\times}$ more pretraining tokens. These results indicate successful transfer of the Qwen-initialized causal decoder's generation capability and substantial scaling headroom, especially because the latent DiT remains far from convergence. More broadly, \method shows that continuous-latent hierarchical modeling can integrate effectively with established token-space generators, including autoregressive and masked-diffusion decoders, rather than depending on a single decoder paradigm.

\clearpage

\bibliographystyle{plainnat}
\bibliography{main}

\clearpage

\beginappendix

\section{Detailed Derivations of the Probabilistic Model and Information Decomposition}
\label{app:basic}

\subsection{Conditional ELBO}

During training, the encoder $q_\phi(Z\mid C,Y)$ serves as the approximate posterior, where $\phi$ denotes the encoder parameters. By Jensen's inequality, the conditional log-likelihood in Eq.~\eqref{eq:main-prob} satisfies
\begin{equation}
    \log p(Y\mid C)
    \ge
    \E_{q_\phi(Z\mid C,Y)}
    \left[
      \log p_\psi(Y\mid C,Z)
      -
      \log\frac{q_\phi(Z\mid C,Y)}{p_\theta(Z\mid C)}
    \right].
\label{eq:app-elbo}
\end{equation}

Define the conditional aggregated posterior as
\begin{equation}
    \bar q_\phi(Z\mid C)
    =
    \int q_\phi(Z\mid C,Y)\data(Y\mid C)\dd Y.
    \label{eq:app-agg}
\end{equation}
Adding and subtracting $\log \bar q_\phi(Z\mid C)$ in the average KL term gives
\begin{equation}
\begin{aligned}
&\E\KL\!\left(q_\phi(Z\mid C,Y)\|p_\theta(Z\mid C)\right)\\
&=I(Y;Z\mid C)
+\E_C\KL\!\left(\bar q_\phi(Z\mid C)\|p_\theta(Z\mid C)\right).
\end{aligned}
\label{eq:app-kl-decomp}
\end{equation}

Eq.~\eqref{eq:app-kl-decomp} clarifies the probabilistic role of each of the three modules: the encoder determines what information the latent carries from the target text, the latent generator fits the conditional latent distribution induced by the encoder, and the decoder models tokens conditioned on the latent. Flow Matching is merely the particular solver used in this work to learn $p_\theta(Z\mid C)$, not part of the definition of the generative model itself.

\subsection{Information Identities and Semantic Decomposition}

Conditional mutual information is defined as
\begin{equation}
    I(Y;Z\mid C)
    =
    H(Y\mid C)-H(Y\mid C,Z),
    \label{eq:app-mi-def}
\end{equation}
and rearranging directly yields Eq.~\eqref{eq:main-info} in the main text.

If $S$ is a high-level semantic variable determined by $Y$, then $I(Y;Z\mid C)=I(S,Y;Z\mid C)$. Applying the chain rule of mutual information,
\begin{equation}
\begin{aligned}
    I(Y;Z\mid C)
    &=I(S,Y;Z\mid C)\\
    &=I(S;Z\mid C)+I(Y;Z\mid S,C),
\end{aligned}
\label{eq:app-chain}
\end{equation}
which yields the semantic decomposition in Eq.~\eqref{eq:main-decomp}. Under the compression constraint $K=\rho T$, we further define the semantic information density per latent position as
\begin{equation}
    \mathcal D_{\mathrm{sem}}
    =\frac{I(S;Z\mid C)}{K}.
    \label{eq:app-sem-density}
\end{equation}
This quantity need not be estimated explicitly during training; it distinguishes two different objectives: merely shortening the latent sequence versus concentrating semantic information in a limited number of latent positions.

If the target tokens are partitioned into a visible part $Y^{\mathrm{vis}}$ and a target part $Y^{\mathrm{tar}}$, then
\begin{equation}
\begin{aligned}
    I(Y^{\mathrm{tar}};Z\mid C,Y^{\mathrm{vis}})
    ={}&H(Y^{\mathrm{tar}}\mid C,Y^{\mathrm{vis}})\\
    &-H(Y^{\mathrm{tar}}\mid C,Y^{\mathrm{vis}},Z).
\end{aligned}
\label{eq:app-predictive-info}
\end{equation}
Generation-oriented latent supervision can therefore be interpreted as biasing the limited latent capacity toward predictive semantic information by reducing the conditional uncertainty of missing or future tokens given the latent. Section~\ref{app:semantic} gives a concrete interpretation of the fixed-span encoder-mask objective.

\subsection{KL, Compression Rate, and Latent Information Capacity}

Let $p_0(Z\mid C)$ be the reference distribution used for regularization in Stage~1. For any such reference distribution,
\begin{equation}
\begin{aligned}
&\E_{C,Y}\KL\!\left(q_\phi(Z\mid C,Y)\|p_0(Z\mid C)\right)\\
&=I(Y;Z\mid C)
+\E_C\KL\!\left(\bar q_\phi(Z\mid C)\|p_0(Z\mid C)\right)\\
&\ge I(Y;Z\mid C).
\end{aligned}
\label{eq:app-kl-upper}
\end{equation}

Thus, the VAE KL term not only constrains the latent distribution but also upper-bounds the amount of information the latent can carry. If the posterior approximately factorizes across latents and dimensions, and the average KL per latent dimension stays roughly at $\bar r$, then the total KL scales approximately with $K d_z$. Since $K=\rho T$, the information capacity available per token follows the approximate scaling
\begin{equation}
    \frac{1}{T}I(Y;Z\mid C)
    \lesssim
    \rho d_z\bar r.
    \label{eq:app-rho-capacity}
\end{equation}

\method normalizes the KL per latent element so that the regularization pressure on each latent dimension remains as comparable as possible across different $K$ and $d_z$. Changing $\rho$ can therefore be understood more directly as changing the number of latent channels available per token. Dynamic segmentation further determines where these limited latent positions are allocated under a fixed target compression rate: the compression rate controls \emph{how many} representation positions exist, and semantic boundaries determine \emph{where} they are placed.

\section{Information-Theoretic Interpretation of Generation-Oriented Latent Supervision}
\label{app:semantic}

The $\Lcal_{\mathrm{gen}}$ term in the main text is not an additional inference-time sampler, but a fixed-span encoder-mask objective applied directly to the encoder. Its constraint is that a supervised token cannot simply rely on a complete current latent observation in one-to-one correspondence with itself, but must be predicted from a context in which its input is missing. The objective therefore specifically increases the \emph{predictive} information of the latent rather than merely its local reconstructability.

\subsection{Fixed-Span Encoder-Mask / BERT-Style Objective}

Let $M$ denote a randomly selected set of spans, $Y_M$ the text in these spans that is masked on the encoder side, $Y_{\bar M}$ the remaining visible text, and $Z^M$ the latent obtained by encoding the masked input under the fixed original span layout. Since the encoder cannot observe the ground-truth $Y_M$ in this forward pass, we have, given $(C,M)$, the Markov chain
\begin{equation}
    Y_M
    \longrightarrow
    (C,Y_{\bar M})
    \longrightarrow
    Z^M.
    \label{eq:app-mask-markov}
\end{equation}
By the data processing inequality,
\begin{equation}
    I(Y_M;Z^M\mid C,M)
    \le
    I(Y_M;Y_{\bar M}\mid C,M).
    \label{eq:app-mask-dpi}
\end{equation}
Hence the objective cannot be satisfied by directly copying the masked tokens; it can only exploit information in the visible context that is statistically related to the target spans. For any predictor $p_\eta(Y_M\mid C,Z^M,M)$, the cross-entropy can be written as
\begin{equation}
\begin{aligned}
    \Lcal_{\mathrm{gen}}^{\mathrm{mask}}
    ={}&H(Y_M\mid C,Z^M,M)\\
    &+\E\KL\!\left(
      q(Y_M\mid C,Z^M,M)
      \|p_\eta(Y_M\mid C,Z^M,M)
    \right).
\end{aligned}
\label{eq:app-mask-ce}
\end{equation}
If the predictor is conditionally optimal, then
\begin{equation}
\begin{aligned}
    (\Lcal_{\mathrm{gen}}^{\mathrm{mask}})^\star
    &=H(Y_M\mid C,Z^M,M)\\
    &=H(Y_M\mid C,M)-I(Y_M;Z^M\mid C,M).
\end{aligned}
\label{eq:app-mask-final}
\end{equation}
Therefore, when the encoder is trainable, reducing this objective drives the latent to contain more predictive information about the missing spans. The current MDM implementation uses fixed-span encoder-mask MDM: the first clean forward pass determines the hard spans, and the second pass replaces only the encoder observations without re-segmenting. Functionally, it thus imposes BERT-style generation pressure, but it is not equivalent to arbitrary forms of the span-BERT loss used in earlier implementations.

\section{Latent Hacking and Gradient Decoupling in Compressed Representations}
\label{app:hacking}

This appendix proves Lemma~\ref{lem:hacking}, and first distinguishes two kinds of ``joint training.'' We deliberately retain the gradient from the discrete generator decoder to the encoder, because token reconstruction and generation are precisely the discrete realization that the latent must ultimately serve. The risk discussed here refers only to allowing the DiT Flow Matching prior loss to backpropagate into the encoder. That loss optimizes how easily the current latent can be fit by the prior solver, which is not equivalent to the reconstructability or generability of the discrete semantics. Below, we first analyze the potential conflict between the two gradients, and then discuss representation drift under fixed capacity and local decoder geometry. These analyses do not require compression to necessarily shrink the valid region, nor do they constitute a general guarantee of training collapse.

\subsection{Joint Objective and Gradient Conflict}

Collect the Stage~1 objectives that shape the representation into
\begin{equation}
    \Lcal_{\mathrm{rep}}(\phi,\psi)
    =
    \Lcal_{\mathrm{bal}}
    +\lambda_{\mathrm{gen}}\Lcal_{\mathrm{gen}}
    +\lambda_{\mathrm{dur}}\Lcal_{\mathrm{dur}}
    +\lambda_{\mathrm{KL}}\Lcal_{\mathrm{KL}},
    \label{eq:app-rep-objective}
\end{equation}
where $\phi$ and $\psi$ denote the encoder and decoder parameters, respectively; both $\Lcal_{\mathrm{bal}}$ and $\Lcal_{\mathrm{gen}}$ update the encoder through the token-prediction path of the discrete generator decoder. If the gradient of the DiT Flow Matching loss is fully allowed to reach the encoder, the encoder effectively optimizes
\begin{equation}
    \Lcal_{\mathrm{joint}}
    =
    \Lcal_{\mathrm{rep}}
    +\lambda_{\mathrm{FM}}\Lcal_{\mathrm{FM}}.
    \label{eq:app-joint-objective}
\end{equation}
Define
\begin{equation}
    g_{\mathrm{rep}}=\nabla_\phi\Lcal_{\mathrm{rep}},
    \qquad
    g_{\mathrm{FM}}=\nabla_\phi\Lcal_{\mathrm{FM}}.
    \label{eq:app-two-gradients}
\end{equation}
A single SGD update is
\begin{equation}
    \phi^{+}
    =
    \phi-\eta\big(g_{\mathrm{rep}}+\lambda_{\mathrm{FM}}g_{\mathrm{FM}}\big).
    \label{eq:app-joint-update}
\end{equation}
A first-order expansion of $\Lcal_{\mathrm{rep}}$ gives
\begin{equation}
    \Delta\Lcal_{\mathrm{rep}}
    \approx
    -\eta\|g_{\mathrm{rep}}\|_2^2
    -\eta\lambda_{\mathrm{FM}}
    \langle g_{\mathrm{rep}},g_{\mathrm{FM}}\rangle .
\label{eq:app-gradient-conflict}
\end{equation}
Therefore, whenever
\begin{equation}
    \langle g_{\mathrm{rep}},g_{\mathrm{FM}}\rangle
    <
    -\frac{\|g_{\mathrm{rep}}\|_2^2}{\lambda_{\mathrm{FM}}},
    \label{eq:app-conflict-condition}
\end{equation}
the DiT loss can decrease while the representation objective locally increases. This is the most basic condition for unstable joint training: the two gradients are not inherently aligned.

\subsection{The DiT-to-Encoder Gradient Is Not a Semantic Gradient}

Let the encoder produce the clean latent via reparameterization,
\begin{equation}
    z_0=g_\phi(x,\epsilon_0),
    \qquad
    \epsilon_0\sim\Normal(0,I).
    \label{eq:app-hack-z0}
\end{equation}
A general v-prediction / Flow Matching process can be written as
\begin{equation}
    z_t=\alpha_tz_0+\sigma_t\epsilon,
    \qquad
    v_t^\star=a_tz_0+b_t\epsilon,
    \qquad
    \epsilon\sim\Normal(0,I),
    \label{eq:app-hack-forward}
\end{equation}
and we define
\begin{equation}
    r
    =
    v_\theta(z_t,t,c)-v_t^\star,
    \qquad
    \ell_v=\|r\|_2^2,
    \label{eq:app-hack-residual}
\end{equation}
where $c$ denotes conditions such as previous latent blocks or the prompt representation. If $c$ is detached from $z_0$, the chain rule gives
\begin{equation}
    \nabla_{z_0}\ell_v
    =
    2\big(\alpha_tJ_{v,z_t}^{\top}-a_tI\big)r,
    \label{eq:app-hack-zgrad}
\end{equation}
where $J_{v,z_t}=\partial v_\theta/\partial z_t$. If the condition path also depends on $z_0$ and is not detached, an additional term appears:
\begin{equation}
    2J_{c,z_0}^{\top}J_{v,c}^{\top}r.
    \label{eq:app-hack-condgrad}
\end{equation}
Finally,
\begin{equation}
    \nabla_\phi\ell_v
    =
    \left(\nabla_{z_0}\ell_v\right)^{\top}
    \frac{\partial z_0}{\partial\phi}.
    \label{eq:app-hack-phigrad}
\end{equation}
These gradients are determined jointly by the denoising residual, the DiT Jacobian, the condition path, and the noise schedule; they contain neither a reconstruction likelihood nor a semantic objective. Therefore, $g_{\mathrm{FM}}$ serves to make the current latent representation easier for the prior solver to fit, rather than to guarantee that the decodable semantics of the latent remain intact.

\subsection{Lowering the FM Loss Admits a ``Simpler Latent'' Shortcut}

Let $W=(z_t,t,c)$ and let the target velocity be $u_t$. For any vector field $v(W)$, the squared loss admits the conditional-expectation decomposition
\begin{equation}
\begin{aligned}
    \E\|v(W)-u_t\|_2^2
    ={}&
    \E\|v(W)-\E[u_t\mid W]\|_2^2\\
    &+
    \E\|u_t-\E[u_t\mid W]\|_2^2.
\end{aligned}
\label{eq:app-fm-biasvariance}
\end{equation}
Hence
\begin{equation}
    v^\star(W)=\E[u_t\mid W],
    \qquad
    \Lcal_{\mathrm{FM}}^\star
    =
    \E\!\left[\operatorname{Var}(u_t\mid z_t,t,c)\right].
    \label{eq:app-fm-minrisk}
\end{equation}
Eq.~\eqref{eq:app-fm-minrisk} shows that, when the encoder is also trainable, the FM loss can decrease in two ways: the DiT can better approximate the conditional mean, or the encoder can actively alter the latent geometry so that the conditional variance of the target velocity itself decreases. The latter path directly flattens semantically meaningful modes.

As a simple counterexample, let the semantic variable $S\in\{-1,+1\}$ be equiprobable and construct
\begin{equation}
    q_{\mathrm{sem}}(z\mid S)=\Normal(mSe_1,I),\quad m>0,
    \qquad
    q_{\mathrm{col}}(z\mid S)=\Normal(0,I).
    \label{eq:app-collapse-example}
\end{equation}
The former satisfies $I(S;Z_{\mathrm{sem}})>0$, whereas the latter satisfies $I(S;Z_{\mathrm{col}})=0$. However, in high-noise or weak-conditioning regions, transporting the bimodal semantic distribution is more complex, while the collapsed distribution is closer to the base Gaussian. It is therefore possible that
\begin{equation}
    \Lcal_{\mathrm{FM}}(q_{\mathrm{col}})
    <
    \Lcal_{\mathrm{FM}}(q_{\mathrm{sem}}),
    \qquad
    I(S;Z_{\mathrm{col}})
    <
    I(S;Z_{\mathrm{sem}}).
    \label{eq:app-collapse-shortcut}
\end{equation}
Consequently, a lower FM loss alone cannot rule out semantic collapse; under unrestricted joint training, collapse may even become an easier optimization direction.

\subsection{Information Constraints and Local Decoder Geometry under Compression}

In this subsection, $n$ denotes the number of tokens covered by a single latent $Z$, to distinguish it from the latent sequence length $K$ used elsewhere. Given the data distribution and context $C$, define the minimum conditional information required to reach a total reconstruction-loss threshold $\varepsilon$ as
\begin{equation}
\begin{aligned}
    R_n^\star(\varepsilon)
    ={}&\inf_{q(z\mid X_{1:n},C)}
    \Big\{I(X_{1:n};Z\mid C):\\
       &\qquad\inf_{p_\psi}
    \E[-\log p_\psi(X_{1:n}\mid Z,C)]\le\varepsilon\Big\}.
\end{aligned}
\label{eq:app-required-info}
\end{equation}
If the effective conditional information capacity of this latent is $C_n$, the information slack is
\begin{equation}
    \Delta_n=C_n-R_n^\star(\varepsilon).
    \label{eq:app-compression-slack}
\end{equation}
If additional tokens increase the required predictive information while $C_n$ does not increase accordingly, the slack may decrease. This depends on the data, the context, and the distortion criterion, and does not imply an unconditional monotonic relation among separately retrained models with different compression rates.

To illustrate the geometric constraint, fix the decoder, the target span $x_{1:n}$, and the context $c$. Under the causal factorization, let $p_i(z)=p_\psi(x_i\mid x_{<i},z,c)$, and let $\tau,\tau_i\in(0,1]$ denote probability thresholds. Define
\begin{equation}
    \mathcal A_{1:n}(\tau)
    =\left\{z:\prod_{i=1}^{n}p_i(z)\ge\tau\right\},
    \qquad
    \mathcal A_i(\tau_i)=\{z:p_i(z)\ge\tau_i\}.
    \label{eq:app-valid-region}
\end{equation}
Satisfying every per-token threshold is sufficient for satisfying the joint threshold, so
\begin{equation}
    \mathcal B_n:=\bigcap_{i=1}^{n}\mathcal A_i(\tau_i)
    \subseteq
    \mathcal A_{1:n}\!\left(\prod_{i=1}^{n}\tau_i\right).
    \label{eq:app-valid-intersection}
\end{equation}
With the existing $p_i$ and thresholds unchanged, adding one token constraint gives
\begin{equation}
    \mathcal B_{n+1}
    =\mathcal B_n\cap\mathcal A_{n+1}(\tau_{n+1})
    \subseteq\mathcal B_n.
    \label{eq:app-valid-nesting}
\end{equation}
This only states that the intersection of fixed constraints does not grow; it does not guarantee strict shrinkage, nor does it directly compare relearned latent geometries.

To exclude the trivial scale growth caused by summing losses, we use the per-token average loss below. With $\ell_i(z)=-\log p_i(z)$,
\begin{equation}
    \Lcal_{\mathrm{dec}}^{(n)}(z)
    =\frac1n\sum_{i=1}^{n}\ell_i(z).
    \label{eq:app-dec-k}
\end{equation}
Its local Hessian is
\begin{equation}
    H_n(z)=\nabla_z^2\Lcal_{\mathrm{dec}}^{(n)}(z)
    =\frac1n\sum_{i=1}^{n}\nabla_z^2\ell_i(z).
    \label{eq:app-hessian-k}
\end{equation}
For MDM, the same local analysis can take $\ell_i$ as the token denoising loss under a fixed masked state, without treating the product as the full MDM sequence likelihood. If the loss is twice continuously differentiable near $z$, then for a latent drift $\delta$,
\begin{equation}
\begin{aligned}
    \Lcal_{\mathrm{dec}}^{(n)}(z+\delta)
    -\Lcal_{\mathrm{dec}}^{(n)}(z)
    ={}&\nabla\Lcal_{\mathrm{dec}}^{(n)}(z)^\top\delta
       +\frac12\delta^\top H_n(z)\delta\\
       &+o(\|\delta\|_2^2).
\end{aligned}
\label{eq:app-dec-drift}
\end{equation}
Only when $z$ is close to a stationary point, the first-order and higher-order terms are negligible, and $\|H_n(z)\|_2>0$, the local quadratic perturbation scale corresponding to an allowed average-loss increase $\epsilon_d>0$ is approximately
\begin{equation}
    r_n^{\mathrm{quad}}(\epsilon_d)
    \simeq\sqrt{\frac{2\epsilon_d}{\|H_n(z)\|_2}}.
    \label{eq:app-safe-radius}
\end{equation}
This scale is neither the boundary of the actual valid region nor a guaranteed safe radius. More tokens need not increase the spectral norm of the average Hessian, and the actual loss change also depends on the drift direction. Only when the additional constraints increase the local curvature along the relevant directions can the decoder become more sensitive to drift along those directions.

\subsection{Prior-Induced Local Drift and Gradient Decoupling}

Considering only the DiT gradient component acting on the encoder, one parameter update is
\begin{equation}
    \delta\phi_{\mathrm{FM}}
    =-\eta\lambda_{\mathrm{FM}}g_{\mathrm{FM}}.
    \label{eq:app-fm-phi-drift}
\end{equation}
Let $J_{z,\phi}=\partial z_0/\partial\phi$ be the encoder Jacobian. Taking $g_{\mathrm{FM}}=\nabla_\phi\ell_v$ for the single-sample FM term, the induced first-order latent drift is
\begin{equation}
\begin{aligned}
    \delta z_{\mathrm{FM}}
    &\approx J_{z,\phi}\delta\phi_{\mathrm{FM}}\\
    &=-\eta\lambda_{\mathrm{FM}}
      J_{z,\phi}J_{z,\phi}^{\top}\nabla_{z_0}\ell_v.
\end{aligned}
\label{eq:app-fm-latent-drift}
\end{equation}
For batch updates, $g_{\mathrm{FM}}$ should be replaced by the corresponding aggregated parameter gradient. Fixing the current decoder and setting $z=z_0$, Eq.~\eqref{eq:app-dec-drift} gives
\begin{equation}
\begin{aligned}
    \Delta\Lcal_{\mathrm{dec}}^{\mathrm{FM}}
    :={}&\Lcal_{\mathrm{dec}}^{(n)}(z+\delta z_{\mathrm{FM}})
          -\Lcal_{\mathrm{dec}}^{(n)}(z)\\
    \approx{}&\nabla\Lcal_{\mathrm{dec}}^{(n)}(z)^\top\delta z_{\mathrm{FM}}
      +\frac12\delta z_{\mathrm{FM}}^\top H_n(z)\delta z_{\mathrm{FM}}.
\end{aligned}
\label{eq:app-fm-decoder-drift}
\end{equation}
Hence, the FM-induced drift may increase the token loss along sensitive directions of the decoder, with an effect that depends on the encoder Jacobian, the direction and magnitude of the drift, and the simultaneous adaptation of the decoder. This analysis characterizes the local risk of prior-induced representation drift; it does not establish a general relation between the compression rate and training collapse.

Full detachment directly removes the prior-solver gradient acting on the encoder:
\begin{equation}
    \nabla_\phi\Lcal_{\mathrm{joint}}
    \xrightarrow{\mathrm{full\ detach}}
    \nabla_\phi\Lcal_{\mathrm{rep}},
    \qquad
    \nabla_\theta\Lcal_{\mathrm{joint}}
    =\lambda_{\mathrm{FM}}\nabla_\theta\Lcal_{\mathrm{FM}}.
    \label{eq:app-full-detach}
\end{equation}
The VAE still learns under its own objectives, and the DiT still fits the aggregated posterior using samples from the current encoder. This operation guarantees that the path by which the prior loss directly rewrites the encoder is blocked, but does not guarantee that the full training process is necessarily stable.

\section{Joint Information-Theoretic Interpretation of VAE and DiT Noise}
\label{app:noise}

This appendix proves Lemma~\ref{lem:noise}. To analyze the relation between the Stage~1 latent noise and the Stage~2 diffusion noise, we write the clean latent in the simplified form
\begin{equation}
    Z=G+N_0,
    \label{eq:app-vae-noise}
\end{equation}
where $G$ denotes the signal associated with high-level semantics and $N_0$ denotes the VAE posterior noise. The DiT further constructs
\begin{equation}
    Z_t=\alpha_t Z+\sigma_t\varepsilon,
    \qquad
    \varepsilon\sim\Normal(0,I).
    \label{eq:app-dit-noise}
\end{equation}

Under a Gaussian approximation, assume that $G\sim\Normal(0,\Sigma_G)$ and $N_0\sim\Normal(0,\Sigma_0)$ are independent. Conditioned on $G$, the effective noise covariance of $Z_t$ is $\alpha_t^2\Sigma_0+\sigma_t^2I$, and therefore
\begin{equation}
    I(G;Z_t)
    =
    \frac12\log\det\!\left[
      I+
      \alpha_t^2\Sigma_G
      (\alpha_t^2\Sigma_0+\sigma_t^2I)^{-1}
    \right].
\label{eq:app-gaussian}
\end{equation}

Defining $\lambda_t=\log(\alpha_t^2/\sigma_t^2)$, the same expression can equivalently be written as
\begin{equation}
    I(G;Z_t)
    =
    \frac12\log\det\!\left[
      I+
      \Sigma_G(\Sigma_0+e^{-\lambda_t}I)^{-1}
    \right].
    \label{eq:app-logsnr-mi}
\end{equation}

Eq.~\eqref{eq:app-logsnr-mi} expresses a central fact: the VAE posterior noise and the DiT corruption jointly determine how discriminable the latent semantic signal is. Applying the same DiT noise schedule to VAE latents of different sharpness therefore corresponds to different semantic-information regimes. We accordingly treat the two noise sources as quantities that must be calibrated jointly, rather than claiming that any fixed logSNR or timestep is a model-independent theoretical optimum.

\section{Reconstruction--Generation Balance in the Discrete Generator Decoder}
\label{app:decoder}

This appendix proves Lemma~\ref{lem:decoder}.

\subsection{Information Allocation under Different Latent Conditions}

Let $Z_A$ denote the current-block latent condition actually received by the discrete generator decoder under condition mode $A$, where $A\in\{\mathrm{cond},\mathrm{noise},\mathrm{drop}\}$. For each condition, Eq.~\eqref{eq:main-info} in the main text holds:
\begin{equation}
    H(Y\mid C)
    =
    I(Y;Z_A\mid C)
    +
    H(Y\mid C,Z_A).
    \label{eq:app-decoder-info}
\end{equation}

The \textsf{cond} state uses the complete clean latent and therefore typically provides the most informative condition for the current block. The \textsf{drop} state zeros the latent condition of the current block, which corresponds to masked token prediction that does not depend on that latent. The \textsf{noise} state applies a small random perturbation to the clean latent to train the decoder's robustness to imprecise conditions. Together, the three states cover three training regimes: strongly conditioned reconstruction, local robustness, and weakly conditioned generation.

For any discrete generator decoder $p_\psi(Y\mid C,Z_A)$, the cross-entropy decomposes exactly as
\begin{equation}
\begin{aligned}
    \Lcal_A(\psi)
    ={}&H(Y\mid C,Z_A)\\
    &+\E\KL\!\left(
      q(Y\mid C,Z_A)
      \|p_\psi(Y\mid C,Z_A)
    \right).
\end{aligned}
\label{eq:app-mode-ce}
\end{equation}

The first term is determined by how much uncertainty remains under the given condition; the second is the modeling error of the finite-capacity discrete generator decoder relative to the optimal conditional distribution. Training only under \textsf{cond} mainly optimizes the strongly conditioned region and may make the model overly dependent on precise latents, whereas training only under \textsf{drop} may cause the discrete generator decoder to ignore the latent. Mixing \textsf{cond}, \textsf{noise}, and \textsf{drop} trains the same decoder parameters to reduce the modeling error across different levels of condition reliability. The optimal mixture depends on model capacity, training budget, and the distribution of latent errors at inference, and should therefore be determined empirically.

\subsection{Local Perturbation Path of the Noise Condition}

The \textsf{noise} condition uses the same linear interpolation as the latent generator, but restricts the range of the perturbation coefficient. Let $u\in[0,u_{\max}]$, where $u_{\max}\ll1$ controls the perturbation scale; then
\begin{equation}
    Z_u=(1-u)Z+u\varepsilon,
    \qquad
    \varepsilon\sim\Normal(0,I).
    \label{eq:app-shared-path}
\end{equation}

For a single $d_z$-dimensional latent, let $\Delta Z_u=Z_u-Z$. Independent standard Gaussian noise gives
\begin{equation}
    \E[\|\Delta Z_u\|_2^2\mid Z,u]
    =u^2(d_z+\|Z\|_2^2).
    \label{eq:app-noise-perturbation-scale}
\end{equation}
A small $u$ controls the mean-squared perturbation scale, but Gaussian noise does not have finite support. For each $u$, the decoder receives information $I(Y;Z_u\mid C)$; if a stronger perturbation constitutes a more degraded stochastic channel, this mutual information does not increase, and the corresponding $H(Y\mid C,Z_u)$ does not decrease. The \textsf{noise} branch optimizes the expected token loss under this perturbation distribution to improve adaptation to imprecise latents; it neither assumes that the artificial perturbation shares the distribution of DiT output errors, nor guarantees uniform robustness to all inference-time deviations.

\section{Probabilistic Derivation of Classifier-Free Guidance}
\label{app:cfg}

The single-axis guidance of the latent DiT and each axis of the two-axis MDM guidance are based on the linear operator
\begin{equation}
    f_\gamma
    =f_{\mathrm u}+\gamma\big(f_{\mathrm c}-f_{\mathrm u}\big),
    \label{eq:app-cfg-general}
\end{equation}
where $f_{\mathrm c}$ and $f_{\mathrm u}$ are the branch outputs with and without the selected condition, and $\gamma$ is the guidance scale; the latent generator applies this operator to velocities, whereas the MDM decoder combines logits along the prefix and latent axes. Below, we give a per-noise-time interpretation under exact scores and optimal FM regression fields, and derive the token-probability reweighting under a fixed decoding state; neither automatically yields a power-mixture guarantee for the complete generated distribution.

\subsection{From Bayes' Rule to the Guided Score}

First consider a continuous variable $X_u$ with a differentiable positive density, where $u$ is the noise time and $x$ its value. Let $\mathcal K$ denote the context shared by the conditional and unconditional branches, let $A$ denote the additional condition dropped in CFG, and define
\begin{equation}
\begin{aligned}
    s_{\mathrm c}(x,u)
    &=\nabla_x\log p_u(x\mid \mathcal K,A),\\
    s_{\mathrm u}(x,u)
    &=\nabla_x\log p_u(x\mid \mathcal K).
\end{aligned}
\label{eq:app-cfg-scores}
\end{equation}
By Bayes' rule,
\begin{equation}
    p_u(x\mid \mathcal K,A)
    =\frac{p_u(A\mid x,\mathcal K)p_u(x\mid \mathcal K)}{p(A\mid \mathcal K)},
\end{equation}
and taking the gradient with respect to $x$ yields
\begin{equation}
    s_{\mathrm c}(x,u)-s_{\mathrm u}(x,u)
    =\nabla_x\log p_u(A\mid x,\mathcal K).
    \label{eq:app-cfg-score-difference}
\end{equation}
The standard CFG score therefore satisfies
\begin{equation}
\begin{aligned}
    s_\gamma
    &=s_{\mathrm u}+\gamma(s_{\mathrm c}-s_{\mathrm u})\\
    &=\nabla_x\log p_u(x\mid \mathcal K)
      +\gamma\nabla_x\log p_u(A\mid x,\mathcal K)\\
    &=\nabla_x\log \widetilde p_{\gamma,u}(x\mid \mathcal K,A),
\end{aligned}
\label{eq:app-cfg-guided-score}
\end{equation}
where the unnormalized guided distribution is
\begin{equation}
\begin{aligned}
    \widetilde p_{\gamma,u}(x\mid \mathcal K,A)
    &\propto
    p_u(x\mid \mathcal K)\,p_u(A\mid x,\mathcal K)^\gamma\\
    &\propto
    p_u(x\mid \mathcal K,A)^\gamma
    p_u(x\mid \mathcal K)^{1-\gamma}.
\end{aligned}
\label{eq:app-cfg-guided-density}
\end{equation}
When the normalizing constant is finite, it is independent of $x$ and does not affect the score. For a fixed $u$, in Eq.~\eqref{eq:app-cfg-guided-density}, $\gamma=0$ recovers the base density at that time, $\gamma=1$ recovers the conditional density, $0<\gamma<1$ attenuates the conditional likelihood ratio, and $\gamma>1$ amplifies it. This equivalence describes the instantaneous score field rather than the final sampling distribution.

\subsection{Equivalence of Flow Matching Velocity CFG and Score CFG}

The latent generator of \method predicts velocities rather than scores, so it remains to explain why applying the same linear combination to velocities, as in Eq.~\eqref{eq:workflow-cfg}, is still valid. Consider a general Gaussian corruption path
\begin{equation}
    X_u=\alpha_u X_0+\sigma_u\varepsilon,
    \qquad
    \varepsilon\sim\Normal(0,I),
    \label{eq:app-cfg-gaussian-path}
\end{equation}
where $\alpha_u$ and $\sigma_u$ are differentiable path coefficients and dots denote derivatives with respect to the noise time $u$. The following identities apply at interior times with $\alpha_u\neq0$ and $\sigma_u>0$. For an arbitrary condition $Q$, write
\begin{equation}
    s_Q(x,u)=\nabla_x\log p_u(x\mid Q).
\end{equation}
The score identity for Gaussian convolution gives
\begin{equation}
    \E[\varepsilon\mid X_u=x,Q]
    =-\sigma_u s_Q(x,u).
    \label{eq:app-cfg-eps-score}
\end{equation}
On the other hand, since $x=\alpha_uX_0+\sigma_u\varepsilon$,
\begin{equation}
    \E[X_0\mid X_u=x,Q]
    =\frac{x+\sigma_u^2s_Q(x,u)}{\alpha_u}.
    \label{eq:app-cfg-x0-score}
\end{equation}
The optimal Flow Matching regression target is the conditional expected path velocity, so
\begin{equation}
\begin{aligned}
    v_Q(x,u)
    &=\E[\dot\alpha_u X_0+\dot\sigma_u\varepsilon
      \mid X_u=x,Q]\\
    &=\frac{\dot\alpha_u}{\alpha_u}x
      +\sigma_u\left(
        \frac{\dot\alpha_u\sigma_u}{\alpha_u}
        -\dot\sigma_u
      \right)s_Q(x,u).
\end{aligned}
\label{eq:app-cfg-velocity-score}
\end{equation}
In other words, at a fixed $u$, the velocity and the score are related by an affine transformation shared by all conditions,
\begin{equation}
    v_Q(x,u)=a_u x+b_u s_Q(x,u),
    \label{eq:app-cfg-affine}
\end{equation}
where $a_u=\dot\alpha_u/\alpha_u$ and $b_u=\sigma_u(\dot\alpha_u\sigma_u/\alpha_u-\dot\sigma_u)$. When the conditional and unconditional branches are subtracted, the shared term $a_ux$ cancels exactly, so
\begin{equation}
\begin{aligned}
    v_{\mathrm u}
    +\gamma_{\mathrm{gen}}(v_{\mathrm c}-v_{\mathrm u})
    &=a_ux+b_u
      \left[s_{\mathrm u}
      +\gamma_{\mathrm{gen}}(s_{\mathrm c}-s_{\mathrm u})\right]\\
    &=a_ux+b_us_{\gamma}.
\end{aligned}
\label{eq:app-cfg-velocity-equivalence}
\end{equation}
Therefore, under a shared Gaussian path and optimal branch regression fields, velocity CFG and score CFG correspond pointwise, with $\gamma=\gamma_{\mathrm{gen}}$. For the linear path used in this work,
\begin{equation}
    \alpha_u=1-u,
    \qquad
    \sigma_u=u,
\end{equation}
Eq.~\eqref{eq:app-cfg-velocity-score} further reduces to
\begin{equation}
    v_Q(x,u)
    =-\frac{x}{1-u}
     -\frac{u}{1-u}s_Q(x,u),
    \qquad 0<u<1,
    \label{eq:app-cfg-linear-velocity}
\end{equation}
with the endpoints understood as continuous limits when these exist. In \method, $v_{\mathrm c}$ uses the historical latent cache while $v_{\mathrm u}$ clears it, so $\gamma_{\mathrm{gen}}$ modulates the influence of the historical condition on the instantaneous field.

Let $v_\gamma=v_{\mathrm u}+\gamma(v_{\mathrm c}-v_{\mathrm u})$. If the normalized density in Eq.~\eqref{eq:app-cfg-guided-density} were also the actual marginal of this flow, it would have to satisfy the continuity equation
\begin{equation}
    \partial_u\widetilde p_{\gamma,u}
    +\nabla_x\!\cdot\big(\widetilde p_{\gamma,u}v_\gamma\big)=0.
    \label{eq:app-cfg-path-consistency}
\end{equation}
The per-time score identity alone does not guarantee this path-consistency condition. Hence, even with exact branch fields, one cannot conclude that the endpoint of guided integration follows the data distribution reweighted by the same power.

\subsection{Prefix--Latent Two-Axis CFG for MDM}
\label{app:prefix-cfg}

For the \texttt{masked\_only} MDM, let $\mathcal P$ be the committed pure-token prefix, $\mathcal U$ the token/mask state of the current block together with the fixed span layout, and $H^Z$ the current latent condition refined by the condition stack. Denoting the decoder logits by $\ell(y\mid\cdot)$, the four branches are
\begin{equation}
\begin{aligned}
    \ell_{11}(y)&=\ell(y\mid\mathcal U,\mathcal P,H^Z),&
    \ell_{10}(y)&=\ell(y\mid\mathcal U,\mathcal P,0),\\
    \ell_{01}(y)&=\ell(y\mid\mathcal U,\varnothing,H^Z),&
    \ell_{00}(y)&=\ell(y\mid\mathcal U,\varnothing,0).
\end{aligned}
\label{eq:app-prefix-cfg-branches}
\end{equation}
The first subscript controls the visibility of the prefix token KV, and the second controls the additive injection of the current latent; $0$ is a zero condition tensor rather than a learnable null embedding. All four branches share $\mathcal U$ and the same $H^Z$, so $\ell_{00}$ can still read the known tokens of the current block and is not an uninformed unconditional text model.

Using guidance scales $\gamma_{\mathrm p}$ and $\gamma_{\mathrm z}$ for the prefix and latent axes, respectively, the expansion is
\begin{equation}
\begin{aligned}
    \ell_{\gamma_{\mathrm p},\gamma_{\mathrm z}}
    ={}&\gamma_{\mathrm p}\gamma_{\mathrm z}\ell_{11}
      +\gamma_{\mathrm p}(1-\gamma_{\mathrm z})\ell_{10}\\
      &+(1-\gamma_{\mathrm p})\gamma_{\mathrm z}\ell_{01}
      +(1-\gamma_{\mathrm p})(1-\gamma_{\mathrm z})\ell_{00}.
\end{aligned}
\label{eq:app-prefix-cfg-expanded}
\end{equation}
This bilinear combination is independent of the nesting order of the two axes. Setting $\gamma_{\mathrm p}=1$ modulates only the latent, $\gamma_{\mathrm z}=1$ modulates only the prefix, and $\gamma_{\mathrm z}=0$ does not use the current latent condition. In general, dropping nonzero terms no longer yields this two-axis rule.

Let $c_{ab}$ be the coefficient of $\ell_{ab}$ in Eq.~\eqref{eq:app-prefix-cfg-expanded}, and $p_{ab}=\operatorname{softmax}(\ell_{ab})$ the corresponding token probability. For any candidate token $y$,
\begin{equation}
    \ell_{ab}(y)=\log p_{ab}(y)+\kappa_{ab},
    \label{eq:app-cfg-logit-calibration}
\end{equation}
where $\kappa_{ab}$ is the log-partition constant of each branch and is independent of $y$. For a fixed current state,
\begin{equation}
    \sum_{a,b\in\{0,1\}}c_{ab}=1,
    \qquad
    p_{\gamma_{\mathrm p},\gamma_{\mathrm z}}(y)
    \propto\prod_{a,b\in\{0,1\}}p_{ab}(y)^{c_{ab}}.
    \label{eq:app-prefix-cfg-density}
\end{equation}
This algebraic identity does not require the network to reach the true conditional distribution. The coefficients can be negative, so this is an affine combination of logits rather than a convex mixture of probabilities; coefficients summing to $1$ also do not guarantee that the sharpness of the distribution or the confidence ranking used for demasking is preserved. Prefix-off only masks the token-cache read path and does not remove the historical information retained in $H^Z$, so these operational branches are not directly equated with four marginal conditions of the true joint distribution.

This rule acts on the token distribution of each prediction; since the states and normalizers of multi-step demasking change with the generated results, a global power mixture of the complete text distribution does not follow automatically. The key mask, position indexing, and read-only cache protocol are given in Appendix~\ref{app:prefix-cfg-implementation}.

\section{Implementation Details of the Overall Workflow}
\label{app:workflow-details}

This appendix supplements the dynamic segmentation, the block-causal MDM decoder, the Stage~1/Stage~2 gradient routing, and the complete inference procedure with latent feedback omitted from Section~\ref{sec:workflow}. The main text retains only the core definitions needed for unified modeling; implementation details are expanded here.

\subsection{Dynamic Semantic Segmentation and Compression-Rate Control}
\label{app:segmentation}

Let $h_t=\operatorname{Emb}(y_t)$ denote the token representation before segmentation. The semantic discontinuity between adjacent positions is defined as
\begin{equation}
    \delta_t
    =1-\cos\!\left(\operatorname{LN}(h_t),\operatorname{LN}(h_{t+1})\right),
    \qquad t=1,\ldots,T-1,
    \label{eq:app-workflow-delta}
\end{equation}
and is mapped to a boundary probability by
\begin{equation}
    p_t
    =\sigma\!\left(\alpha(\delta_t-\tau_0)+b\right),
    \qquad
    \operatorname{boundary}_t=\ind[p_t>0.5],
    \label{eq:app-workflow-boundary}
\end{equation}
where $\alpha$ is the slope, $\tau_0$ the threshold, and $b$ a scalar bias controlling the overall boundary density. If the target compression rate is $\rho^\star$ and the measured rate on the current global batch is $\hat\rho$, then $b$ does not participate in backpropagation and is updated as
\begin{equation}
    b\leftarrow
    \operatorname{clip}\!\left(
      b-\eta_b(\hat\rho-\rho^\star),-c_b,c_b
    \right),
    \label{eq:app-workflow-rate-control}
\end{equation}
where $\eta_b$ is the control step size and $c_b$ the clipping range. The upper bound $L_{\max}$ is enforced by forced cuts during a greedy scan.

\subsubsection{Convergence of Compression-Rate Control}
\label{app:rate-convergence}
By Eq.~\eqref{eq:app-workflow-boundary}, a hard boundary satisfies
\begin{equation}
    p_t>\frac12
    \quad\Longleftrightarrow\quad
    \delta_t>\tau_{\mathrm{eff}}(b)
    :=\tau_0-\frac{b}{\alpha}.
    \label{eq:app-effective-threshold}
\end{equation}
Thus, increasing $b$ lowers the effective threshold and increases the boundary density; together with Eq.~\eqref{eq:app-workflow-rate-control}, the sign of $b$ forms a negative-feedback loop. With the encoder and data distribution fixed, define
\begin{equation}
    r(b):=\E[\hat\rho\mid b]
    \label{eq:app-rate-response}
\end{equation}
as the expected compression rate after hard thresholding and length constraints. Length-based forced cuts introduce plateaus and an attainable lower bound into $r(b)$, but as long as the target $\rho^\star$ lies within the attainable range and the local rate response near the equilibrium remains monotone, the following stability result holds.

\begin{lemma}[Local Convergence of Dynamic Compression-Rate Control]
\label{lem:rate-control-convergence}
Assume that there exists $b^\star\in(-c_b,c_b)$ such that $r(b^\star)=\rho^\star$. If $r$ is differentiable in a neighborhood $\mathcal I$ of $b^\star$ and satisfies
\begin{equation}
    0<m\le r'(b)\le L<\infty,
    \qquad b\in\mathcal I,
    \label{eq:app-rate-slope-bound}
\end{equation}
then for any $0<\eta_b<2/L$, the mean-response update map
\begin{equation}
    F(b)=\Pi_{[-c_b,c_b]}\!\left[b-\eta_b\big(r(b)-\rho^\star\big)\right]
    \label{eq:app-rate-map}
\end{equation}
is a contraction in a sufficiently small neighborhood of $b^\star$, with
\begin{equation}
    |b_{n+1}-b^\star|
    \le q|b_n-b^\star|,
    \qquad
    q=\max\{|1-\eta_bm|,|1-\eta_bL|\}<1.
    \label{eq:app-rate-geometric}
\end{equation}
Consequently, $b_n\to b^\star$ and $r(b_n)\to\rho^\star$, with at least a geometric convergence rate $\mathcal O(q^n)$.
\end{lemma}

\begin{proof}
Let $e_n=b_n-b^\star$. By the mean value theorem, there exists $\xi_n$ between $b_n$ and $b^\star$ such that
\begin{equation}
    r(b_n)-r(b^\star)=r'(\xi_n)e_n.
\end{equation}
Ignoring the projection, the mean-response update in Eq.~\eqref{eq:app-rate-map} gives
\begin{equation}
    e_{n+1}
    =\big(1-\eta_br'(\xi_n)\big)e_n.
    \label{eq:app-rate-error}
\end{equation}
From Eq.~\eqref{eq:app-rate-slope-bound} and $0<\eta_b<2/L$, we have $|1-\eta_br'(\xi_n)|\le q<1$. The interval projection $\Pi_{[-c_b,c_b]}$ is non-expansive, and since $b^\star$ lies in the interior of the interval, the projection does not enlarge the error. This yields Eq.~\eqref{eq:app-rate-geometric} and completes the proof.
\end{proof}

In practice, with the encoder fixed, the hard rate on a finite global batch can be written as $\hat\rho_n=r(b_n)+\varepsilon_n$, where $\varepsilon_n$ denotes measurement fluctuation. If the $i$-th sample contains $T_i$ valid tokens, the quantization step of the rate is approximately $1/\sum_iT_i$. Under local stability and controlled noise, a fixed $\eta_b$ typically results in fluctuation around the target; using a decreasing step size with $\sum_n\eta_n=\infty$ and $\sum_n\eta_n^2<\infty$ additionally requires stochastic-approximation conditions such as zero-mean bounded-variance noise and iterate stability before convergence can be discussed.

During joint training, let $\phi_n$ be the encoder parameters at the $n$-th update, $r_n(b)=\E[\hat\rho_n\mid b,\phi_n]$ the actual response, and $b_n^\star$ its equilibrium satisfying $r_n(b_n^\star)=\rho^\star$. If the response at every step has a uniform contraction factor $q<1$ over the visited region, then
\begin{equation}
    |b_{n+1}-b_{n+1}^\star|
    \le q|b_n-b_n^\star|
       +\eta_b|\varepsilon_n|
       +|b_{n+1}^\star-b_n^\star|,
    \label{eq:app-rate-tracking}
\end{equation}
where $\varepsilon_n=\hat\rho_n-r_n(b_n)$. The last term captures the target drift caused by encoder updates. The lemma therefore analyzes local convergence under a fixed response and provides a tracking interpretation for slowly varying responses, rather than a convergence guarantee for fully non-stationary joint training.

Given the assignment map $g(t)$, span pooling uses block-diagonal bidirectional attention:
\begin{equation}
    M^{\mathrm{pool}}_{ij}=\ind[g(i)=g(j)],
    \qquad
    r_k=\frac{1}{L_k}\sum_{t\in\mathcal S_k}\widetilde h_t,
    \label{eq:app-workflow-pool}
\end{equation}
where $\widetilde h=\operatorname{Attn}_{M^{\mathrm{pool}}}(h)$. A causal Transformer along the latent axis then produces $\bar r_k$, and the posterior parameters are
\begin{equation}
    [\mu_k,\log\sigma_k^2]=W_{\mathrm{post}}\bar r_k.
    \label{eq:app-workflow-posterior-head}
\end{equation}
Reparameterized sampling is used during training, and the posterior mean is used at inference.

\subsection{Decoder Duration and Token-Level Conditioning}

Let $g_k=W_{\mathrm{in}}z_k$. The duration head outputs a distribution over $L_{\max}$ classes,
\begin{equation}
    p(L_k\mid z_k)=\operatorname{softmax}(\operatorname{MLP}_{\mathrm{dur}}(g_k)).
    \label{eq:app-workflow-duration}
\end{equation}
Given $L_k$, $g_k$ is repeated $L_k$ times and concatenated with the relative position within the span. If token $t$ has local index $j$ in its span of length $L$, then
\begin{equation}
    h_t^Z
    =W_c\!\left[
      g_{g(t)}\,\middle\|\,
      \operatorname{MLP}_{\mathrm{loc}}\!\left(
        \frac{j}{L},\frac{L-j}{L},\log L
      \right)
    \right].
    \label{eq:app-workflow-expander}
\end{equation}
Before being injected into the token decoder, this condition is further refined by a few block-causal conditioning layers. Consequently, multiple token positions expanded from the same latent remain distinguishable.

\subsection{Block-Causal Masked-Diffusion Decoder}
\label{app:discrete-generator-details}

The main text denotes the token-side module by the discrete generator decoder $D_\psi$, instantiated as a block-causal MDM. It receives the committed token context and the token-aligned latent condition $H^Z$, and realizes continuous semantics as discrete tokens. The decoder is connected to the compression encoder through the duration/expansion interface, uses latent reliability augmentation, and is trained jointly with the detached DiT in Stage~2.

During training, the decoder uses a dual layout $[\text{clean}\mid\text{masked}]$ of length $2T$. The clean half stores the ground-truth tokens, while selected positions in the masked half are replaced with a learnable \texttt{[MASK]} embedding. Let $\beta(i)$ denote the latent block containing token $i$, and let $\operatorname{cl}(i)\in\{0,1\}$ indicate whether the position belongs to the clean half. The attention mask is
\begin{equation}
\operatorname{Allow}(i,j)=
\begin{cases}
\operatorname{cl}(j)\wedge[\beta(j)\le\beta(i)],
&\operatorname{cl}(i)=1,\\[3pt]
\big(\operatorname{cl}(j)\wedge[\beta(j)<\beta(i)]\big)
\vee
\big(\neg\operatorname{cl}(j)\wedge[\beta(j)=\beta(i)]\big),
&\operatorname{cl}(i)=0.
\end{cases}
\label{eq:app-workflow-mdm-mask}
\end{equation}
Hence token blocks are causal with respect to one another, while bidirectional denoising is allowed within the current masked block; a masked query cannot read the clean copy of its current block, which prevents target leakage. Under the \texttt{masked\_only} configuration, the clean half receives no latent condition, while the masked half receives a clean, noisy, or zero latent condition according to the condition state. This design naturally decouples the historical clean-token cache from the latent condition and allows the conditioned and unconditioned logits to share the same prefix cache.

During training, each token block independently samples a masking ratio $\pi_j$, and the cross-entropy of each masked token is weighted by $\pi_j^{-1}$. With $\mathcal M_j$ denoting the masked positions in block $j$,
\begin{equation}
    \Lcal_{\mathrm{MDM}}^{A}
    =\frac{1}{|\mathcal M_A|}
    \sum_{t\in\mathcal M_A}
    \frac{-\log p_\psi(y_t\mid \widetilde Y,H_A^Z)}
         {\pi_{\beta(t)}},
    \qquad
    A\in\{\textsf{cond},\textsf{noise},\textsf{drop}\},
    \label{eq:app-workflow-mdm-loss}
\end{equation}
where the supervised sets of the three states are normalized independently. The implementation also uses fully masked blocks with some probability to directly train the recovery of a complete new block from the latent and the historical blocks.

The decoder uses the encoder-mask generation objective as $\Lcal_{\mathrm{gen}}$. A clean encoder forward pass first fixes the hard spans; a random subset of spans is then selected, their encoder input tokens are replaced with encoder-side masks, and a second encoder forward pass is performed without rerunning the dynamic segmenter. The decoder uses fully masked queries for blocks touched by these spans, but supervises only the tokens actually masked on the encoder side:
\begin{equation}
    \Lcal_{\mathrm{encmask}}
    =-\frac{1}{|\mathcal M_{\mathrm e}|}
      \sum_{t\in\mathcal M_{\mathrm e}}
      \log p_\psi(y_t\mid \widetilde Y,Z^{\mathrm{encmask}}).
    \label{eq:app-workflow-encmask}
\end{equation}
Since the span layout is fixed by stop-gradient, this branch forces the encoder to use context to form decodable latents when local observations are missing, rather than lowering the task difficulty by moving boundaries.

\subsection{Stage 1: Unified Compression VAE / Discrete Generator Decoder Objective}
\label{app:stage1-details}

The block-causal MDM decoder is trained end to end with the encoder in Stage~1; this is the core of the co-adaptation between continuous semantics and discrete realization in this work, as opposed to first fixing the encoder and then fitting the decoder. For each block, a clean latent $Z$ is first constructed, and then a mutually exclusive condition state is sampled. The clean state uses $Z$ directly; the noise state locally perturbs $Z$; the drop state removes only the latent condition of the current block while keeping the same token history. A generic formulation is
\begin{equation}
\widetilde z_k=
\begin{cases}
z_k,&A_k=\textsf{cond},\\
(1-u_k^{\mathrm c})z_k+u_k^{\mathrm c}\epsilon_k,
&A_k=\textsf{noise},\\
\varnothing,&A_k=\textsf{drop},
\end{cases}
\qquad
u_k^{\mathrm c}\in[0,u_{\max}],
\quad \epsilon_k\sim\Normal(0,I).
\label{eq:app-workflow-condition-state}
\end{equation}
The implementation may additionally apply a weaker sample-level latent perturbation before state partitioning to improve adaptation to small latent perturbations; this is input augmentation rather than a separate reconstruction loss, and it does not assume the same distribution as the terminal error of the DiT.

Let $\Lcal_{\mathrm{tok}}^{A}$ denote the MDM masked denoising objective under state $A$. The reconstruction--generation balance is
\begin{equation}
    \Lcal_{\mathrm{bal}}
    =\lambda_{\mathrm c}\Lcal_{\mathrm{tok}}^{\textsf{cond}}
    +\lambda_{\mathrm n}\Lcal_{\mathrm{tok}}^{\textsf{noise}}
    +\lambda_{\mathrm d}\Lcal_{\mathrm{tok}}^{\textsf{drop}},
    \label{eq:app-workflow-dg-unified}
\end{equation}
where $\Lcal_{\mathrm{tok}}^A=\Lcal_{\mathrm{MDM}}^A$ as defined in Eq.~\eqref{eq:app-workflow-mdm-loss}.

Separately from the three-state decoder balance, we denote the objective dedicated to injecting \emph{generative predictive information} into the encoder by
\begin{equation}
    \Lcal_{\mathrm{gen}}=\Lcal_{\mathrm{encmask}}.
    \label{eq:app-workflow-gen-variants}
\end{equation}
The unified objective of the discrete generator decoder is therefore
\begin{equation}
    \Lcal_{\mathrm{DGD}}
    =\Lcal_{\mathrm{bal}}
    +\lambda_{\mathrm{gen}}\Lcal_{\mathrm{gen}}.
    \label{eq:app-workflow-dgd-total}
\end{equation}
The key target of $\Lcal_{\mathrm{gen}}$ is the latent representation: the masked-input encoder must recover the missing semantics from context and receives gradients directly from ground-truth discrete token prediction.

The Gaussian KL of each latent is
\begin{equation}
    \KL_k
    =\frac12\sum_{d=1}^{d_z}
    \left(\mu_{k,d}^2+\sigma_{k,d}^2-1-\log\sigma_{k,d}^2\right),
    \label{eq:app-workflow-klk}
\end{equation}
and the duration loss is
\begin{equation}
    \Lcal_{\mathrm{dur}}
    =-\frac1K\sum_{k=1}^{K}\log p_\psi(L_k\mid z_k).
    \label{eq:app-workflow-dur-loss}
\end{equation}
Combined with Eq.~\eqref{eq:workflow-kl}, the total Stage~1 objective is
\begin{equation}
    \Lcal^{(1)}
    =\Lcal_{\mathrm{DGD}}
    +\lambda_{\mathrm{dur}}\Lcal_{\mathrm{dur}}
    +\beta_{\mathrm{KL}}\Lcal_{\mathrm{KL}}.
    \label{eq:app-workflow-stage1-total}
\end{equation}
This objective remains active in Stage~2; the transition from Stage~1 to Stage~2 is therefore a full-state handoff rather than ``train the VAE, then freeze the VAE.'' The discrete token losses can always shape the encoder latents through the condition path, and only the DiT prior loss is isolated by stop-gradient.

\subsection{Stage 2: Detached Block-Causal Flow Matching}
\label{app:stage2-details}

Stage~2 additionally trains the latent DiT on top of Eq.~\eqref{eq:app-workflow-stage1-total}. Let $Z$ be the sampled posterior latent of the current online encoder. The DiT reads only
\begin{equation}
    \bar Z=\sg(Z),
    \label{eq:app-workflow-detached-z}
\end{equation}
so the DiT objective cannot backpropagate into the encoder; meanwhile, $E_\phi$, the latent condition path, and $D_\psi$ are still updated normally by $\Lcal^{(1)}$.

The latent sequence is grouped into blocks of $B$ latents. Each block independently samples a noise level $u_j$, all latents in a block share the same timestep, and the linear path in Eq.~\eqref{eq:workflow-flow} is used. The training layout consists of clean history and noisy current blocks: every noisy query can read only strictly earlier clean latent blocks and interacts bidirectionally within its current noisy block. The attention mask can be written as
\begin{equation}
    \operatorname{Allow}_{\mathrm{DiT}}(i,j)
    =\ind[\beta(j)<\beta(i)]\operatorname{cl}(j)
     \;\vee\;
     \ind[\beta(j)=\beta(i)]\neg\operatorname{cl}(j),
    \label{eq:app-workflow-dit-mask}
\end{equation}
where clean-history positions act as read-only conditions, and the current noisy block cannot read its corresponding clean target copy.

To cover inference cases in which the prompt is not aligned with latent block boundaries, training may randomly restore the first $c\in\{1,\ldots,B-1\}$ latents of a noisy block to their clean values and remove them from the corresponding Flow Matching supervision. This partial-clean-prefix augmentation matches the handling of the first mixed block at inference. Finally,
\begin{equation}
    \nabla_\phi\Lcal_{\mathrm{FM}}=0,
    \qquad
    \nabla_\theta\Lcal_{\mathrm{FM}}\neq0,
    \label{eq:app-workflow-full-detach}
\end{equation}
and
\begin{equation}
    \Lcal^{(2)}
    =\Lcal^{(1)}+\lambda_{\mathrm{FM}}\Lcal_{\mathrm{FM}}.
    \label{eq:app-workflow-stage2-total}
\end{equation}
Stage~2 can thus be understood as continuously updating the online compression encoder and discrete generator decoder while letting the DiT track their evolving latent distribution, but forbidding the DiT from shaping the encoder through the diffusion loss.

\subsection{Complete Inference Procedure and Latent Feedback}
\label{app:inference-details}

Inference follows \emph{propose--decode--re-encode--commit}. The dynamic encoder runs on the prompt only once at the beginning and uses the posterior mode. Suppose the prompt yields $K_{\mathrm p}$ latents with
\begin{equation}
    K_{\mathrm p}=Bq+r,
    \qquad 0\le r<B.
    \label{eq:app-workflow-prompt-split}
\end{equation}
The first $Bq$ block-aligned latents are prefilled into the persistent caches. If $r>0$, the trailing $r$ prompt latents and their tokens are not written into the persistent cache as a completed block; instead, they are kept as the clean prefix of the first generated block, and the remaining $B-r$ latent slots are initialized from Gaussian noise:
\begin{equation}
    Z_j^{(0)}
    =\left[
      z^{\mathrm p}_{K_{\mathrm p}-r+1:K_{\mathrm p}}
      \,\middle\|\,
      \epsilon_1,\ldots,\epsilon_{B-r}
    \right],
    \qquad \epsilon_i\sim\Normal(0,I).
    \label{eq:app-workflow-first-block}
\end{equation}
This explains the meaning of ``discard trailing remainder'' in Figure~\ref{fig:method_overview}: the trailing remainder is excluded from the \emph{committed cache}, not permanently removed from the prompt semantics. After each DiT Euler update, the clean prompt prefix in the mixed block is restored to its original values.

\paragraph{Latent proposal.}
The DiT performs $S_{\mathrm{gen}}$ denoising steps on top of the committed latent prefix. With CFG,
\begin{equation}
    \hat v
    =v_{\mathrm u}+\gamma_{\mathrm{gen}}(v_{\mathrm c}-v_{\mathrm u}),
    \label{eq:app-workflow-gen-cfg}
\end{equation}
where the conditional branch reads the committed corrected latent cache and the unconditional branch removes this prefix. Importantly, the current proposal remains a read-only temporary state until token generation and feedback correction are complete, and is not written into the persistent DiT/condition caches.

\paragraph{Duration and token decoding.}
The proposal latents first pass through the duration/expansion path. For newly generated latent slots,
\begin{equation}
    \widehat L_k=\arg\max_L p_\psi(L\mid z_k),
    \label{eq:app-workflow-length-infer}
\end{equation}
whereas the slots belonging to the prompt tail in the mixed first block keep the ground-truth span lengths known to the encoder. This yields the effective width and the token-aligned condition of the current token block.

The MDM decoder initializes unknown positions to \texttt{[MASK]}, repeatedly predicts them over several refinement steps, and progressively commits them by confidence; intermediate masked states do not enter the KV cache, and only after the whole block is determined is a single clean forward pass used to commit the final tokens, yielding the finalized clean token block $Y_j$.

The decoder logits use the prefix--latent two-axis CFG in Eq.~\eqref{eq:workflow-prefix-cfg}. The four branches share the current block state and the latent-free clean-token prefix cache; the prefix-off branches mask the historical token keys, and the latent-off branches zero the injection of the current proposal latent condition. Execution details are given below.

\paragraph{Fixed-span re-encoding.}
After the current token block is finalized, the feedback encoder does not rerun the dynamic segmenter. Instead, it directly constructs fixed span boundaries from the span lengths just determined by the duration predictor. It applies to these tokens the same span-local bidirectional pooling as the initial encoder, reads the committed encoder latent-prefix KV, obtains the posterior of the current block through the causal latent encoder, and takes its mode:
\begin{equation}
    Z_j^{\mathrm{corr}}
    =\mu_\phi\!\left(Y_j;
      \widehat L_{\mathcal B_j},
      \mathcal K^{E}_{<j}\right).
    \label{eq:app-workflow-corrected-latent}
\end{equation}
For the first mixed block, the previously saved pooled representations of the prompt tail are concatenated with those of the newly generated spans to form a complete $B$-latent block before running the latent encoder. Since the span boundaries are fixed, this correction changes ``the latent obtained by re-encoding the generated tokens under this span layout,'' rather than re-optimizing the segmentation.

\paragraph{Atomic commit.}
After feedback, the proposal latent is replaced by $Z_j^{\mathrm{corr}}$, and the caches are updated according to the invariants
\begin{equation}
\begin{aligned}
    \mathcal K^{\mathrm{DiT}}_j
    &=Z^{\mathrm{corr}}_{\le j},\\
    \mathcal K^{\mathrm{cond}}_j
    &=\operatorname{Cond}(Z^{\mathrm{corr}}_{\le j}),\\
    \mathcal K^{E}_j
    &=\operatorname{EncState}(Y_{\le j}),\\
    \mathcal K^{\mathrm{tok}}_j
    &=Y^{\mathrm{clean}}_{\le j}.
\end{aligned}
\label{eq:app-workflow-cache-invariant}
\end{equation}
That is, the DiT history, the condition history, and the encoder latent history all correspond to corrected encoder latents, while the token cache contains only the final clean tokens. The proposal latent never enters the persistent history. The current block is not regenerated after correction, so the effect of feedback starts from the next latent block.

This transaction also establishes an important train--inference consistency invariant. During training, the clean history fitted by the DiT consists of online-encoder latents of ground-truth tokens; at inference, after each block is completed, the next history is obtained from the final committed tokens through the same fixed-span encoder path. Hence, for the committed prefix,
\begin{equation}
    \mathcal K^{Z}_{\le j}
    =E_\phi\!\left(\mathcal K^{Y}_{\le j};\widehat{\mathcal S}_{\le j}\right),
    \label{eq:app-workflow-context-alignment}
\end{equation}
rather than $\mathcal K^Z_{\le j}=Z^{\mathrm{prop}}_{\le j}$. The former keeps the continuous context of the DiT, the latent history of the condition stack, and the clean-token context of the discrete generator decoder pointing to the same finalized content. If the proposal were committed directly, any proposal--decode discrepancy would be permanently written into the latent cache and would cause cross-space context bias in subsequent blocks; fixed-span correction is designed precisely to eliminate this cumulative mismatch. Since the correction reuses the spans just predicted, it only aligns the representation content without reintroducing segmentation drift at inference.

Algorithm~\ref{alg:inference} summarizes the unified procedure.

\begin{algorithm}[t]
\caption{Unified block-wise inference of \method}
\label{alg:inference}
\begin{algorithmic}[1]
    \Require Prompt $C$, latent block size $B$, DiT steps $S_{\mathrm{gen}}$
    \State $z^{\mathrm p}_{1:K_{\mathrm p}}\gets\mu(E_\phi(C))$
    \State Prefill the complete block-aligned prompt prefix into the persistent caches, and keep the remainder as the clean prefix of the first block
    \While{stopping criterion not met}
        \State Initialize the current latent block from Gaussian noise (and the optional clean prompt tail)
        \For{$i=0,\ldots,S_{\mathrm{gen}}-1$}
            \State Update the DiT proposal on the committed corrected-latent prefix; restore the clean tail if present
        \EndFor
        \State Predict span lengths from the proposal latents and construct the token-aligned latent condition
        \State Complete the current token block with the block-causal MDM decoder using the two-axis CFG in Eq.~\eqref{eq:workflow-prefix-cfg}
        \State Re-encode the final tokens with the fixed predicted spans to obtain $Z_j^{\mathrm{corr}}$
        \State Atomically commit the finalized tokens, $Z_j^{\mathrm{corr}}$, and the corresponding condition/encoder states
    \EndWhile
    \State \Return the generated token sequence
\end{algorithmic}
\end{algorithm}

\subsubsection{Execution Details of MDM Prefix--Latent CFG}
\label{app:prefix-cfg-implementation}

\paragraph{Condition preparation and scope.}
This path requires \texttt{decoder\_cond\_drop\_scope=masked\_only}, which guarantees that the prefix token KV contains no latent condition. We follow the branch subscripts $a,b$ of Appendix~\ref{app:prefix-cfg}. The refined condition is prepared only once per block:
\begin{equation}
    H_j^Z
    =\operatorname{Cond}\!\left(
      \operatorname{Expand}(Z_j^{\mathrm{prop}},\widehat L_{\mathcal B_j});
      \mathcal K^{\mathrm{cond}}_{<j}\right),
    \qquad H_j^{Z,(b)}=bH_j^Z,
    \label{eq:app-prefix-cfg-condition}
\end{equation}
where $\mathcal K^{\mathrm{cond}}_{<j}$ is the committed condition-stack history and $\widehat L_{\mathcal B_j}$ follows the existing duration rule. All four branches share the same $H_j^Z$; $b=0$ zeros the injected condition after refinement, without recomputing the duration or the condition stack.

\paragraph{Prefix-key masking and position indexing.}
Omitting the batch index, let $m_r\in\{0,1\}$ indicate the validity of the $r$-th key slot, $N_{\mathrm k}$ the total number of key slots, and $N_{\mathrm q}$ the number of current query slots. The current block occupies the trailing index set $\mathcal Q_j=\{N_{\mathrm k}-N_{\mathrm q}+1,\ldots,N_{\mathrm k}\}$. The two prefix visibilities correspond to
\begin{equation}
    m_r^{(a)}=m_r\big(a+(1-a)\ind[r\in\mathcal Q_j]\big),
    \qquad
    \operatorname{pos}^{(a)}_r
    =\max\!\left(\sum_{s=1}^{r}m_s^{(a)}-1,0\right).
    \label{eq:app-prefix-cfg-key-mask}
\end{equation}
For valid current-block queries, $m^{(a)}$ controls both key visibility and position indexing. With $a=0$, only the valid keys of the current block are kept, and their positions are indexed contiguously from $0$. If the first mixed block contains a clean prompt tail, the tail still belongs to $\mathcal Q_j$ and is not masked together with the persistent prefix.

Our implementation caches keys before RoPE and applies RoPE according to $\operatorname{pos}^{(a)}$ after concatenating the cached and current keys. The two prefix branches thus reuse the same persistent KV without an additional cache or a modified RoPE operator, although the effective position indices differ between branches. Masking does not modify the tokens, the DiT proposal, or the condition-stack history. This rule adds no training objective; the first training block provides an input form without historical token keys, but we do not assume from this that all prefix-off inference states are in-distribution with respect to training.

\paragraph{Read-only branches, demasking, and commit.}
Let $M_j$ be the number of refinement steps of the current block, $s\in\{0,\ldots,M_j-1\}$ the iteration index, and $\mathcal M_s$ the valid positions that are still masked. At each step, the branches with nonzero coefficients are computed on the same current token state, all with \texttt{update\_kv=False}, and the logits are then combined according to Eq.~\eqref{eq:app-prefix-cfg-expanded}. Denoting the combined logits at position $t$ by $\ell_t^{(s)}$, the candidate token and confidence under greedy decoding are
\begin{equation}
    \widehat y_t^{(s)}=\arg\max_y\ell_t^{(s)}(y),
    \qquad
    c_t^{(s)}=\log\operatorname{softmax}(\ell_t^{(s)})
                    \big(\widehat y_t^{(s)}\big),
    \quad t\in\mathcal M_s.
    \label{eq:app-prefix-cfg-confidence}
\end{equation}
Under non-greedy decoding, candidates are sampled from the guided logits as configured. At each step, the highest-confidence candidates are kept according to the existing quota and the remaining positions stay masked; the known prompt tail does not participate in generation. Since $c_t^{(s)}$ is also derived from the guided logits, the two-axis weights affect both token selection and potentially the unmasking order, and confidence effects cannot be ruled out merely by the scale invariance of the greedy argmax.

After the block is finalized, the final tokens are written into the cache by a single latent-free clean forward pass (\texttt{update\_kv=True}) without CFG, restoring full prefix visibility and normal position indexing. The original fixed-span feedback and corrected-history commit remain unchanged. The persistent token KV is determined only by finalized tokens; branch hidden states and masked states are temporary.

The two configuration parameters are \texttt{mdm\_decoder\_cfg\_weight} ($\gamma_{\mathrm z}-1$) and \texttt{mdm\_decoder\_prefix\_cfg\_weight} ($\gamma_{\mathrm p}-1$). The prefix parameter defaults to $0$, corresponding to $\gamma_{\mathrm p}=1$. For a fixed single-step state, the four branch logits can be cached and recombined with different weights; however, different weights change subsequent tokens, the unmasking order, and the committed history, so a complete rollout cannot share a single offline logits trajectory.

\section{Detailed Derivation of the Computational Cost}
\label{app:compute-details}

This section re-derives the computational cost according to the current training and inference execution graphs. There are two key differences: Stage~2 continues training the same compression encoder and discrete generator decoder and only additionally adds a detached DiT; and inference performs fixed-span feedback after the discrete tokens of each block are finalized, before committing the corrected latent history. Hence, the previous accounting of ``Stage~1 VAE + Stage~2 DiT/decoder'' or a one-way ``DiT $\rightarrow$ decoder'' no longer applies.

\subsection{Notation}

Below, one multiply--add counts as two FLOPs, and we count backbone linear layers, attention matrix multiplications, output heads, and auxiliary projections, while omitting element-wise operations such as embedding lookup, normalization, and softmax. Training includes the forward and backward passes, but excludes activation recomputation, optimizer updates, and communication. If these additional operations are counted, \method and AR should use the same convention.

For a module $R$, let $N_R$ be the number of its backbone linear-layer parameters (excluding separate heads/interfaces), $n_R$ its number of layers, $d_R^{\mathrm q}$ the total width of all query heads per layer, $Q_R$ the cumulative number of processed positions, and $\mathcal A_R^{\mathrm{exec}}$ the cumulative number of executed query--key pairs per layer. When all layers share the same shape,
\begin{equation}
\begin{aligned}
    \mathcal C_{R,\mathrm{backbone}}^{\mathrm{fwd}}
    &\approx 2N_RQ_R+4n_Rd_R^{\mathrm q}\mathcal A_R^{\mathrm{exec}},\\
    \mathcal C_{R,\mathrm{backbone}}^{\mathrm{train}}
    &\approx 6N_RQ_R+12n_Rd_R^{\mathrm q}\mathcal A_R^{\mathrm{exec}},
\end{aligned}
\label{eq:app-cost-backbone-convention}
\end{equation}
where the attention term includes both $QK^\top$ and the attention--value product; heterogeneous layers are summed layer by layer. Let $\mathcal A_R^{\mathrm{valid}}$ be the number of pairs allowed by the mask; it equals the executed amount only when the kernel actually skips invalid regions, and dense masks, tile boundaries, and padding must be charged according to the actual shapes.

If the vocabulary size is $V$, the output hidden width is $d_{\mathrm{out}}$, and logits are actually computed at $Q_{\mathrm{logit}}$ positions, then
\begin{equation}
    \mathcal C_{\mathrm{head}}^{\mathrm{fwd}}
    \approx 2Vd_{\mathrm{out}}Q_{\mathrm{logit}},
    \qquad
    \mathcal C_{\mathrm{head}}^{\mathrm{train}}
    \approx 6Vd_{\mathrm{out}}Q_{\mathrm{logit}}.
    \label{eq:app-cost-output-head}
\end{equation}
Weight sharing with the embedding does not eliminate this output projection, and supervising only a subset of positions does not automatically reduce the backbone $Q_R$. All of the above are approximations of the dominant matrix-multiplication terms; the total cost accumulates over actual calls rather than being counted by the number of distinct parameters.

A training sequence contains $T$ valid tokens and, after compression,
\begin{equation}
    K=\rho T.
    \label{eq:app-cost-compression}
\end{equation}
At inference, let $T_{\mathrm{new}}$ be the number of newly generated tokens, $K_{\mathrm{new}}$ the number of corresponding latents, and $\rho_{\mathrm{gen}}=K_{\mathrm{new}}/T_{\mathrm{new}}$ the realized ratio. Ignoring truncation at the beginning and end, the average number of tokens $P$ covered by a latent block of size $B$ and the number of blocks $J$ satisfy
\begin{equation}
    P\approx\frac{B}{\rho_{\mathrm{gen}}},
    \qquad
    J\approx\frac{\rho_{\mathrm{gen}}T_{\mathrm{new}}}{B}.
    \label{eq:app-cost-token-per-block}
\end{equation}
Duration prediction does not guarantee $\rho_{\mathrm{gen}}=\rho^\star$. The prompt tail in the first block, truncation in the last block, and batch padding may make the actual processed amount larger than the effective output, so we first give per-call counts and then approximations that ignore boundary effects.

\subsection{Training: Persistent Representation/Decoder Path and an Additional Detached DiT}
\label{app:ar-comparison}

\paragraph{Stage~1 representation/decoder path.}
Denote the training cost of this path by $\mathcal C_{\mathrm{rep}}(T,K)$. Clean, noisy, and drop are mutually exclusive condition states within the same main request, so the backbone forward pass is not multiplied by three; the additional requests of the generation-oriented $\Lcal_{\mathrm{gen}}$, however, must be counted separately. Below, $\mathcal C_E$ includes the encoder and the posterior projection, $\mathcal C_D^{\mathrm{MDM}}$ includes the MDM token backbone, attention, and the corresponding output head, and the remaining operations such as conditioning and duration are listed as separate terms.

The MDM main request uses one clean encoder pass and one dual-layout decoder pass of length $2T$; $\Lcal_{\mathrm{encmask}}$ fixes the original spans and runs an additional masked-input encoder pass and a separate MDM request. Hence, when both requests are fully executed,
\begin{equation}
\begin{aligned}
    \mathcal C_{\mathrm{rep}}^{\mathrm{MDM}}(T,K)
    \approx{}&
    2\,\mathcal C_E(T,K)
    +2\,\mathcal C_D^{\mathrm{MDM}}(2T)\\
    &+\mathcal C_{\mathrm{cond/dur/KL}}^{\mathrm{MDM}},
\end{aligned}
\label{eq:app-cost-mdm-train}
\end{equation}
where each single-pass cost is the average over the two requests. The two decoder requests together process about $4T$ backbone positions, with a dominant linear training term of $24N_DT$, where $N_D$ is the number of MDM backbone parameters. Concatenating the two requests along the batch dimension does not change this count; the output head is charged by the actual logit positions, and repeated calls of the condition path are included in the last term. The two-axis inference CFG does not increase the number of training requests in either stage.

Let $\E_s$ denote the average over the actual training samples or batch layouts of Stage~$s$, and $T$ the corresponding number of valid tokens. Define
\begin{equation}
    c_{\mathrm{rep}}^{(s)}
    :=\frac{\E_s[\mathcal C_{\mathrm{rep}}(T,K)]}{\E_s[T]},
    \qquad s\in\{1,2\}.
    \label{eq:app-cost-rep-per-token}
\end{equation}
This definition normalizes the actual executed amount by valid tokens. The representation path remains fully present in Stage~2, but when the sequence length, compression layout, or request configuration changes, $c_{\mathrm{rep}}^{(1)}=c_{\mathrm{rep}}^{(2)}$ need not hold.

\paragraph{Stage~2 detached DiT.}
The same clean encoder latent participates in the token objectives and also enters the DiT as $\bar Z=\operatorname{sg}(Z)$. Suppose the training sequence has $J_{\mathrm{tr}}=\lceil K/B\rceil$ latent blocks, and the last block has length $r_{\mathrm{last}}=K-B(J_{\mathrm{tr}}-1)\in\{1,\ldots,B\}$. Omitting the clean copy of the last block, $K_{\mathrm c}=K-r_{\mathrm{last}}$, and we define
\begin{equation}
    \chi_G=\frac{K+K_{\mathrm c}}{K},
    \qquad 1\le\chi_G<2.
    \label{eq:app-cost-dit-layout-factor}
\end{equation}
When $K$ is divisible by $B$, $r_{\mathrm{last}}=B$ rather than $0$. The dominant linear training term of the DiT backbone is
\begin{equation}
    \mathcal C_{G,\mathrm{linear}}^{\mathrm{train}}(T,K)
    \approx 6N_G(K+K_{\mathrm c})
    =6\chi_G\rho N_GT.
    \label{eq:app-cost-dit-train-linear}
\end{equation}
Detachment only cuts the gradient to the encoder; it does not skip the backward pass of the DiT itself.

When $K=J_{\mathrm{tr}}B$ and the clean stream is also computed block-causally, the numbers of valid pairs of clean and noisy queries correspond to the first $J_{\mathrm{tr}}-1$ blocks and all $J_{\mathrm{tr}}$ blocks, respectively:
\begin{equation}
    \mathcal A_G^{\mathrm{valid}}
    =B^2\left(\sum_{j=1}^{J_{\mathrm{tr}}-1}j
             +\sum_{j=1}^{J_{\mathrm{tr}}}j\right)
    =K^2.
    \label{eq:app-cost-dit-valid-pairs}
\end{equation}
Layouts with incomplete blocks remain $\mathcal O(K^2)$. The dominant attention training term is
\begin{equation}
    \mathcal C_{G,\mathrm{attn}}^{\mathrm{train}}
    \approx 12n_Gd_G^{\mathrm q}\mathcal A_G^{\mathrm{exec}},
    \qquad
    \mathcal A_G^{\mathrm{valid}}=\mathcal O(\rho^2T^2),
    \label{eq:app-cost-dit-train-attn}
\end{equation}
where $n_G$ and $d_G^{\mathrm q}$ follow the module-level notation. Let $\mathcal C_{G,\mathrm{aux}}^{\mathrm{train}}$ include the latent input/output and time/condition projections not counted in $N_G$. The average per-token cost of the added DiT is then
\begin{equation}
    c_G^{(2)}
    =\frac{\E_2[\mathcal C_{G,\mathrm{linear}}^{\mathrm{train}}
                 +\mathcal C_{G,\mathrm{attn}}^{\mathrm{train}}
                 +\mathcal C_{G,\mathrm{aux}}^{\mathrm{train}}]}{\E_2[T]}.
    \label{eq:app-cost-dit-per-token}
\end{equation}
Hence $\rho$ and $\rho^2$ describe the scaling of the linear length and the attention pairs under the same architecture and a comparable layout, rather than a fixed speedup factor for the complete system or an arbitrary kernel.

\paragraph{Total training cost and AR reference.}
If the two stages process $D_1$ and $D_2$ valid tokens, respectively, then
\begin{equation}
    \mathcal C_{\text{\method}}^{\mathrm{train}}
    \approx D_1c_{\mathrm{rep}}^{(1)}
           +D_2\big(c_{\mathrm{rep}}^{(2)}+c_G^{(2)}\big).
    \label{eq:app-cost-total-train}
\end{equation}
With the same representation configuration, the first part can be merged into $(D_1+D_2)c_{\mathrm{rep}}$, but the DiT attention and auxiliary projections must still be retained.

Taking as reference an AR model with training budget $D_{\mathrm{AR}}$ and average per-token cost $c_{\mathrm{AR}}^{\mathrm{train}}$ under the same convention, define
\begin{equation}
    r_{\mathrm{rep}}^{(s)}
    =\frac{c_{\mathrm{rep}}^{(s)}}{c_{\mathrm{AR}}^{\mathrm{train}}},
    \qquad
    r_G^{\mathrm{train}}
    =\frac{c_G^{(2)}}{c_{\mathrm{AR}}^{\mathrm{train}}}.
    \label{eq:app-cost-train-normalized}
\end{equation}
Then
\begin{equation}
\begin{aligned}
    R_{\mathrm{train}}
    &:=\frac{\mathcal C_{\text{\method}}^{\mathrm{train}}}
             {D_{\mathrm{AR}}c_{\mathrm{AR}}^{\mathrm{train}}}\\
    &\approx\frac{D_1}{D_{\mathrm{AR}}}r_{\mathrm{rep}}^{(1)}
      +\frac{D_2}{D_{\mathrm{AR}}}
       \big(r_{\mathrm{rep}}^{(2)}+r_G^{\mathrm{train}}\big).
\end{aligned}
\label{eq:app-cost-train-ratio}
\end{equation}
The AR cost $c_{\mathrm{AR}}^{\mathrm{train}}$ likewise includes attention and the output head. $6N_{\mathrm{AR}}D_{\mathrm{AR}}$ is only its backbone linear reference and cannot serve as the denominator for a complete cost that includes other operations. Compression mainly reduces the cost of the added prior and does not automatically make the entire training scale with $\rho$.

\subsection{Inference: Proposal, Discrete Decoding, Feedback, and Commit}

Excluding the one-time prompt encoding/prefill, the transaction of each generated block is
\begin{equation}
    \text{DiT proposal}
    \rightarrow\text{DGD token decoding}
    \rightarrow\text{fixed-span re-encoding}
    \rightarrow\text{corrected-cache commit}.
    \label{eq:app-cost-infer-transaction}
\end{equation}
Subsequent attention still reads the prompt KV, so its history length cannot be dropped merely because prefill is excluded. Below, $\mathcal C_G^{\mathrm{infer}}$ includes the DiT proposal/commit, $\mathcal C_D^{\mathrm{infer}}$ includes condition preparation, token decoding, and the token-cache commit, and $\mathcal C_{\mathrm{fb}}$ includes the feedback encoder and the corrected condition commit; the three parts do not overlap.

\paragraph{DiT proposal and corrected DiT commit.}
Let $B_j^{\mathrm{exec}}$ be the number of latent query slots actually processed in the $j$-th block, and $\nu_{G,j,i}$ the number of velocity branches actually executed at its $i$-th denoising step. Then
\begin{equation}
\begin{aligned}
    Q_G^{\mathrm{prop}}
    &=\sum_{j=1}^{J}B_j^{\mathrm{exec}}
      \sum_{i=1}^{S_{\mathrm{gen}}}\nu_{G,j,i}\\
    &\approx\rho_{\mathrm{gen}}\nu_{\mathrm{gen}}S_{\mathrm{gen}}T_{\mathrm{new}}.
\end{aligned}
\label{eq:app-cost-gen-proposal-query}
\end{equation}
The second line assumes a uniform branch count $\nu_{\mathrm{gen}}$ and ignores boundaries and padding. The proposal does not write to the persistent cache; corrected latents are committed only after feedback. To avoid folding attention into the linear coefficient, define
\begin{equation}
    \kappa_G^{\mathrm{commit}}
    :=\frac{\mathcal C_{G,\mathrm{linear}}^{\mathrm{commit}}}
            {2N_GK_{\mathrm{new}}},
    \label{eq:app-cost-gen-commit-factor}
\end{equation}
where the numerator is only the backbone linear cost of the corrected DiT commit. If each new latent requires one additional full clean forward pass, $\kappa_G^{\mathrm{commit}}\approx1$ ignoring boundaries. The total cost is
\begin{equation}
\begin{aligned}
    \mathcal C_G^{\mathrm{infer}}
    \approx{}&2N_GQ_G^{\mathrm{prop}}
       +\mathcal C_{G,\mathrm{linear}}^{\mathrm{commit}}\\
       &+\mathcal C_{G,\mathrm{attn}}^{\mathrm{infer}}
       +\mathcal C_{G,\mathrm{aux}}^{\mathrm{infer}},
\end{aligned}
\label{eq:app-cost-gen-infer}
\end{equation}
where the last two terms accumulate the attention and remaining input/output projections of the proposal and commit. If the current effective block width is $b_j$ and the committed latent history length is $P_j^Z$, the numbers of valid pairs per branch with prefix on and off are $b_j(P_j^Z+b_j)$ and $b_j^2$, respectively; these are accumulated over the actual denoising steps and commits, and the executed pair counts are substituted into Eq.~\eqref{eq:app-cost-backbone-convention}.

\paragraph{MDM discrete generator decoder.}
Let $W_j^{\mathrm{exec}}$ be the query width actually processed in each refinement of the $j$-th block, $W_j^{\mathrm{com}}$ the number of positions processed in the final clean commit, $M_j$ the number of refinement steps, and $\nu_{D,j,i}$ the number of decoder branches actually executed at step $i$. Then
\begin{equation}
\begin{aligned}
    \mathcal C_D^{\mathrm{MDM}}
    \approx{}&2N_D\sum_{j=1}^{J}
       \left(W_j^{\mathrm{exec}}\sum_{i=1}^{M_j}\nu_{D,j,i}
             +W_j^{\mathrm{com}}\right)\\
       &+\mathcal C_{D,\mathrm{attn}}^{\mathrm{MDM}}
        +\mathcal C_{\mathrm{head/interface}}^{\mathrm{MDM}}.
\end{aligned}
\label{eq:app-cost-mdm-infer}
\end{equation}
Ignoring the beginning, end, and padding, if $\nu_{\mathrm{dec}}$ branches are computed at every step, the dominant linear term reduces to $2N_D\sum_j(\nu_{\mathrm{dec}}M_j+1)W_j$, where $W_j$ is the effective block width; with a uniform $M_j=M$, it is approximately $2N_D(\nu_{\mathrm{dec}}M+1)T_{\mathrm{new}}$.

The full two-axis CFG executes four branches, but $\nu_{D,j,i}$ counts the actual forward passes rather than only the nonzero coefficients. The cost is reduced only if the code skips zero-coefficient branches or merges identical branches before the call; multiplying logits by zero after the call saves no backbone FLOPs. The four branches share one proposal condition preparation per block, and only one clean-token commit without CFG is performed at the end.

To expand the effect of prefix visibility, let $P_j^Y$ be the number of tokens committed before the $j$-th block (including the prompt), and $\nu_{j,i}^{\mathrm{on}}$ the number of branches with the prefix kept at that step. With effective width $W_j$, the number of valid pairs over all refinements and the clean commit is
\begin{equation}
\begin{aligned}
    \mathcal A_D^{\mathrm{valid}}
    =\sum_{j=1}^{J}\Bigg[&
       \sum_{i=1}^{M_j}\left(
         \nu_{j,i}^{\mathrm{on}}W_jP_j^Y
         +\nu_{D,j,i}W_j^2\right)\\
       &+W_j(P_j^Y+W_j)\Bigg].
\end{aligned}
\label{eq:app-cost-mdm-attention-pairs}
\end{equation}
The full two-axis CFG has $(\nu^{\mathrm{on}},\nu_{\mathrm{dec}})=(2,4)$. Prefix-off branches only allow keys from the current block, but if the kernel does not skip the masked historical blocks, the actual FLOPs must use $\mathcal A_D^{\mathrm{exec}}$:
\begin{equation}
    \mathcal C_{D,\mathrm{attn}}^{\mathrm{MDM}}
    \approx4n_Dd_D^{\mathrm q}\mathcal A_D^{\mathrm{exec}}.
    \label{eq:app-cost-mdm-attention}
\end{equation}
$\mathcal C_{\mathrm{head/interface}}^{\mathrm{MDM}}$ includes all actual logit projections and one proposal duration/expansion/condition refinement. The recomputation of the corrected condition is attributed to feedback and is not double-counted here.

\paragraph{Fixed-span feedback and corrected condition commit.}
Each finalized block is re-encoded with fixed spans, without rerunning the dynamic segmenter. Let $N_{F,t}$ and $N_{F,z}$ be the numbers of linear-layer parameters of the span-local token pooling and the causal latent encoder, and $Q_{F,t}$ and $Q_{F,z}$ the numbers of positions they actually process. Then
\begin{equation}
\begin{aligned}
    \mathcal C_{\mathrm{fb,enc}}
    \approx{}&2N_{F,t}Q_{F,t}+2N_{F,z}Q_{F,z}\\
       &+\mathcal C_{\mathrm{fb,attn}}
        +\mathcal C_{\mathrm{fb,proj}},
\end{aligned}
\label{eq:app-cost-feedback-encoder}
\end{equation}
where $\mathcal C_{\mathrm{fb,proj}}$ includes projections such as the posterior head that are not counted in the above parameters; ignoring boundaries and padding, $Q_{F,t}\approx T_{\mathrm{new}}$ and $Q_{F,z}\approx K_{\mathrm{new}}$. A token-encoder depth of zero does not eliminate the computation of span-local pooling. Its valid attention pairs accumulate as $\sum_k L_k^2$ over spans, and the latent encoder is charged according to the causal prefix/current layout.

After obtaining $Z_j^{\mathrm{corr}}$, the corrected condition-stack KV must also be rebuilt and committed; denoting its cumulative cost by $\mathcal C_{\mathrm{cond,commit}}$, we have $\mathcal C_{\mathrm{fb}}=\mathcal C_{\mathrm{fb,enc}}+\mathcal C_{\mathrm{cond,commit}}$. The proposal condition preparation and the corrected condition commit are two distinct computations, but neither is multiplied by four for the four CFG branches; the corrected DiT commit is already included in $\mathcal C_G^{\mathrm{infer}}$.

\paragraph{Unified AR comparison convention.}
AR likewise excludes the one-time prefill. Let $Q_{\mathrm{AR}}$ be the number of positions actually processed by its incremental backbone, $\mathcal A_{\mathrm{AR}}^{\mathrm{exec}}$ the executed attention pairs, and $\mathcal C_{\mathrm{head,AR}}^{\mathrm{infer}}$ the output-projection cost. Then
\begin{equation}
\begin{aligned}
    \mathcal C_{\mathrm{AR}}^{\mathrm{infer}}
    \approx{}&2N_{\mathrm{AR}}Q_{\mathrm{AR}}
      +4n_{\mathrm{AR}}d_{\mathrm{AR}}^{\mathrm q}
        \mathcal A_{\mathrm{AR}}^{\mathrm{exec}}\\
      &+\mathcal C_{\mathrm{head,AR}}^{\mathrm{infer}}.
\end{aligned}
\label{eq:app-cost-ar-inference}
\end{equation}
If the initial KV length is $T_{\mathrm p}$, processing $Q_{\mathrm{AR}}$ new queries consecutively yields $Q_{\mathrm{AR}}T_{\mathrm p}+Q_{\mathrm{AR}}(Q_{\mathrm{AR}}+1)/2$ valid pairs; $Q_{\mathrm{AR}}\approx T_{\mathrm{new}}$ only ignores boundary differences in where logits are taken. Using denominators under the same convention, define
\begin{equation}
    r_G^{\mathrm{infer}}=\frac{\mathcal C_G^{\mathrm{infer}}}
                                 {\mathcal C_{\mathrm{AR}}^{\mathrm{infer}}},
    \quad r_{\mathrm{DGD}}=\frac{\mathcal C_D^{\mathrm{infer}}}
                               {\mathcal C_{\mathrm{AR}}^{\mathrm{infer}}},
    \quad r_{\mathrm{fb}}=\frac{\mathcal C_{\mathrm{fb}}}
                              {\mathcal C_{\mathrm{AR}}^{\mathrm{infer}}},
    \label{eq:app-cost-infer-normalized-parts}
\end{equation}
so that
\begin{equation}
    R_{\mathrm{infer}}
    =\frac{\mathcal C_G^{\mathrm{infer}}+\mathcal C_D^{\mathrm{infer}}
           +\mathcal C_{\mathrm{fb}}}{\mathcal C_{\mathrm{AR}}^{\mathrm{infer}}}
    =r_G^{\mathrm{infer}}+r_{\mathrm{DGD}}+r_{\mathrm{fb}}.
    \label{eq:app-infer-flop-ratio-ar}
\end{equation}
Using only the AR backbone linear term as reference and ignoring boundaries, the corresponding DiT contribution is
\begin{equation}
    \frac{2N_GQ_G^{\mathrm{prop}}+\mathcal C_{G,\mathrm{linear}}^{\mathrm{commit}}}
         {2N_{\mathrm{AR}}T_{\mathrm{new}}}
    \approx\rho_{\mathrm{gen}}
      \big(\nu_{\mathrm{gen}}S_{\mathrm{gen}}+\kappa_G^{\mathrm{commit}}\big)
      \frac{N_G}{N_{\mathrm{AR}}}.
    \label{eq:app-cost-gen-linear-reference}
\end{equation}
This is not the total inference ratio. When end-to-end initialization is included, the numerator and denominator of Eq.~\eqref{eq:app-infer-flop-ratio-ar} should add the prefill costs of \method and AR, respectively.

\subsection{Serial Call Rounds, Latency, and Block Parallelism}

We use $\mathcal H$ as a forward-round proxy, not as Transformer layer depth or actual latency. By convention, one module forward pass counts as one round, and independent CFG branches at the same step count as one group under ideal parallelism. Let $H_{D,j}$ be the number of call rounds for condition preparation, token decoding, and clean-token commit of the $j$-th block, and $H_{F,j}$ the number of rounds for the feedback encoder and the corrected condition/DiT commit. Then
\begin{equation}
    \mathcal H_{\text{\method}}
    \approx\sum_{j=1}^{J}\big(S_{\mathrm{gen}}+H_{D,j}+H_{F,j}\big).
    \label{eq:app-cost-depth}
\end{equation}
Ignoring boundaries, AR corresponds to $\mathcal H_{\mathrm{AR}}\approx T_{\mathrm{new}}$. With $H_D$ and $H_F$ denoting the average per-block numbers of rounds,
\begin{equation}
    \frac{\mathcal H_{\text{\method}}}{\mathcal H_{\mathrm{AR}}}
    \approx\frac{\rho_{\mathrm{gen}}}{B}
      \big(S_{\mathrm{gen}}+H_D+H_F\big).
    \label{eq:app-depth-speedup-ar}
\end{equation}
Thus, the condition under which this idealized call-round proxy is smaller is
\begin{equation}
    B>\rho_{\mathrm{gen}}\big(S_{\mathrm{gen}}+H_D+H_F\big),
    \label{eq:app-depth-better-condition}
\end{equation}
which is not a sufficient condition for end-to-end speedup. The decoding part of MDM contains about $M_j$ refinement groups and one clean commit. If the CFG branches are invoked sequentially, every forward pass must be counted, and they cannot be treated as one round merely because the formula is parallelizable.

Let $\tau_{\mathrm{init}}$ be the one-time initialization latency, $\tau_{G,j,i}$ the latency of the $i$-th DiT update of the $j$-th block (including the actual CFG scheduling), and $\tau_{\mathrm{prep},j}$, $\tau_{D,j}$, and $\tau_{F,j}$ the latencies of condition preparation, token decoding/commit, and feedback/corrected-cache commit, respectively. Under the current serial transaction,
\begin{equation}
    \tau_{\text{\method}}
    \approx\tau_{\mathrm{init}}+
    \sum_{j=1}^{J}\left[
      \sum_{i=1}^{S_{\mathrm{gen}}}\tau_{G,j,i}
      +\tau_{\mathrm{prep},j}+\tau_{D,j}+\tau_{F,j}
    \right].
    \label{eq:app-cost-latency}
\end{equation}
These latencies depend on module size, query/prefix length, branch batching, and hardware utilization, and should be measured on the corresponding execution paths; the same time cannot be assigned to every module forward pass.

\subsection{KV Cache and Temporary States}

Let $K_{\mathrm{hist}}$ and $T_{\mathrm{hist}}$ be the lengths of the committed latent/token histories, including the prompt. The DiT and feedback latent-encoder KV lie on the latent axis with length $\mathcal O(K_{\mathrm{hist}})$; the condition-stack and token-decoder KV lie on the token axis with length $\mathcal O(T_{\mathrm{hist}})$. The MDM token cache is bidirectional within a block and causal across blocks.

For $R\in\{G,E,C,D\}$, denoting the DiT, the feedback latent encoder, the condition stack, and the token decoder, respectively, let $n_R$ be the number of KV layers and $d_R^{\mathrm{KV}}=n_{\mathrm{kv},R}d_{\mathrm{head},R}$ the total KV-head width per layer, where $n_{\mathrm{kv},R}$ and $d_{\mathrm{head},R}$ are the number of KV heads and the per-head width. For an equal-length batch of $N_{\mathrm{seq}}$ sequences with $s_{\mathrm{KV}}$ bytes per cached element,
\begin{equation}
\begin{aligned}
    \operatorname{Mem}_{\mathrm{KV}}
    \approx 2s_{\mathrm{KV}}N_{\mathrm{seq}}\Big[&
       (n_Gd_G^{\mathrm{KV}}+n_Ed_E^{\mathrm{KV}})K_{\mathrm{hist}}\\
       &+(n_Cd_C^{\mathrm{KV}}+n_Dd_D^{\mathrm{KV}})T_{\mathrm{hist}}\Big].
\end{aligned}
\label{eq:app-cost-kv}
\end{equation}
The factor $2$ accounts for keys and values; different lengths, KV layer shapes, or precisions are summed according to the actual allocation, and padding must also be counted. Under GQA, $d_R^{\mathrm{KV}}$ should be used instead of the hidden width.

The four MDM CFG branches share one persistent token KV, so Eq.~\eqref{eq:app-cost-kv} is not multiplied by four; temporary logits, branch hidden states, and attention workspaces still affect peak memory. Proposals and masked states do not accumulate as history, but the sizes of cache concatenation, RoPE, and attention temporaries may still depend on the prefix length. This formula therefore counts only persistent KV and does not equal the total peak memory.

In summary, compression shortens the sequence and the latent-axis history of the latent prior, while continuous representation learning, token decoding, CFG, and feedback all retain their own costs. End-to-end comparisons should use a unified FLOP accounting scope and separately report actually processed positions, serial call rounds, latency, and KV capacity, rather than equating any single metric with overall speedup.

\section{Final-Checkpoint Pretraining Recipe}
\label{app:final-pretraining}

This section describes the large-scale recipe used for the final checkpoint in Table~\ref{tab:exp-pretrained-comparison}. It is separate from the matched component and controlled scaling setup in Section~\ref{subsec:exp-setup}.

\subsection{Shared Training Configuration}

The pretraining corpus uses OLMo~2 pretraining data retokenized with the Qwen3-4B tokenizer. The sequence length is $4096$.

All stages use AdamW with $\beta=(0.9,0.95)$ and weight decay $0.01$. The EMA decay in Stages~II and III is $0.9998$. At stage boundaries, model parameters, optimizer state, the learning-rate scheduler, and the data cursor are restored.

\subsection{Stage I: Dynamic-Compression VAE}

Stage~I processes approximately $52.43$B source tokens. Dynamic segmentation compresses about two tokens into one latent on average, with maximum span length $8$ and latent dimension $64$. The decoder is initialized from Qwen3-4B-Base, and the VAE encoder embedding is copied from the same model; the remaining compression modules are initialized from scratch.

The Qwen-initialized parameters remain frozen for the first $5.24$B source tokens and are subsequently updated at $0.1\times$ the main learning rate; the compression encoder also uses a $0.1\times$ multiplier. The learning rate is linearly warmed up from $10^{-6}$ to $10^{-4}$ over the first $5.24$B source tokens and then cosine-decayed to $10^{-6}$.

The decoder uses a causal dual-stream objective. Masked positions are supplied as \texttt{[MASK]} inputs, but the decoder self-attention remains causal rather than becoming bidirectional. Clean, noisy, and dropped latent conditions are sampled with proportions $0.4/0.3/0.3$, and their three token-loss weights are all $1$. The KL, duration, and encoder-mask auxiliary loss weights are $10^{-4}$, $0.1$, and $1$, respectively, and the encoder-mask span ratio is $0.15$.

\subsection{Stage II: Joint VAE and Latent-DiT Training}

Stage~II restores the complete VAE checkpoint, optimizer, scheduler, and data cursor from Stage~I. It introduces an approximately 4B-parameter latent DiT with $34$ Transformer layers, hidden dimension $2560$, $20$ attention heads, and latent block size $8$. This stage processes approximately $125.83$B source tokens.

The VAE continues to use the Stage-I causal dual-stream objective, with the encoder-mask span ratio increased to $0.30$. The DiT reads detached VAE latents, so its diffusion loss does not backpropagate into the VAE; the VAE is updated only by its compression and decoder objectives. The DiT uses velocity-prediction MSE, with timesteps sampled from $\operatorname{LogitNormal}(\mathrm{loc}=2,\mathrm{scale}=1)$.

The DiT learning rate is warmed up from $10^{-6}$ to $1.5\times10^{-4}$ and then cosine-decayed. The schedule is defined over a nominal 2.1T-token horizon, although training terminates before reaching that horizon.

\subsection{Stage III: Frozen VAE and Continued DiT Training}

Stage~III restores the complete DiT parameters, EMA, optimizer, learning-rate scheduler, and data cursor from the Stage-II checkpoint. The Stage-II EMA VAE encoder is frozen, and only the DiT continues to be optimized. The released model uses the matching EMA VAE decoder associated with this encoder.

The timestep distribution remains $\operatorname{LogitNormal}(\mathrm{loc}=2,\mathrm{scale}=1)$. This stage processes approximately $167.77$B source tokens from the retokenized OLMo~2 corpus.

\subsection{Source-Token Accounting}

\begin{table}[H]
    \centering
    \caption{\textbf{Final-checkpoint pretraining budget.} Token counts report source-token positions processed in each stage.}
    \label{tab:final-pretraining-budget}
    \small
    \setlength{\tabcolsep}{10pt}
    \begin{tabular}{lrl}
        \toprule
        \textbf{Stage} & \textbf{Source tokens} & \textbf{Primary update} \\
        \midrule
        Stage I   & 52.43B  & Dynamic-compression VAE \\
        Stage II  & 125.83B & VAE and latent DiT \\
        Stage III & 167.77B & Latent DiT only \\
        \midrule
        \textbf{Total} & \textbf{346.03B} & -- \\
        \bottomrule
    \end{tabular}
\end{table}

\end{document}